\documentclass[nohyperref]{article}

\usepackage{microtype}
\usepackage{graphicx}
\usepackage{subcaption}
\usepackage{booktabs} 

\usepackage{hyperref}

\usepackage[accepted]{icml2022}

\usepackage{amsmath,amsfonts,bm}

\def\eqref#1{equation~\ref{#1}}

\def\1{\bm{1}}

\def\rvx{{\mathbf{x}}}

\DeclareMathAlphabet{\mathsfit}{\encodingdefault}{\sfdefault}{m}{sl}
\SetMathAlphabet{\mathsfit}{bold}{\encodingdefault}{\sfdefault}{bx}{n}

\usepackage{hyperref}
\usepackage{url}
\usepackage{amsmath}
\usepackage{amsthm}
\usepackage{algorithm}
\usepackage{enumitem}
\usepackage{algorithmic}
\usepackage{mathtools}
\usepackage{tabularx}
\usepackage{wrapfig}
\usepackage{capt-of}
\usepackage{booktabs}
\usepackage{float}
\usepackage{comment}
\usepackage{varwidth}

\usepackage[capitalize,noabbrev]{cleveref}

\theoremstyle{plain}
\newtheorem{theorem}{Theorem}[section]
\newtheorem{proposition}[theorem]{Proposition}

\theoremstyle{definition}
\newtheorem{definition}[theorem]{Definition}

\theoremstyle{remark}

\usepackage[textsize=tiny]{todonotes}

\definecolor{Red}{rgb}{0.768, 0.054, 0.054}
\definecolor{Green}{rgb}{0,0.4,0.7}

\hypersetup{
    colorlinks=true,
    citecolor=teal,
    linkcolor=Red,
    urlcolor=Green,
    }

\icmltitlerunning{Set Based Stochastic Subsampling}

\begin{document}

\twocolumn[
\icmltitle{Set Based Stochastic Subsampling}



\icmlsetsymbol{equal}{*}

\begin{icmlauthorlist}
\icmlauthor{Bruno Andreis}{KA}
\icmlauthor{Seanie Lee}{KA}
\icmlauthor{A. Tuan Nguyen}{OX}
\icmlauthor{Juho Lee}{KA,AI}
\icmlauthor{Eunho Yang}{KA,AI}
\icmlauthor{Sung Ju Hwang}{KA,AI}
\end{icmlauthorlist}

\icmlaffiliation{KA}{Graduate School of AI, Korea Advanced Institute of Science and Technology (KAIST), Seoul, South Korea}
\icmlaffiliation{AI}{AITRICS, Seoul, South Korea}
\icmlaffiliation{OX}{University of Oxford, Oxford, United Kingdom}
\icmlcorrespondingauthor{Bruno Andreis}{andries@kaist.ac.kr}
\icmlcorrespondingauthor{Sung Ju Hwang}{sjhwang82@kaist.ac.kr}

\icmlkeywords{Machine Learning, ICML}

\vskip 0.3in
]



\printAffiliationsAndNotice{}  

\begin{abstract}
Deep models are designed to operate on huge volumes of high dimensional data such as images. In order to reduce the volume of data these models must process, we propose a set-based two-stage end-to-end neural subsampling model that is jointly optimized with an \textit{arbitrary} downstream task network (e.g. classifier). In the first stage, we efficiently subsample \textit{candidate elements} using conditionally independent Bernoulli random variables by capturing coarse grained global information using set encoding functions, followed by conditionally dependent autoregressive subsampling of the candidate elements using Categorical random variables by modeling pair-wise interactions using set attention networks in the second stage. We apply our method to feature and instance selection and show that it outperforms the relevant baselines under low subsampling rates on a variety of tasks including image classification, image reconstruction, function reconstruction and few-shot classification. Additionally, for nonparametric models such as Neural Processes that require to leverage the whole training data at inference time, we show that our method enhances the scalability of these models. 
\end{abstract}
\section{Introduction}
Deep models operate on large volumes of high-dimensional dense inputs such as the pixels of an image~\citep{imagenet,cifar,celeba}. Training or evaluating models with such data is computationally expensive and several works~\citep{concreteautoencoders,dps,invase} have proposed subsampling techniques to subsample such dense inputs. Subsampling methods have the potential to drastically reduce the data acquisition effort and reduce the inference time of algorithms that operate on dense inputs. Additionally, subsampling techniques have found applications in medical research for the purpose of interpretation~\citep{lime}.

However, these methods have a major drawback in that they require a fixed input structure. For instance, some of these methods are only applicable when each feature (e.g. a pixel) of the input is 1-dimensional. 
This restriction is imposed by the way the subsampling methods are designed: the input (e.g. an image) is flattened to a single vector and a model predicts a binary mask for each feature (e.g. pixels). This becomes problematic when we consider
a 3-channel image. Flattening out the image results in ambiguities as to which pixels to select since channels are treated independently.
For instance, in order to perform subsampling on CIFAR10 images, INVASE~\citep{invase} and DPS~\citep{dps} convert the images into single channel images (e.g. grey-scaled images) before subsampling pixels. In more extreme cases such as subsampling of training instances (different from instance-wise feature selection), each feature is itself an image (each possibly multi-channel) and hence those subsampling techniques are inapplicable. Finally, we show in our experiments that
most of these methods fail under extremely low subsampling rates and their performance is similar to random sampling in this setting.

\begin{figure*}
    \centering
    \includegraphics[width=1.0\textwidth]{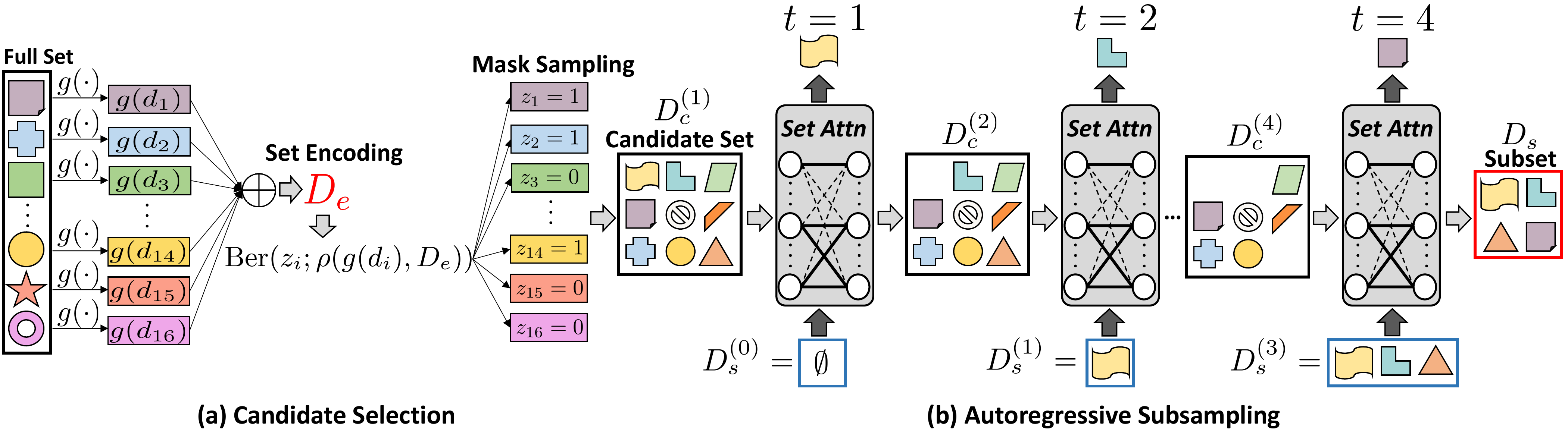}
    \vspace{-0.3in}
    \caption{\small \textbf{Concept}: Two-stage set-based stochastic subsampling process. \textbf{(a)} In the first stage, we screen out less important samples to construct the candidate subset. \textbf{(b)} In the second stage, we autoregressively subsample from the candidate subset.}
    \label{concept-figure}
    \vspace{-0.15in}
\end{figure*}

In order to tackle these limitation, we propose to consider each feature or instance as an element of a set.
We formulate the subsampling problem as selecting a subset of features or instances that minimizes the performance
degradation of an arbitrary model on an arbitrary task such as image classification, regression or instance subsampling
for target tasks such as few-shot classification. As a result, there are several advantages compared to the previous works.
First, we can handle arbitrary input structure using set functions~\citep{zaheer2017deep, lee2018set} parameterized with
expressive neural networks. Second, a subsampling model with set functions can
process arbitrary number of elements. As a result, the model is robust to a wide range of subsampling rates at test time even when trained with a
fixed sampling rate. Lastly, the set-based formulation unifies the feature and instance subsampling tasks under a single framework.

However, it is prohibitively expensive to process all the set elements (e.g. all the pixels in an image) with expressive set function such as Set Transformer~\citep{lee2018set} due to the self-attention among elements. Hence we propose an efficient two-stage subsampling method. In the first stage, as shown in Fig.~\ref{concept-figure}-(a), we learn the sampling rate for individual samples and efficiently screen out less important ones resulting in a subset which we call the \emph{candidate set}. The second stage is more fine-grained and designed to select a smaller subset from the candidate set by considering the relative importance of the samples in the candidate set using a conditionally dependent Categorical distribution through an autoregressive procedure as shown in Fig.~\ref{concept-figure}-(b). Once optimized, the resulting subsampling model can perform stochastic subsampling of a given input with linear time complexity. We call the resulting model \emph{Set based Stochastic Subsampling} (SSS) which is a general subsampling framework that is applicable to both \emph{feature} and \emph{instance} selection.

We validate SSS on multiple datasets and tasks such as 1D function regression, 2D image reconstruction and classification for both feature and instance selection. The experimental results show that SSS is able to subsample with minimal degradation on the target task performance under extremely low subsampling  rates, largely outperforming the relevant baselines. We summarize our contribution as follows.

\vspace{-0.1in}

\begin{itemize}
	\item {We reformulate the feature and instance subsampling problem by treating all the features or instances as members of a \textit{set}. This allows us to apply set based functions to the subsampling problem and extend its range of applicability.}
	
	\item {We propose a set based two-stage stochastic subsampling method that learns to efficiently subsample a set with minimal performance degradation on a target task.}
	
	\item {We verify the efficacy and generality of our method on various datasets for feature selection in the input space (e.g. pixels) and instance selection from a dataset, and show that it significantly outperforms the relevant baselines.}
\end{itemize}
\section{Related Work}\label{related-work}
\textbf{Set Functions}
Recently, extensive research efforts have been made in the area of set representation learning with the goal of obtaining
order-invariant (or equivariant) and size-invariant representations. Many propose simple methods to obtain set representations
by applying non-linear transformations to each element before a pooling layer~\citep{ravanbakhsh2016deep,qi2017pointnet,zaheer2017deep,sannai2019universal}. However, these models have limited
expressive power. Other approaches such as Set Transformer~\citep{lee2018set} consider the pairwise interactions among set
elements and hence can capture more complex statistics of set distributions. For large sets, the computational cost is expensive due to the self-attention operation. 

\textbf{Deep Learning Based Subsampling} 
Interest in deep learning based subsampling methods has produced many works mostly applied to feature selection. In ~\citet{concreteautoencoders}, continuous approximation of the Concrete Distribution~\citep{maddison2016concrete} is used for
\textit{global} feature selection where a fixed set of features are sampled across an entire dataset. In ~\citet{learningtoexplain},
instance-wise feature selection is used for interpretation of deep learning models applied to medical data. 
~\citet{dovrat2019learning} propose learning to subsample by generating virtual points, then matching them back to the original input. 
Several works~\citep{pointnet,pointnet++,pointcnn,fps,fastfps} also propose \emph{farthest point sampling}, which selects $k$ 
points from an input by ensuring that the selected samples are far from each other on a metric space. However our work is most 
similar to the recent works of ~\citet{invase} and ~\citet{dps} which learn a subsampling model conditioned on an given task.
However these models have limitations both in terms of their range of applicability and poor performance under extremely low subsampling
rates. Our method on the other hand is flexible, performs well under extremely low subsampling rates, and applicable to a wide range of subsampling problems.

\textbf{Image Compression} Due to the huge demand for data transfer over the internet, some works attempt to compress images with minimal
distortion. These models~\citep{toderici2017full,rippel2017real,mentzer2018conditional,li2018learning} typically consist of an encoder and
decoder, where the encoder transforms the image with a compact matrix and the decoder reconstructs the image. These methods, while highly
successful for the image compression problem, are less flexible than ours. Our model can be applied to arbitrary set structured data 
while the aforementioned models mainly work for images represented in tensor form.

\textbf{Active Learning} Active learning  aims to select data points for labeling given a small labeled set. This domain is
different from ours  since active learning does not consider the label information for the selected data points but our method does utilize label information. Also, our
motivation is quite different. We focus on optimal subsampling conditioned on an arbitrary task and this greatly differs from the goal of active learning. Methods such as~\citep{sener2017active,coleman2019selection,wei2015submodularity} all tackle the data selection problem in the active learning setting.

\textbf{Core-set Selection} Core-set~\citep{feldman2020introduction} methods aim at selecting a small \textit{weighted} subset of a given dataset that approximates the full \textit{dataset} with theoretical guarantees. They are mostly  targeted at instance selection and not, in generally, applicable to feature selection. Although we can utilize our subsampling method, SSS, for some instance selection tasks, our subsamplimg method, as well as those of ~\citet{dps} and ~\citet{invase}, is $\emph{not}$ a core-set selection method. Ours is based on a data driven approach, where we leverage expressive neural networks to learn to subsample the most representative subset for various downstream tasks. 
\section{Approach}\label{sec:approach}

\subsection{Preliminaries}
We consider a set $D = \{d_i\}_{i=1}^n$ as an input where each $d_i$ either represents an instance-label pair $(x_i, y_i)$ or a \textit{feature} such as the pixel value of an image.
We cast the subsampling problem as the selection of a subset $D_s = \{s_{j}\}_{j=1}^k \subset D$ with $k \ll n$ such that $\ell(\cdot, D) \approx \ell(\cdot, D_s)$ for an arbitrary loss function $\ell(\cdot, D)$  over the full set $D$.
In order to apply set functions to the subsampling problem, we need to properly design the neural network components to have some symmetrical properties such as permutation invariance (Definition~\ref{def-inv}), equivariance (Definition~\ref{def-equiv}), and exchangeability (Definition~\ref{def-exchange}).

\begin{definition}[Permutation]
We say a function $\pi$ is a permutation iff  $\pi \in \mathfrak{S}_n = \{f:[n]\rightarrow [n] \mid f\text{ is bijective}\}$.
\end{definition}

\begin{definition}[Permutation Invariance]
We say a function $f:X^n\rightarrow Y$ is permutation invariant iff $f(\pi(\rvx)) = f(\rvx)$ for all $\pi \in \mathfrak{S}_n$ and for all $\rvx \in X^n$.
\label{def-inv}
\end{definition}

\begin{definition}[Permutation Equivariance]
We say $f: X^n \rightarrow Y^n$ is permutation equivariant iff $\pi(f(\rvx)) = f(\pi(\rvx))$ for all $\pi \in \mathfrak{S}_n$ and for all $\rvx \in X^n$.
\label{def-equiv}
\end{definition}

\begin{definition}[Exchangeability]
A distribution for a set of random variables $X=\{\rvx_i\}_{i=1}^n$ is exchangeable iif $p(X) = p(\pi(X))$ for all $\pi \in \mathfrak{S}_n$
\label{def-exchange}
\end{definition}

In the following sections, we propose a two-stage Set based Stochastic Subsampling (SSS) method that leverages permutation invariant and equivariant set functions parameterized by $\theta$ to learn the conditional distribution $p_\theta(D_s|D)$. The first stage, \textit{candidate selection}, and the second stage, \textit{autoregressive subset selection}, are illustrated in Fig.~\ref{concept-figure}. In general, we estimate the parameters of the subsamping model $\theta$ by minimizing the following loss: $\mathbb{E}_{p(D)}[\mathbb{E}_{p_\theta(D_s|D)} [\ell(\cdot, D_s)]]$, where  $p(\cdot)$ denotes some unknown data distribution.

\subsection{Set based Stochastic Subsampling}
To select $D_s$, we propose to model the pairwise interactions among the elements of $D$ and then choose a few representative elements in $D$ based on the relative sample importance computed from the interaction scores. However, when the cardinality of $D$ is large, 
modeling pairwise interactions becomes computationally infeasible since we need to \textit{compare each element in $D$ with all the other elements}. This computational bottleneck motivates the first stage of SSS of which the goal is to construct a smaller subset $D_c$, which we refer to as the \emph{candidate set}, at a coarse level without considering pairwise interaction. We call the first stage \textit{candidate selection} and the second stage, which is more fine-grained, \textit{autoregressive subset selection} and selects $D_s$ from $D_c$.

\subsection{Candidate Selection}
We formulate the candidate selection problem as a random Bernoulli process where the parameters of the Beronulli distribution are conditioned on the \textit{set representation} of $D$ and the individual elements $d_i \in D$. Specifically, we first encode the set $D$ to a single representation $D_e$ with a set encoding function (see Fig.~\ref{concept-figure}-(a)) as follows:
\begin{equation}
    D_e = \frac{1}{n} \sum_{i=1}^{n} g (d_i), \quad n=|D|
\label{set-enc}
\end{equation}
where $g$ is a neural network which projects each element in $D$ independently to a lower
dimension. $D_e$ captures coarse-grained global information in $D$ with computational efficiency.
This encoding scheme is similar to DeepSets~\citep{zaheer2017deep} except that we do not perform message-passing, which is computationally expensive, between the set elements. 
\begin{proposition}
Given the set $D$ and the affine transformation with non-linearity $g$, the set encoding $D_e$ in Eq.~\ref{set-enc} is permutation invariant.
\label{prop:invariance}
\end{proposition}

We then concatenate every $g(d_i)$ with $D_e$, denoted as $\overline{d_i}$. That is, $\overline{d_i} = [d_i, D_e]$, where $[\:]$ is the concatenation operation.
This ensures that each element of $D$ has a \textit{global} view of all the other elements in the set at a coarse level.
For each $d_i \in D$, we sample a mask $z_i\sim p_\theta (z_i|d_i,D)$ with
\begin{equation}
     p_\theta(z_i | d_i, D) = \text{Ber}(z_i ; \rho (\overline{d_i})),\quad \rho(\overline{d_i}) = \sigma(h(\overline{d_i}))
\label{can-ber}
\end{equation}
where $h$ is a neural network that outputs the logits for the probability that $d_i$ is in the candidate set $D_c$ and $\sigma(\cdot)$
is the sigmoid function, and Ber denotes the Bernoulli distribution. $z_i$ is a binary random variable where $z_i = 1$ indicates that
$d_i$ is an element in $D_c$. We concatenate all $z_i$'s to obtain a single vector $Z = [z_1, \ldots, z_n]$. 
Since sampling from the Bernoulli distribution is not differentiable, during training, we use the continuous relaxations of the Bernoulli distribution~\citep{maddison2016concrete,jang2016categorical, gal2017concrete} to sample $z_i$ for each $d_i$. This is illustrated as Mask Sampling in Fig.~\ref{concept-figure}-(a). 
Although pairwise interactions are not considered in this stage, the ablation studies (Appendix~\ref{app:ablation}) show that learning $p_\theta(z_i | d_i, D)$
leads to selecting highly informative samples compared to random selection of the candidate set $D_c$. 

\textbf{Constraining the size of $D_c$} For computational efficiency, we want to restrict the size of $D_c$ to save computational cost when constructing $D_s$. 
Hence we introduce a sparse Bernoulli prior $p(Z)=\prod_{i=1}^n\text{Ber}(z_i;r)$ with small $r>0$ and minimize the KL divergence along with a target downstream task loss $\ell(\cdot,D_s)$ w.r.t $\theta$ as follows:
\begin{equation}
  \mathbb{E}_{p(D)} \left[ \mathbb{E}_{p_\theta(D_s | D)} [\ell(\cdot, D_s)] + \beta \text{KL}[p_\theta(Z|D) || p(Z)]\right]
  \label{eq:genearl-obj}
\end{equation}
where $p_\theta(Z|D) = \prod_{i=1}^n p_\theta(z_i|d_i,D)$ and $\beta >0$ is a hyperparmeter used to control the sparsity level in $Z$.

\begin{proposition}
The candidate selection function which outputs the probability for each element $D$ is permutation equivariant and the probability $p_\theta(Z|D)$ is exchangeable.
\label{prop-candidate}
\end{proposition}


\subsection{Autoregressive Subset Selection}\label{sec3-4} 
At this stage in the pipeline, we have a set $D_c$ with $m=|D_c| \ll |D|$, which is small enough to perform fine-grained subset selection through pairwise modeling.
To select a subset with $k$ elements from $D_c$, we require $k$ iterative steps. As shown in Fig.~\ref{concept-figure}-(b), at time step $t$, we have the subset $D_s^{(t-1)}$ constructed from the previous iteration with $D_s^{(0)} = \emptyset$ and $D_c^{(t)} = \{s^{(t)}_1,\ldots, s^{(t)}_{m_t}\} = D_c \setminus D_s^{(t-1)}$. Assuming we have a function $\varphi \circ f$ (Set Attention in Fig~\ref{concept-figure}-(b)) for modeling pairwise interactions between the elements of an input set, we autoregressively compute the interaction scores at time step $t$ as follows:
\begin{equation}\label{eqn:interactions}
    \Tilde{\pi}^{(t)} = (\Tilde{\pi}^{(t)}_1, \ldots, \Tilde{\pi}^{(t)}_{m_t}) = \sigma (\varphi \circ f(D_c^{(t)}, D_s^{(t-1)}))
\end{equation}
where $\sigma(\cdot)$ denotes the sigmoid function and $\varphi \circ f$ is a composition of two neural networks: $f$ computes interaction scores between elements in $D^{(t)}_c$ and $\varphi$ outputs element-wise logits using the interaction scores. Further, $\Tilde{\pi}^{(t)}$ is the vector of interaction scores for all elements in $D_c^{(t)}$ at the current time step $t$. Given $\Tilde{\pi}^{(t)}_i>0$ for all $i=1,\ldots, m_t$, we can compute the probability of an element $s^{(t)}_i$ being selected from $D_c^{(t)}$ as:
\begin{equation}\label{eqn:interaction_probs}
    p_\theta(s^{(t)}_i|D^{(t)}_c, D^{(t-1)}_s) = \pi^{(t)}_i, \quad \pi^{(t)}_i = \frac{\Tilde{\pi}^{(t)}_i}{\sum_{j=1}^{m_t}\Tilde{\pi}^{(t)}_j},
\end{equation}
where $m_t = |D^{(t)}_c|$. That is, we normalize $\Tilde{\pi}^{(t)}$ over all the elements in $D^{(t)}_c$ at time step $t$ to obtain a valid probability distribution. The key to avoiding redundant elements in $D_s$ lies in the fact that for each element added to $D_s$, its selection is conditioned on both the candidate
set $D_c^{(t)}$ and all the elements in the subset $D_s^{(t-1)}$ as described in Eq.~\ref{eqn:interactions} \&~\ref{eqn:interaction_probs}.
For the choice of the function $f$, we use a MultiHead Attention Block (MAB)~\citep{lee2018set} which we describe in detail in \textbf{Appendix}~\ref{model-spec}. Additionally, we can stack multiple MABs for the function $f$ to model higher level interactions.

\begin{proposition}
The functions $\varphi$ and $f$, and the pairwise interaction score $\Tilde{\pi}_t$ are permutation equivariant for all time steps $t$ in the autoregressive subset selection stage.
\label{prop:equivariance}
\end{proposition}

With Eq.~\ref{eqn:interaction_probs}, we can sample an element $s^{(t)}\sim \text{Cat}(\pi^{(t)}_1, \ldots, \pi^{(t)}_{m_t})$ from the candidate subset $D_c^{(t)}$ and construct $D^{(t)}_s = D^{(t-1)}_s \cup \{s^{(t)}\}$, where Cat is the Categorical distribution. During training, it can be expensive to sample $k$ times from the Categorical distribution since it involves computing Eq.~\ref{eqn:interactions} $k$ times. We 
remedy this by selecting $l$ elements from $D_c^{(t)}$ at once, which reduces the number of iterations to $k/l$ for selecting $k$ elements. We may also sample $l$ elements from the multinomial distribution with probability $\pi^{(t)}$ without replacement. However, this sampling procedure is non-differentiable, and hence it cannot be trained with backpropagation. Instead, we independently sample $l$ elements
from the continuous relaxation of Categorical distributions~\citep{maddison2016concrete, jang2016categorical} using the same probabilities in Eq.~\ref{eqn:interaction_probs} to approximate sampling from the multinomial distribution as shown in Fig.~\ref{concept-figure}-(b). Since we want to simulate sampling without replacement, we discard all elements sampled more than once. This sampling procedure guarantees that we get at most $l$ elements at each iteration. A similar sampling procedure is adopted in previous works~\citep{concreteautoencoders, learningtoexplain}. We detail this training algorithm in  \textbf{Appendix}~\ref{greedy}. 

\textbf{Proofs} of Propositions~\ref{prop:invariance},~\ref{prop-candidate} \& ~\ref{prop:equivariance} are in Appendix~\ref{proofs}.



\textbf{Time Complexity} The time complexity of SSS depends heavily on the choice of the  function $f$. Using MAB as $f$, the time complexity of SSS is $O(n) + O(k^2m/l)$ where $n, m, k$ correspond to $|D|$, $|D_c|$ and $|D_s|$ respectively.

\begin{figure*}[t]
    \centering
    	\includegraphics[width=1.0\linewidth]{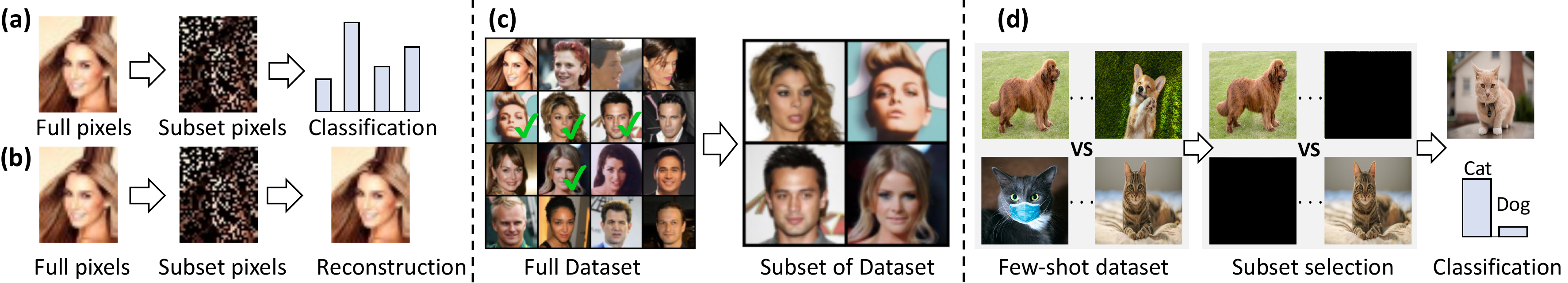}
    \vspace{-0.25in}
    \caption{\small \textbf{Target Tasks:} \textbf{(a)} Feature selection for reconstruction. \textbf{(b)} Feature selection for prediction. \textbf{(c)} Selection of representative instances. \textbf{(d)} Instance selection for few-shot classification.}
    \vspace{-0.15in}
    \label{fig-task-description}
\end{figure*}

\subsection{Tasks}\label{tasks}
\textbf{Set Classification \& Prediction} As shown in Fig.~\ref{fig-task-description}-(a), we train a neural network parmeterized with $\phi$ to predict a single target value $y_D$ for the subset $D_s$ of the given full set $D$, where $D$ is a collection of the features from a single instance such as the pixels of an image. For instance, the target $y_D$ is either the class of an image for classification or the attributes of a face in an image. Here, our goal is learning to select the most representative subset $D_s \subset D$ such that we can maximize the log likelihood $\log p_{\phi}(y_D | D_s)$ with computational efficiency. In order to achieve this goal, we jointly train the SSS model and the neural network which predicts the target value $y_D$ for $D_s$ to minimize the negative log-likelihood, the loss function $\ell(\cdot, D_s)$ described in Eq.~\ref{eq:genearl-obj} and KL divergence to enforce sparsity in $Z$ (the selection masks for the candidate set) as follows:
\begin{equation}
    \begin{split}
       \mathbb{E}_{p(D)} [& \mathbb{E}_{p_\theta(D_s | D)} [ - \log p_{\phi} (y_D | D_s) ] \\
                               & + \beta \text{KL} [p_\theta(Z|D) || p(Z)] ]
    \end{split}
\end{equation}
where $p(Z) = \prod_{i=1}^n \text{Ber}(z_i;r)$ with small $r>0$. We provide experimental results in Section \ref{sec411} and a corresponding graphical model in \textbf{Appendix}~\ref{graphical-model}.

\textbf{Set Reconstruction} Given a full set $D=\{X, Y\}$ consisting of  2d coordinates $X=\{x_i\in\mathbb{R}^2\}_{i=1}^n$ and corresponding pixel values $Y=\{y_i \in\mathbb{R}^3 \}_{i=1}^n$, we want to select the most representative subset $D_s=\{X_s, Y_s \mid X_s \subset X, Y_s \subset Y\}$ to reconstruct all pixel values $y_i \in Y$ for each $x_i \in X$, as shown in Fig.~\ref{fig-task-description}-(b). We jointly train the SSS model and a neural network with parameters $\phi$ predicting pixel values to minimize the loss function w.r.t $\theta$ and $\phi$ as follows:
\begin{equation}
    \begin{split}
         \mathbb{E}_{p(D)} [& \mathbb{E}_{p_\theta(D_s | D)} [ - \log p_{\phi} (Y|X,D_s) ] \\
                                 & + \beta \text{KL} [p_\theta(Z|D) || p(Z)] ]
    \end{split}
\end{equation}
We enforce sparsity on the subset $D_s$ by minimizing the KL-divergence between the mask probability $p_\theta(Z|D)$ and sparse prior $p(Z)=\prod_{i=1}^n\text{Ber}(z_i;r)$ with small $r>0$. Moreover, minimizing the negative log likelihood, which corresponds to $\ell(\cdot, D_s)$ in Eq.~\ref{eq:genearl-obj}, ensures that the constructed $D_s$ is the most representative for the downstream tasks. 
We implement $p_{\theta}(Y|X,D_s)$ as an Attentive Neural Process (ANP)~\citep{anp}. The ANP takes $D_s$  as input
and predicts a distribution of the elements in the original set $D$. It mimics the behaviour of a Gaussian Process but with reduced inference
complexity. We present experimental results for this task in Section~\ref{sec41} and a corresponding graphical model depiction in  \textbf{Appendix}~\ref{graphical-model}.

\textbf{Dataset Distillation: Instance Selection} In this task, we are given a collections of datasets $\mathcal{D} = \{D^{(1)},\ldots,D^{(m)}\}$ with $D^{(i)} \cap D^{(j)} =\emptyset$ for $i\neq j$ and $D^{(i)}\stackrel{iid}{\sim}p(D)$. The goal is to select the most representative subset $D^{(i)}_s$ with $|D^{(i)}_s| \ll |D^{(i)}|$ for each dataset $D^{(i)}=\{d^{(i)}_1,\ldots,d^{(i)}_{n}\}\in\mathcal{D}$, where $d^{(i)}_i$ is a data point uniformly sampled from the entire datasets $\mathcal{D}$. Using CelebA dataset as an illustrative example, shown in Fig.~\ref{fig-task-description}-(c), $D^{(i)}$ consists of $n$ randomly sampled faces from the entire dataset and the task is to construct a subset, $D^{(i)}_s$, most representative of $D^{(i)}$. 

In order to learn to select the subset $D_s$ from each $D\in\mathcal{D}$ with \textit{unsupervised learning}, we jointly train the SSS model and  a generative model such that the SSS model chooses the most representative subset so that the generative model can reconstruct all the images $d_i \in D$ from the subset. Naïvely, we can minimize the sum of negative log-likelihood $\sum_{d_i\in D} -\log p_\phi (d_i|D_s)$ for the loss function $\ell (\cdot, D_s)$ and  KL divergence in Eq.~\ref{eq:genearl-obj}. However, we find that the generative model outputs mean images for all  $d_i$. To capture variations of different images, we introduce three latent variables $\alpha_i$, $c_i$, and $w_i$ which both depend on $d_i$. We provide graphical model illustration of this task in the \textbf{Appendix}~\ref{graphical-model}. Since it is intractable to compute the log likelihood $\log p_\phi (d_i|D_s)$ by marginalizing over all the latent variables, we derive the upper bound of the marginal negative log likelihood using variational inference and plug the upper bound into the loss function $\ell(\cdot,D_s)$ in Eq.~\ref{eq:genearl-obj} as follows:
\begin{equation}\label{distill}
\small    
    \begin{split}
    &\mathbb{E}_{p(D)}\Bigg[\mathbb{E}_{p_\theta(D_s|D)} \Big[  \sum_{d_i \in D} [ \mathbb{E}_{q_{\psi}(w_i, c_i | d_i, D_s)} \left[-\log p_{\phi}(d_i | w_i, c_i) \right] \\
				           & + \text{KL}[q_{\psi}(w_i | d_i) || p_{\xi}(w_i)] + \text{KL}[q_{\psi}(\alpha_{i} | d_i) || p_{\xi}(\alpha_i)] \\
				           & + \text{KL}[q_{\psi}(c_i | D_s, \alpha_i) || p_{\xi}(c_i)] ] \Big] +\beta\text{KL}[p_\theta(Z|D) || p(Z)]\Bigg]
    \end{split}
\end{equation}
where $p_{\xi}(\cdot)$ are priors on their respective latent variables, $p(Z)=\prod_{i=1}^n\text{Ber}(z_i;r)$ is sparse prior with small $r>0$ over the mask for candidate set selection in SSS, $p_\phi(\cdot)$ is the decoder to reconstruct $d_i$, and all variational posteriors $q_{\psi}(\cdot)$ are parameterized with neural networks. All priors are the standard normal distribution.

In summary, we jointly train both the SSS and generative model to minimize the objective in Eq.~\ref{distill} w.r.t $\theta,\phi$, and $\psi$ for all $D\in\mathcal{D}$ and leverage the optimized SSS to select a few representative instances of the dataset, which results in distilled dataset. Experimental results are in Section~\ref{sec4.3}.

\textbf{Dataset Distillation: Classification} Finally for the dataset distillation task, we consider the problem of selecting prototypes for few-shot classification as shown in Fig.~\ref{fig-task-description}-(d). We adopt Prototypical Networks~\citep{protonet} and deploy the SSS model for selecting representative prototypes from the support set for each class. We minimize the objective in Eq.~\ref{eq:genearl-obj}, where we use the distance loss induced by the metric space from Prototypical Networks for the target task loss $\ell(\cdot,D_s)$, to jointly train the Prototypical Networks and SSS. Note that we use $D_s$, the subset of the support set, for computing loss and prediction.  By learning to select the prototypes, we can remove outliers that would otherwise change the decision boundaries in the classification task where we need to predict the label $y_*$ for an unseen instance $x_*$. Experimental results for this task are in Section~\ref{sec4.3} and its graphical model description is in \textbf{Appendix}~\ref{graphical-model}.

\begin{figure*}
\vspace{-0.1in}
\centering
    \begin{subfigure}{0.6\textwidth}
		\centering
		\includegraphics[width=\linewidth]{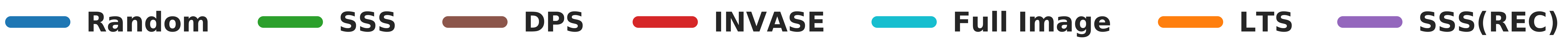}
		\vspace{-0.15in}
	\end{subfigure}
	\begin{subfigure}{.25\textwidth}
		\centering
		\includegraphics[width=\linewidth]{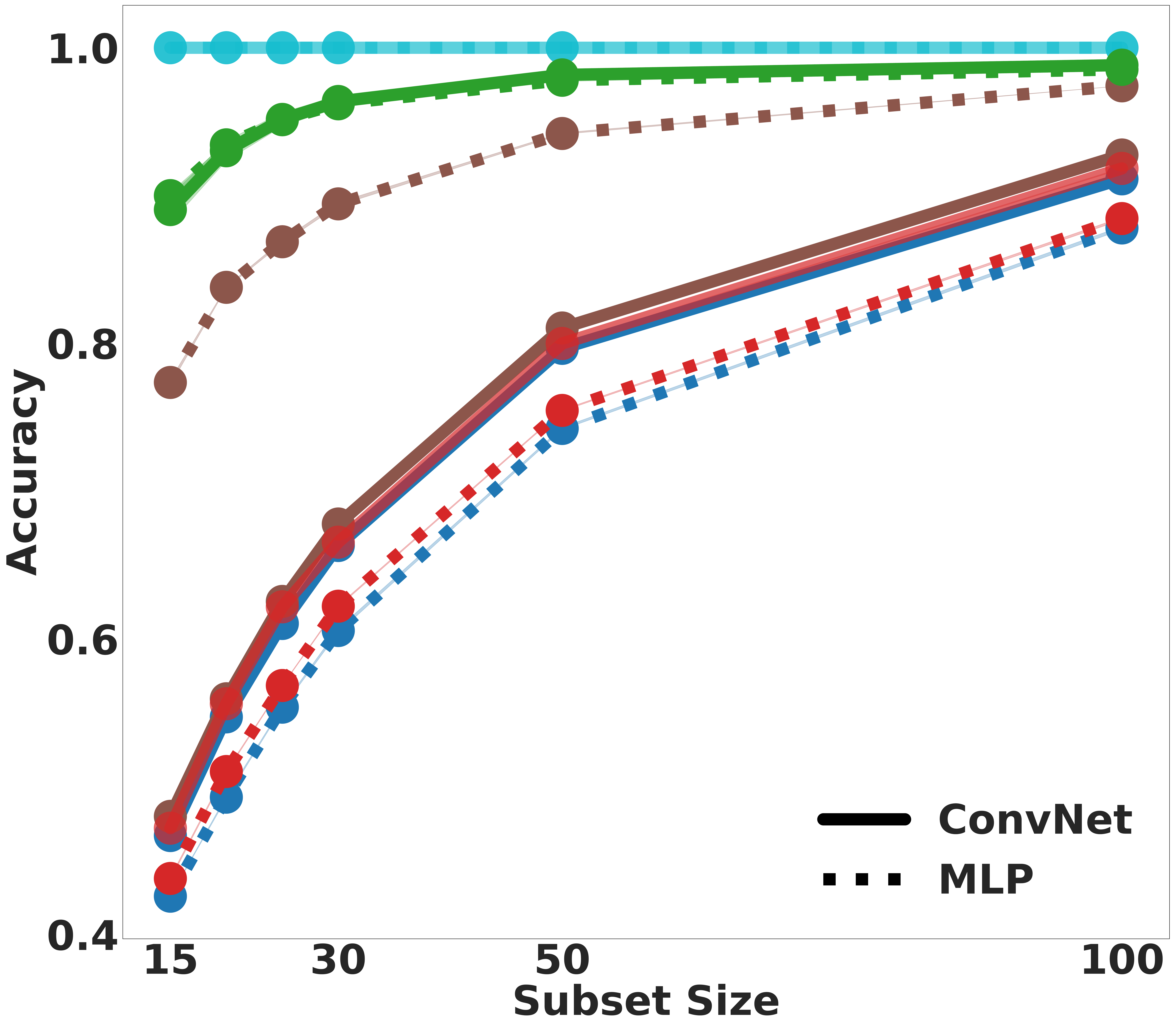}
		\captionsetup{justification=centering,margin=0.5cm}
		 \vspace{-0.25in}
		\caption{\small}
		\label{mnist_classification}
	\end{subfigure}%
	\begin{subfigure}{.25\textwidth}
		\centering
		\includegraphics[width=\linewidth]{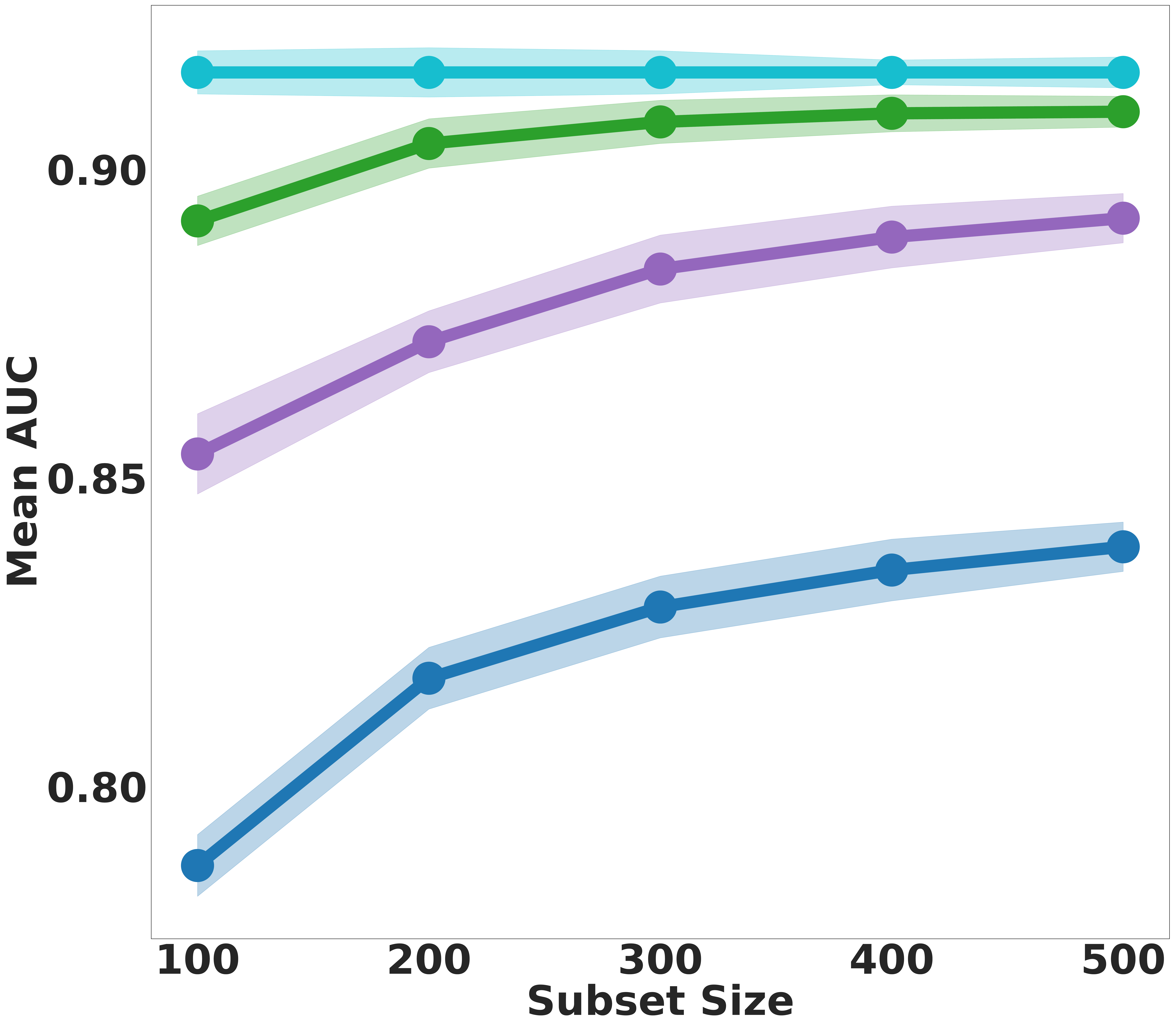}
		\captionsetup{justification=centering,margin=0.5cm}
		\vspace{-0.25in}
		\caption{\small}
		\label{celeba_classification}
	\end{subfigure}%
	\begin{subfigure}{.25\textwidth}
		\centering
		\includegraphics[width=\linewidth]{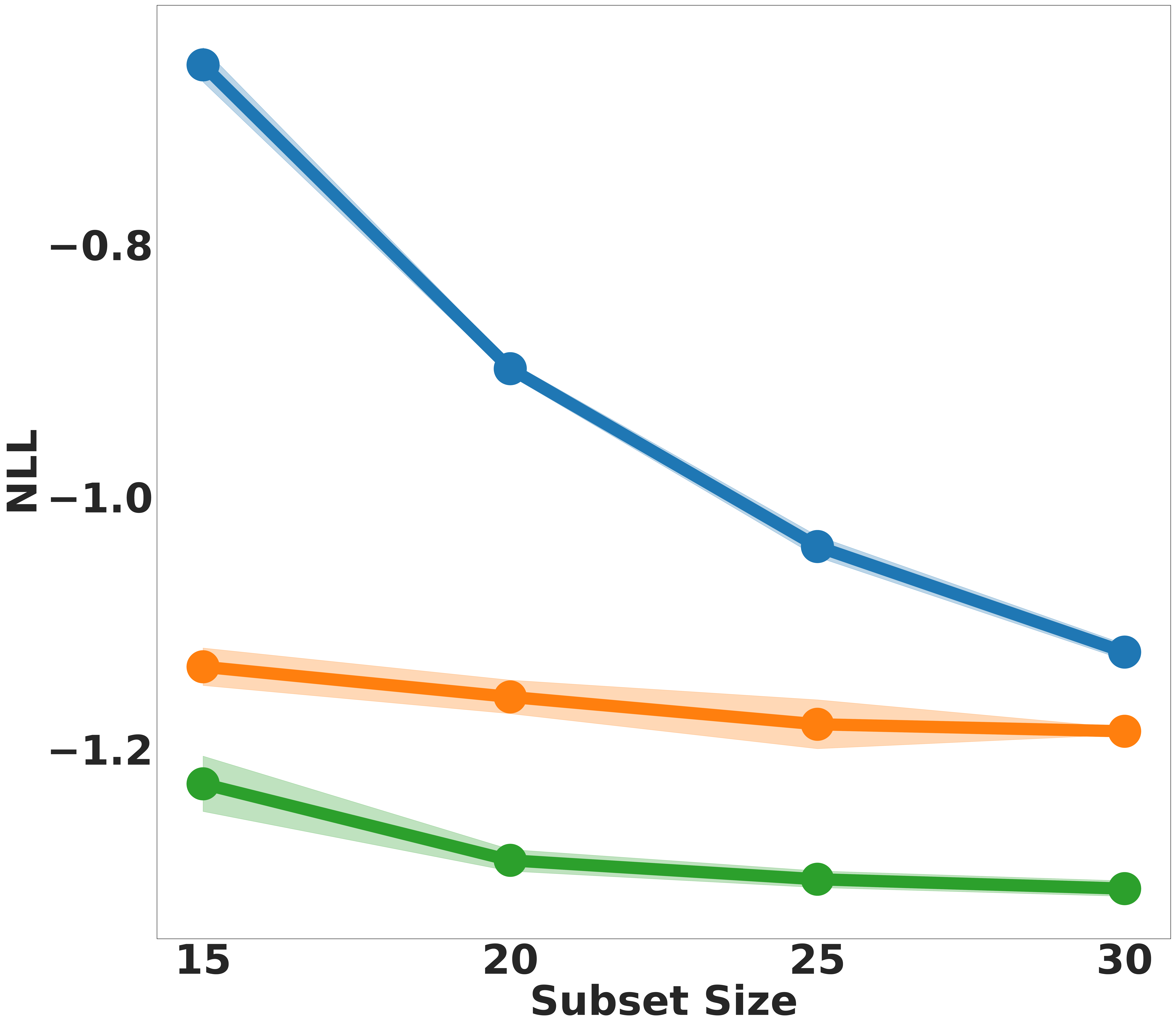}
		\captionsetup{justification=centering,margin=0.5cm}
		\vspace{-0.25in}
		\caption{\small}
		\label{function_reconstruction}
	\end{subfigure}%
	\begin{subfigure}{.25\textwidth}
		\centering
		\includegraphics[width=\linewidth]{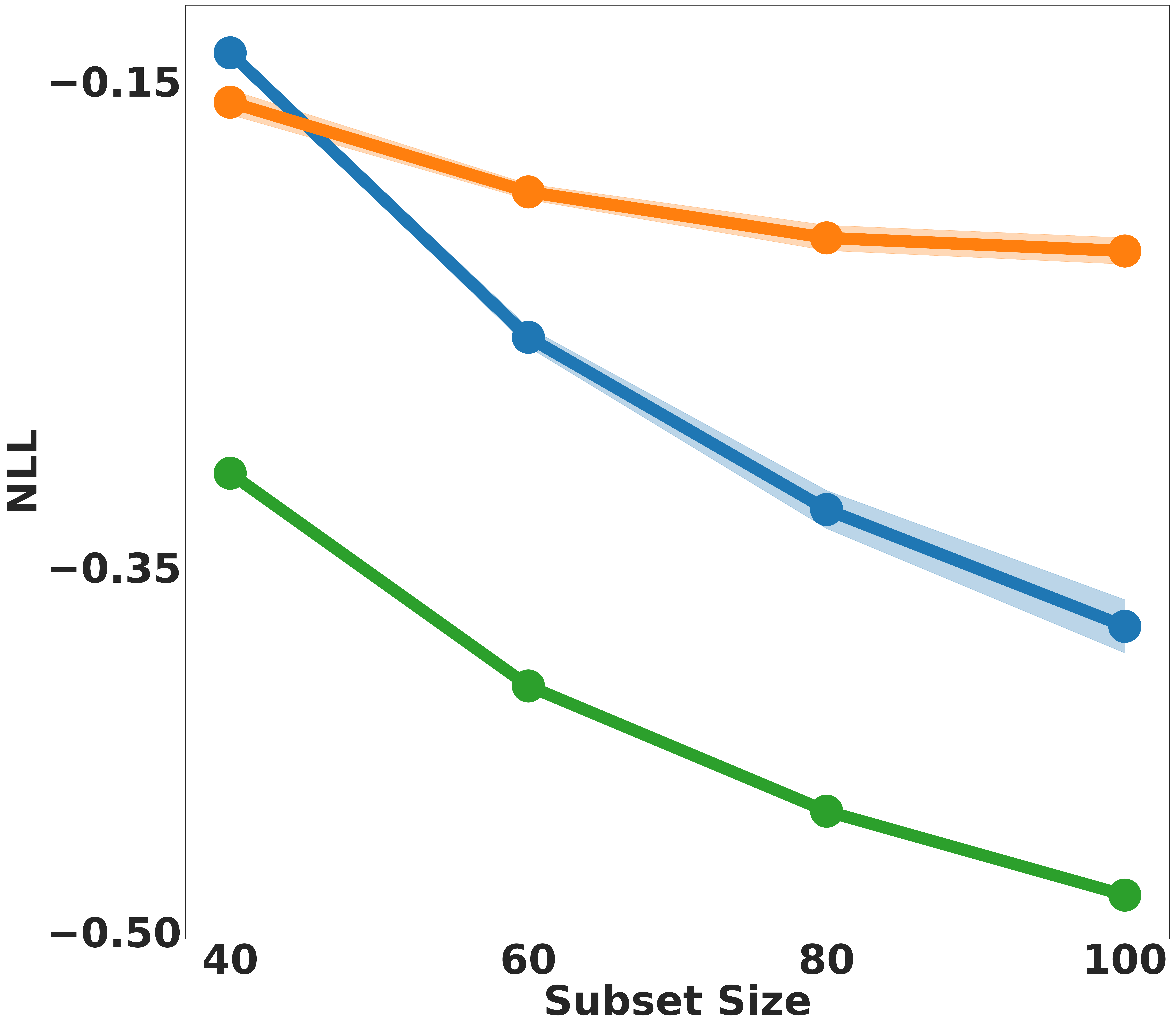}
		\captionsetup{justification=centering,margin=0.5cm}
		\vspace{-0.25in}
		\caption{\small}
		\label{celeba_reconstruction}
	\end{subfigure}
	\vspace{-0.2in}
	\caption[Reconstruction NLL]{\small
	\textbf{(a)} MNIST Classification. \textbf{(b)} CelebA classification. \textbf{(c)} 1D Function Reconstruction.
	\textbf{(d)} Image Reconstruction on CelebA. 
	}
	\label{reconstruction}
	\vspace{-0.1in}
\end{figure*}

\begin{figure*}
    \centering
    \begin{subfigure}{0.4\textwidth}
		\centering
		\includegraphics[width=\linewidth]{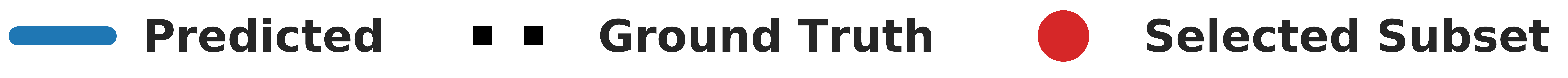}
		\vspace{-0.2in}
	\end{subfigure}
	\begin{subfigure}{.3\textwidth}
		\centering
		\includegraphics[width=\linewidth]{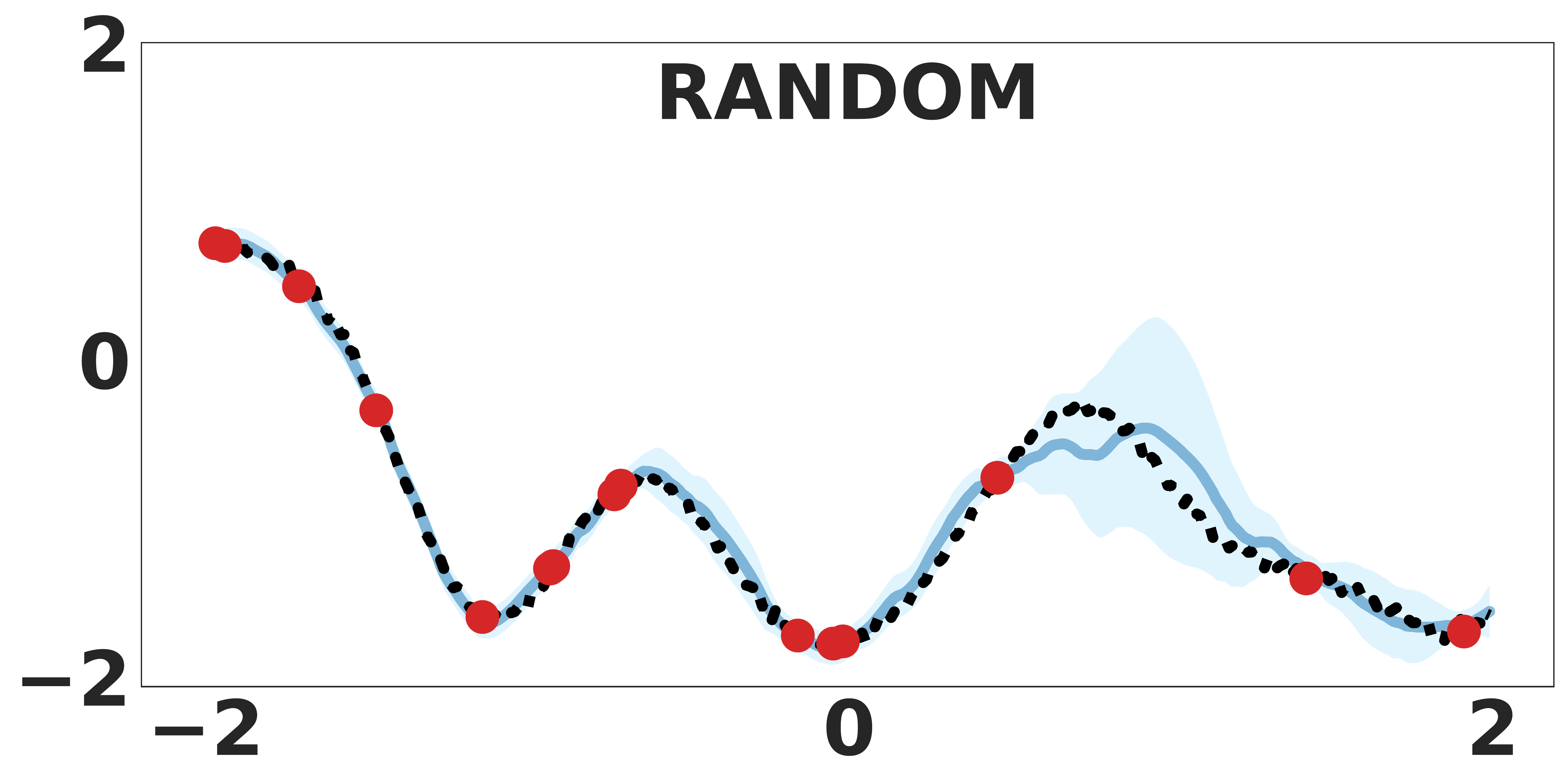}
		\label{reconstruction-function-random}
	\end{subfigure}%
	\begin{subfigure}{.3\textwidth}
		\centering
		\includegraphics[width=\linewidth]{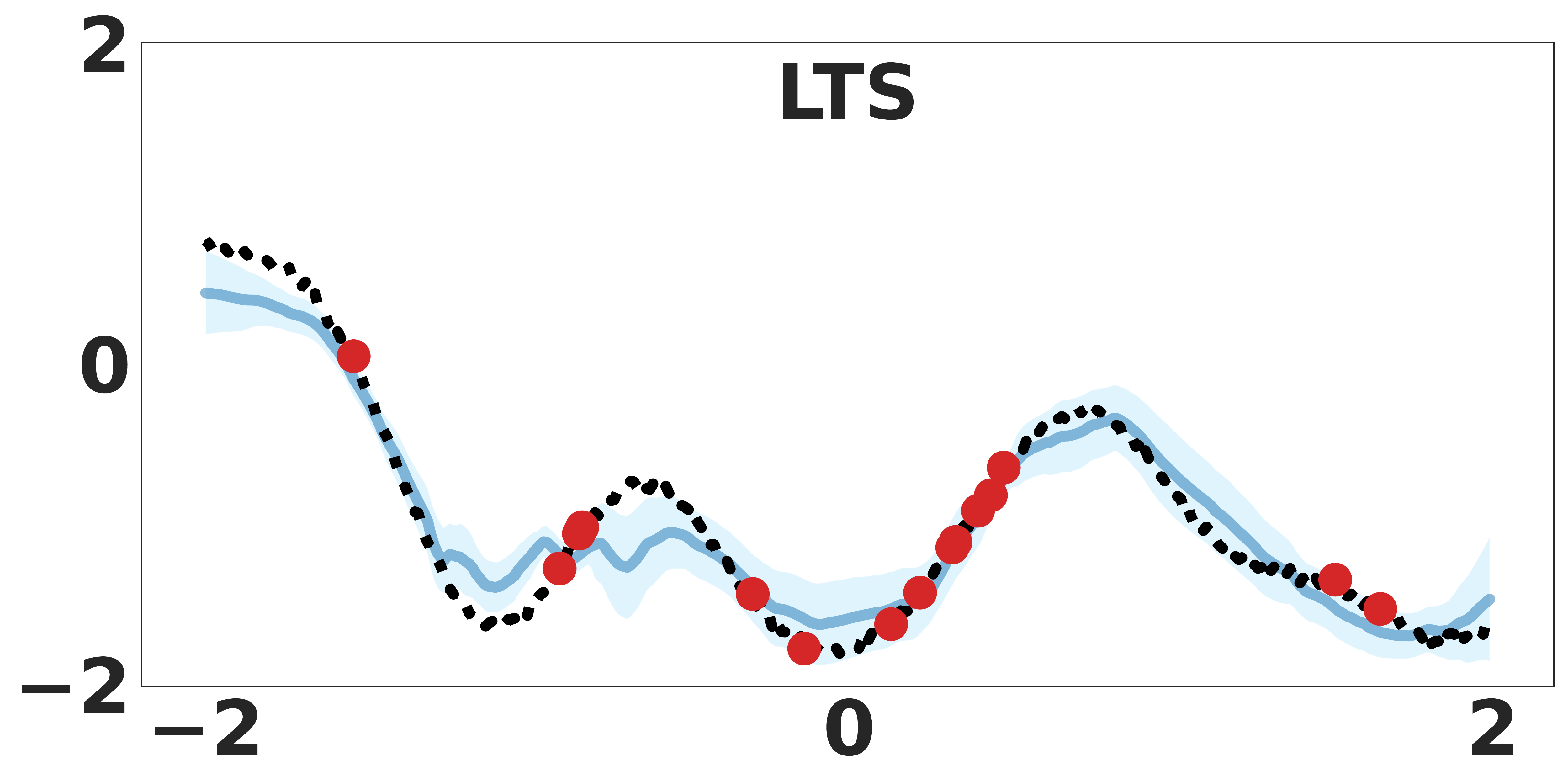}
		\label{reconstruction-function-lts}
	\end{subfigure}%
	\begin{subfigure}{.3\textwidth}
		\centering
		\includegraphics[width=\linewidth]{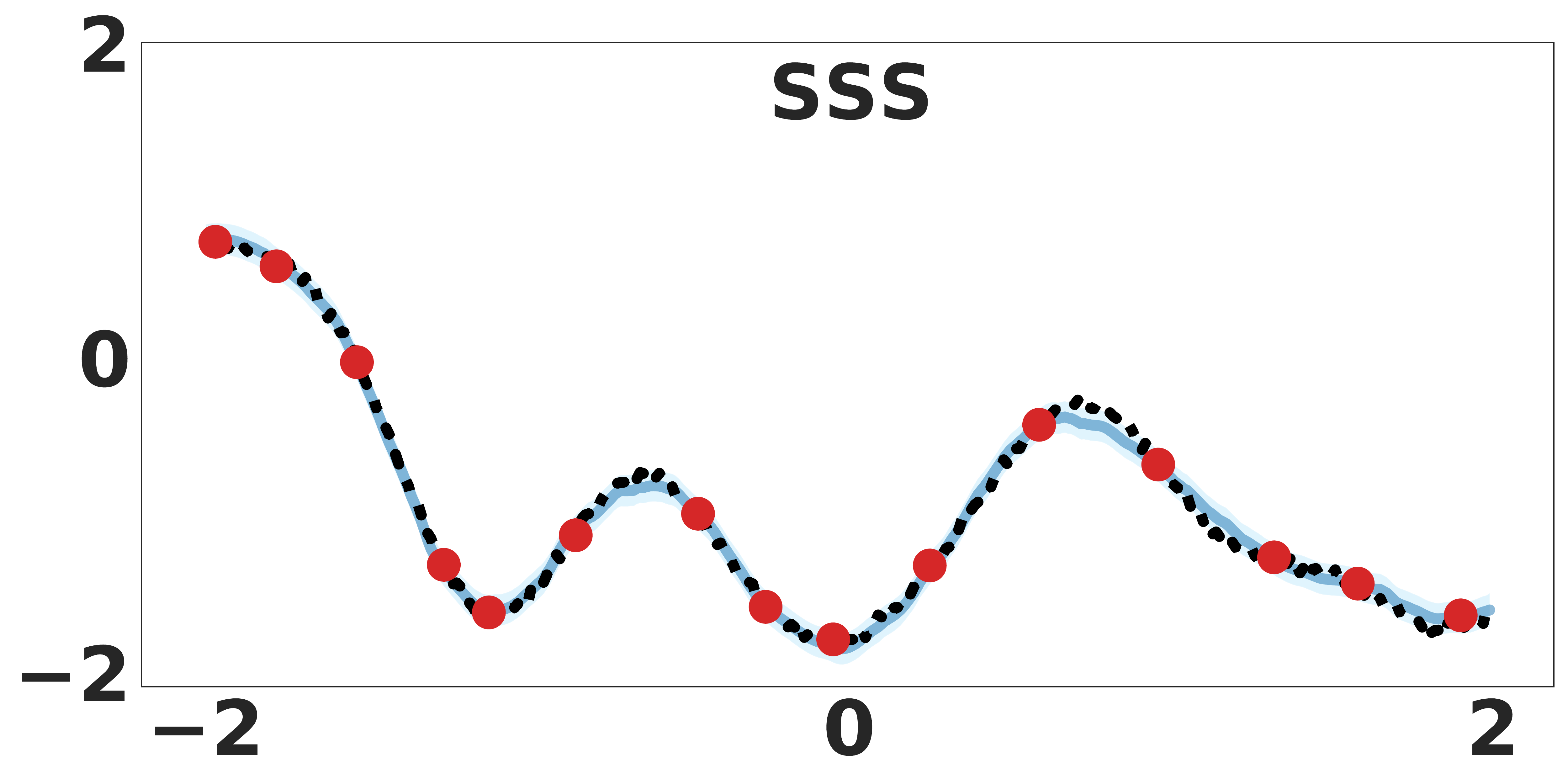}
		\label{reconstruction-function-sss}
	\end{subfigure}

	\vspace{-0.3in}
	\caption{\small Visualization of 1D function reconstruction with three different subset selection models.}
	\label{reconstruction-functions}
	\vspace{-0.1in}
\end{figure*}
\section{Experiments}\label{experiments}
An extensive \textbf{Ablation} on SSS can be found in \textbf{Appendix~\ref{app:ablation}}.

\subsection{Baselines}\label{sec:baselines}
\textbf{Feature Selection for Classification}
We compare SSS with the following models on MNIST. 1) \textbf{Random Selection}: it randomly subsamples features. 2) \textbf{DPS}~\citep{dps}: this model jointly optimizes the sampling parameters along with
the parameters of a task model. 3) \textbf{INVASE}~\citep{invase}: this model uses actor-critic~\citep{actor-critic} to optimize the parameters
of a sampling network and  target task model (Section~\ref{sec411}). 

Additionally, we perform attribute classification on the CelebA dataset using the selected features from a given image. However, we \emph{cannot} apply DPS and INVASE to this experiment. Since each feature is a tuple of 3-d RGB pixels, flattening these features results in ambiguities as to which features to select with INVASE or DPS. For instance, applying these methods to a 3-channel image can result in some channels being selected by the models, while others are zeroed out. It might lead to the entire pixel being preserved and hence violate the subsampling objective. 

\textbf{Feature Selection for Reconstruction}
In Section~\ref{sec41} we compare SSS against the followings. 1) \textbf{Random Selection}. 2) \textbf{LTS}~\citep{dovrat2019learning}:  a model that learns to generate $k$ virtual elements which can be matched to elements in $D$ and optimized for the downstream task. We use LTS for both the function reconstruction and the image reconstruction tasks. DPS and INVASE are not applicable for these tasks since each feature is multi-dimensional. Note that LTS is not applicable for the classification tasks since the
virtual points generated by LTS cannot be converted back into image form to serve as input to an image classifier.

\textbf{Instance Selection}
In Section~\ref{sec4.3}, we compare SSS with 1) \textbf{k-Center-Greedy}: this algorithm iteratively selects elements in $D$ 
closest to a set of centroids and 2) \textbf{FPS}: this algorithm iteratively selects the most distant elements to a randomly 
initialized $D_s$ and \textbf{3) Random Selection} on the instance selection tasks. Here also, DPS, INVASE, and  LTS are all inapplicable. 

\textbf{Multiple Subsampling Rates} For all the neural network based baselines (DPS, INVASE, LTS), we train a separate model for \textit{every} 
subsampling rate. For instance, to select 15, 20, 25, 30, 50 and 100 pixels from MNIST images in Section~\ref{sec411}, we need to train 6 different models for DPS, INVASE and LTS each with the corresponding target subsampling rate. However for SSS, we train a \textit{single} model and vary the sampling rate on each iteration. During evaluation, we use this single model for \textit{all} the different sampling rates. We find that applying a similar training technique to the baselines result in drastic performance degradation. Thus the set formulation of SSS makes it generalize to varying subsampling rates at test time with train time efficiency.

\begin{figure}[H]
\centering
\vspace{-0.1in}
  \includegraphics[width=0.65\linewidth]{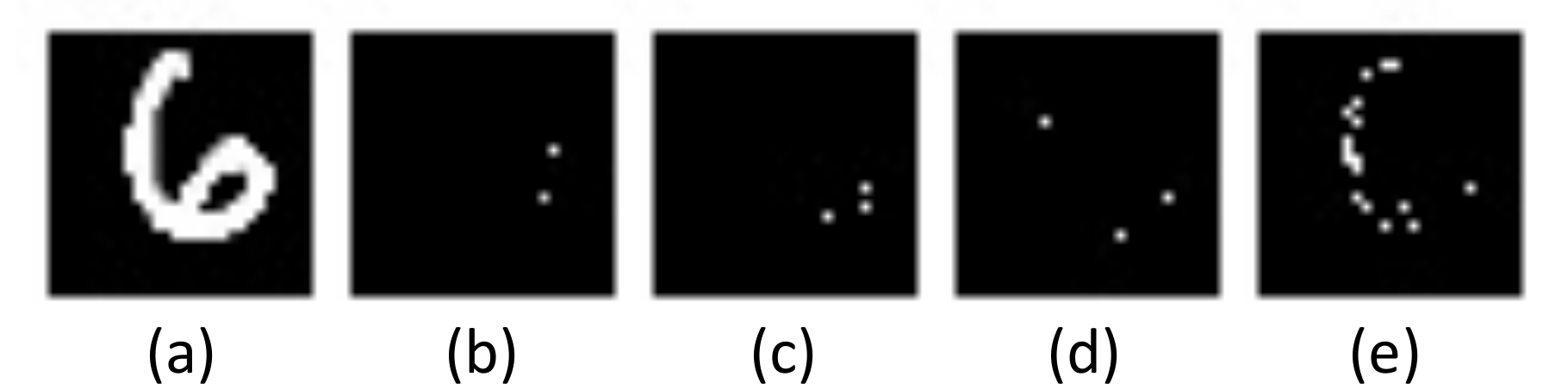}
  \vspace{-0.1in}
  \caption{\small \textbf{(a)} Full-image, \textbf{(b)} Random, \textbf{(c)} DPS, \textbf{(d)} INVASE, \textbf{(e)} SSS. All models select 15 pixels of the original image. Note that under these low subsampling rates, the baseline models end up selecting background pixels as shown in \textbf{(c)} and \textbf{(d)}.}
  \label{mnist_samples}
  \vspace{-0.15in}
\end{figure}
\subsection{Feature Selection for Classification}\label{sec411}
In this subsection, we validate our model on the image classification task with feature selection as illustrated in Fig.~\ref{fig-task-description}-(a). The goal is to select a subset of pixels of an image and predict the label using the chosen subset.

\textbf{MNIST} Given an MNIST image with 784 pixels, the task is to subsamlpe 15, 20, 25, 30, 50 and 100 pixels to be used as input to train and 
evaluate two classification models, a MLP and a ConvNet. We detail the exact architectures in Appendix~\ref{app:mnist_architecture}. Since images in the MNIST dataset have single channel, each feature is of dimension 1 and thus we can train a classifier with DPS or INVASE.
We keep the pixels values for the selected pixels and set all the other pixels to zero. Note that we set the subsampling rates to be much lower than the experimental setup in ~\citet{dps} and ~\citet{invase}. 

As shown in Fig.~\ref{mnist_classification}, SSS significantly outperforms all the baselines with large margin on both the MLP and ConvNet architectures. The MLP is the same architecture used in ~\citet{dps}. However, we test on the full MNIST test set instead of reserving half the test set for validation as done in ~\citet{dps}. SSS reaches  $89\%$ accuracy using only 15 pixels and shows better performance than the baselines with 100 pixels. Moreover under these low subsampling rates, the performance on the baselines are on par with random selection as shown by the ConvNet results in Fig.~\ref{mnist_classification}. Crucially, the performance of SSS is consistent across both the MLP and ConvNet architectures. DPS, which performs relatively well using the MLP shows poor performance on the same dataset using the ConvNet architecture and we observe similar drop in performance as well for INVASE.

Lastly, we provide qualitative results in Fig.~\ref{mnist_samples} where we visualize the 15 selected pixels by all the baselines and SSS. Again we find that SSS selects representative pixels (Fig.~\ref{mnist_samples}-(e)) so that the classifier can predict the correct label of the input image. However, all the baselines tend to select background pixels under these extremely low subsampling rates, which are uninformative for the classification task.

\textbf{CelebA} The CelebA dataset consists of high-quality images of size $218\times 178$. Like the previous experiment, the task is to subsample 100, 200, 300, 400, and 500 pixels from the full $38804$ pixels and perform binary classification for 40 attributes of a face. We use the ConvNet architecture  described in Appendix~\ref{app:celeba_cls_architecture} as the classification network.

We report the mean AUC score on all 40 attributes for varying sizes of $D_s$. Fig.~\ref{celeba_classification} shows that using only 500 pixels ($\sim$1.3\% of total pixels in an image), SSS achieves a mean AUC of 0.9093 (99.3\% of the accuracy obtained with the full image). 
SSS achieves a significantly higher AUC score than Random Selection, showing the effectiveness of our subset selection method. We also include another baseline, namely \textbf{SSS-rec}. This is the SSS model trained for image reconstruction in the Section~\ref{sec:celea_image_reconstruction}, but then later used for classification without any finetuning. Our model also outperforms this variant, showing the effectiveness of training with the target task. Note that we cannot apply LTS, INVASE, or DPS to this experiment. During training, the virtual points generated by LTS cannot be converted back to an image in matrix form due to the virtual coordinate, thus we cannot train the LTS model with CNN-based classification for this task. For DPS and INVASE, they require the dimension of each feature to be 1, thus it is not applicable to multi-channel images. 

\begin{table*}[ht]
\begin{center}
\caption{\small FID Score (the lower is the better) with varying the number of instances}
\resizebox{\linewidth}{!}{
\begin{tabular}{ccccccc}
    \toprule
    \#Instances       & 2  & 5 & 10 & 15 & 20 & 30   \\
    \midrule
    K-Greedy    &   8.8800 $\pm$ 5.5857 & 4.4306 $\pm$ 1.3313 & 4.2199 $\pm$ 1.4214 & 3.7160 $\pm$ 1.1314 & 3.2431 $\pm$ 1.3881 & 2.7554 $\pm$ 0.8554\\
    FPS         &   6.5014 $\pm$ 4.3502 & 4.5098 $\pm$ 2.3809 & 3.0746 $\pm$ 1.0979 & 2.7458 $\pm$ 0.6201 & 2.7118 $\pm$ 1.0410 & 2.2943 $\pm$ 0.8010\\
    Random      &   3.7309 $\pm$ 1.1690 & 1.1575 $\pm$ 0.6532 & 0.8970 $\pm$ 0.4867 & 0.3843 $\pm$ 0.2171 & 0.3877 $\pm$ 0.1906 & 0.1980 $\pm$ 0.1080 \\
    \textbf{SSS}         &   \textbf{2.5307 $\pm$ 1.3583} & \textbf{1.0186 $\pm$ 0.1982} & \textbf{0.5922 $\pm$ 0.3181} & \textbf{0.3331 $\pm$ 0.1169} & \textbf{0.2381 $\pm$ 0.1153} & \textbf{0.1679 $\pm$ 0.0807}\\
    \bottomrule
 \end{tabular}
  }
  \vspace{-0.25in}
  \label{table:fid_table}
\end{center}
\end{table*}

 \begin{table}[ht]
 \begin{center}
 \vspace{-0.1in}
 \caption{\small Accuracy on \textit{mini}ImageNet\label{table:acc_minitable}}
 \resizebox{0.48\textwidth}{!}{
 \begin{tabular}{cccc}
 \toprule
 \#Instances       & 1 & 2 & 5  \\
 \midrule
 FPS		    &   0.432$\pm$0.005 & 0.501$\pm$0.002 & 0.598$\pm$0.000 \\
 Random	    &   0.444$\pm$0.003 & 0.525$\pm$0.005 & 0.618$\pm$0.003 \\
 K-Greedy	    &   0.290$\pm$0.006 & 0.413$\pm$0.005 &	0.570$\pm$0.002 \\
 \textbf{SSS}	    &   \textbf{0.475$\pm$0.006} & \textbf{0.545$\pm$0.011} & \textbf{0.625$\pm$0.006} \\
 \bottomrule
 \end{tabular}
 }
\label{scorefunction}
\end{center}
\vspace{-0.3in}
\end{table}


\subsection{Feature Selection for Regression}\label{sec41}

\textbf{Function Reconstruction}\label{sec:function}
Suppose that we have a function $f: [a,b]\to \mathbb{R}$. We first construct a set of data points with $D=\{(x_1,y_1=f(x_1)), \ldots,(x_n,y_n=f(x_n))\}$, where $(x_1,\ldots,x_n)$ are uniformly sampled from the interval $[a,b]$ and $f$ is a Gaussian process. We sample $(y^{(i)}_1,\ldots, y_n^{(i)}) \stackrel{iid}{\sim} \mathcal{N}(\mathbf{0}, K_{XX} + \sigma^2_y I_n)$ for $i=1,\ldots, N$ where $K_{XX}$ is a squared-exponential kernel with the set of inputs $X=\{x_1,\ldots,x_n\}$ and $\sigma^2_y$ is variance for small likelihood noise. This leads to a collection of sets $(D^{(1)},\ldots,D^{(N)})$. We train our model which consists of the subset selection model $p_\theta(D_s|D)$ and a task network $p_\phi(Y|X,D_s)$, which is an Attentive Neural Process (ANP)~\citep{anp} on this dataset and report the negative log-likelihood (NLL).

Fig.~\ref{function_reconstruction} shows the performance (NLL) of SSS compared to the baselines, Random Selection and LTS. As shown in Fig.~\ref{function_reconstruction}, SSS outperforms the baselines, verifying that the subset selection model $p_\theta(D_s|D)$ learns a meaningful distribution over subsets. We visualize a  reconstructed function and the selected points by each models in Fig.~\ref{reconstruction-functions}. As shown in the rightmost figure (Fig.~\ref{reconstruction-functions}), SSS tends to pick out more elements (red dots) in the drifting parts of the curve, which is reasonable since those are harder to reconstruct than the others. However, the other baselines sometimes fail to do that, which leads to inaccurate reconstructions.

\textbf{CelebA Image Reconstruction}\label{sec:celea_image_reconstruction}
Given an image, we learn to select a representative subset of pixels that best reconstructs the original image. Here, $x_i$ is the 2d pixel coordinates and $y_i\in\mathbb{R}^3$ is the RGB pixel value. We use an ANP to reconstruct the remaining pixels from the subset $D_s=\{x_{i_l}, y_{i_l}\}_{l=1}^k$ constructed by each subsampling model. We conduct the experiment on the CelebA dataset~\citep{liu2018large}. Fig.~\ref{celeba_reconstruction} shows that our model significantly outperforms Random Selection and LTS in terms of NLL. We provide  qualitative examples in \textbf{Appendix}~\ref{celeba-appendix}.

Note that LTS performs worse than random selection. We find that the generated coordinates values by LTS are imprecise. This makes matching the virtual points with the original pixel values extremely difficult. Such inaccurate  coordinate values result in poor performance as the subsampling rate increases even compared to random subsampling as depicted in Figure \ref{celeba_reconstruction}. We observe similar pattern in the CIFAR10 reconstruction task presented in Figure~\ref{cifar10_reconstruction} in \textbf{Appendix}~\ref{app:cifar10}.
On the other hand, for the function reconstruction task in Section 4.3, the point matching stage in LTS is fairly easy and hence the LTS model shows better performance than random subsampling.

\textbf{Efficiency in Nonparametric models} In all the experiments where we used an ANP, we greatly improve the inference 
time complexity. By design, these models need to leverage the full training data at inference time. However by subsampling few highly 
informative instances, we can efficiently perform inference with little degradation in accuracy.
Similar gains can be obtained for models in the Neural Process family of models~\citep{garnelo2018conditional,garnelo2018neural}.

\subsection{Dataset Distillation}\label{sec4.3}
\textbf{Instance Selection} The goal is to select only a few representative images from a given dataset as described in Section~\ref{tasks} and Fig.~\ref{fig-task-description}. We split the CelebA dataset into $m$ disjoint sets $\mathcal{D} = \{D^{(1)},\ldots,D^{(m)}\}$ and jointly train SSS and the generative model to minimize the objective in Eq.~\ref{distill} with $\mathcal{D}$. After training, we discard the generative model and leverage the subsampling model to choose a few representative images from the full CelebA dataset. 

We evaluate the selected subset with the Fréchet Inception Distance (FID)~\cite{fid}, which measures similarity and diversity between two datasets and compare SSS to k-Center-Greedy, FPS and Random Selection. We report the experimental results in Table~\ref{table:fid_table} where SSS achieves the lowest FID score for all selection sizes. Specifically, SSS outperforms all the baselines for selecting very few instances since SSS is able to model the interactions within the dataset and hence selects the most representative subset. Additionally, given that the dataset is highly imbalanced, k-Center-Greedy and FPS perform worst since by selecting extreme or similar elements in the given set and cannot capture the true representation of the full dataset. We provide selected images by SSS from the full dataset in  \textbf{Appendix}~\ref{app:instance_selection}.

\textbf{Classification} In this task, we perform few-shot classification with the \textit{mini}ImageNet dataset~\citep{matching} where the models select 1, 2, or 5 instances from the support set with size 20. As shown in Table \ref{table:acc_minitable}, we compare SSS against Random Selection, FPS, and  k-Center-Greedy. SSS learns to select more representative prototypes than the others especially for small $D_s$ where the choice of prototypes matters more. 
Notably, the K-Greedy method performs poorly for small subset sizes given that the model overfits to a few samples and does not generalize to unseen examples. We show samples of selected prototypes in \textbf{Appendix}~\ref{app:mini_image_net}.
\section{Conclusion}\label{conclusion}
In this paper, we reformulated the subsampling problem as the selection of a subset from a set (e.g features and instances). Based on this reformulation, we proposed a Set based Stochastic Subsampling method that can handle arbitrary input structure as well as variable input set sizes. Additionally, to reduce the cost of modeling pairwise-interactions for large sets, we devised a two-stage subsampling algorithm where we utilize set encoding functions to obtain coarse grained global information in the candidate selection stage followed by a more expressive set interaction network in the autoregressive subset selection stage. We validated the efficacy and generality of our model on various tasks such as feature selection for classification and set reconstruction, instance selection for few shot classification and dataset distillation. We demonstrated that SSS works well and outperforms the relevant baselines.
\section*{Acknowledgement}
This work was supported by Institute of Information \& communications Technology Planning \& Evaluation (IITP) grant funded by the Korea government(MSIT)  (No.2019-0-00075, Artificial Intelligence Graduate School Program(KAIST)), 
the Engineering Research Center Program through the National Research Foundation of Korea (NRF) funded by the Korean Government MSIT (NRF-2018R1A5A1059921), 
Institute of Information \& communications Technology Planning \& Evaluation (IITP) grant funded by the Korea government(MSIT) (No. 2021-0-02068, Artificial Intelligence Innovation Hub), 
the National Research Foundation of Korea (NRF) funded by the Ministry of Education (NRF-2021R1F1A1061655), 
and  Institute of Information \& communications Technology Planning \& Evaluation (IITP) grant funded by the Korea government(MSIT) (No.2022-0-00713).


\bibliography{bibs/bibs}

\begin{thebibliography}{39}
\providecommand{\natexlab}[1]{#1}
\providecommand{\url}[1]{\texttt{#1}}
\expandafter\ifx\csname urlstyle\endcsname\relax
  \providecommand{\doi}[1]{doi: #1}\else
  \providecommand{\doi}{doi: \begingroup \urlstyle{rm}\Url}\fi

\bibitem[Bal{\i}n et~al.(2019)Bal{\i}n, Abid, and Zou]{concreteautoencoders}
Bal{\i}n, M.~F., Abid, A., and Zou, J.
\newblock Concrete autoencoders: Differentiable feature selection and
  reconstruction.
\newblock In \emph{International Conference on Machine Learning}, pp.\
  444--453. PMLR, 2019.

\bibitem[Chen et~al.(2018)Chen, Song, Wainwright, and
  Jordan]{learningtoexplain}
Chen, J., Song, L., Wainwright, M., and Jordan, M.
\newblock Learning to explain: An information-theoretic perspective on model
  interpretation.
\newblock In \emph{International Conference on Machine Learning}, pp.\
  883--892. PMLR, 2018.

\bibitem[Coleman et~al.(2020)Coleman, Yeh, Mussmann, Mirzasoleiman, Bailis,
  Liang, Leskovec, and Zaharia]{coleman2019selection}
Coleman, C., Yeh, C., Mussmann, S., Mirzasoleiman, B., Bailis, P., Liang, P.,
  Leskovec, J., and Zaharia, M.
\newblock Selection via proxy: Efficient data selection for deep learning.
\newblock In \emph{International Conference on Learning Representations}, 2020.

\bibitem[Deng et~al.(2009)Deng, Dong, Socher, Li, Li, and Fei-Fei]{imagenet}
Deng, J., Dong, W., Socher, R., Li, L.-J., Li, K., and Fei-Fei, L.
\newblock Imagenet: A large-scale hierarchical image database.
\newblock In \emph{2009 IEEE conference on computer vision and pattern
  recognition}, pp.\  248--255. Ieee, 2009.

\bibitem[Dovrat et~al.(2019)Dovrat, Lang, and Avidan]{dovrat2019learning}
Dovrat, O., Lang, I., and Avidan, S.
\newblock Learning to sample.
\newblock In \emph{Proceedings of the IEEE Conference on Computer Vision and
  Pattern Recognition}, pp.\  2760--2769, 2019.

\bibitem[Eldar et~al.(1997)Eldar, Lindenbaum, Porat, and Zeevi]{fps}
Eldar, Y., Lindenbaum, M., Porat, M., and Zeevi, Y.~Y.
\newblock The farthest point strategy for progressive image sampling.
\newblock \emph{IEEE Transactions on Image Processing}, 1997.

\bibitem[Feldman(2020)]{feldman2020introduction}
Feldman, D.
\newblock Introduction to core-sets: an updated survey.
\newblock \emph{arXiv preprint arXiv:2011.09384}, 2020.

\bibitem[Gal et~al.(2017)Gal, Hron, and Kendall]{gal2017concrete}
Gal, Y., Hron, J., and Kendall, A.
\newblock Concrete dropout.
\newblock In \emph{Advances in neural information processing systems}, pp.\
  3581--3590, 2017.

\bibitem[Garnelo et~al.(2018{\natexlab{a}})Garnelo, Rosenbaum, Maddison,
  Ramalho, Saxton, Shanahan, Teh, Rezende, and Eslami]{garnelo2018conditional}
Garnelo, M., Rosenbaum, D., Maddison, C., Ramalho, T., Saxton, D., Shanahan,
  M., Teh, Y.~W., Rezende, D., and Eslami, S.~A.
\newblock Conditional neural processes.
\newblock In \emph{International Conference on Machine Learning}, pp.\
  1704--1713. PMLR, 2018{\natexlab{a}}.

\bibitem[Garnelo et~al.(2018{\natexlab{b}})Garnelo, Schwarz, Rosenbaum, Viola,
  Rezende, Eslami, and Teh]{garnelo2018neural}
Garnelo, M., Schwarz, J., Rosenbaum, D., Viola, F., Rezende, D.~J., Eslami, S.,
  and Teh, Y.~W.
\newblock Neural processes.
\newblock \emph{arXiv preprint arXiv:1807.01622}, 2018{\natexlab{b}}.

\bibitem[Heusel et~al.(2017)Heusel, Ramsauer, Unterthiner, Nessler, and
  Hochreiter]{fid}
Heusel, M., Ramsauer, H., Unterthiner, T., Nessler, B., and Hochreiter, S.
\newblock Gans trained by a two time-scale update rule converge to a local nash
  equilibrium.
\newblock In \emph{Advances in neural information processing systems}, pp.\
  6626--6637, 2017.

\bibitem[Huijben et~al.(2019)Huijben, Veeling, and van Sloun]{dps}
Huijben, I.~A., Veeling, B.~S., and van Sloun, R.~J.
\newblock Deep probabilistic subsampling for task-adaptive compressed sensing.
\newblock In \emph{International Conference on Learning Representations}, 2019.

\bibitem[Jang et~al.(2017)Jang, Gu, and Poole]{jang2016categorical}
Jang, E., Gu, S., and Poole, B.
\newblock Categorical reparameterization with gumbel-softmax.
\newblock In \emph{5th International Conference on Learning Representations},
  2017.

\bibitem[Kim et~al.(2019)Kim, Mnih, Schwarz, Garnelo, Eslami, Rosenbaum,
  Vinyals, and Teh]{anp}
Kim, H., Mnih, A., Schwarz, J., Garnelo, M., Eslami, A., Rosenbaum, D.,
  Vinyals, O., and Teh, Y.~W.
\newblock Attentive neural processes.
\newblock In \emph{International Conference on Learning Representations}, 2019.

\bibitem[Krizhevsky et~al.(2009)Krizhevsky, Nair, and Hinton]{cifar}
Krizhevsky, A., Nair, V., and Hinton, G.
\newblock Cifar-10 and cifar-100 datasets.
\newblock \emph{URl: https://www. cs. toronto. edu/kriz/cifar. html}, 2009.

\bibitem[Lee et~al.(2019)Lee, Lee, Kim, Kosiorek, Choi, and Teh]{lee2018set}
Lee, J., Lee, Y., Kim, J., Kosiorek, A., Choi, S., and Teh, Y.~W.
\newblock Set transformer: A framework for attention-based
  permutation-invariant neural networks.
\newblock In \emph{International Conference on Machine Learning}, pp.\
  3744--3753. PMLR, 2019.

\bibitem[Li et~al.(2018{\natexlab{a}})Li, Zuo, Gu, Zhao, and
  Zhang]{li2018learning}
Li, M., Zuo, W., Gu, S., Zhao, D., and Zhang, D.
\newblock Learning convolutional networks for content-weighted image
  compression.
\newblock In \emph{Proceedings of the IEEE Conference on Computer Vision and
  Pattern Recognition}, pp.\  3214--3223, 2018{\natexlab{a}}.

\bibitem[Li et~al.(2018{\natexlab{b}})Li, Bu, Sun, Wu, Di, and Chen]{pointcnn}
Li, Y., Bu, R., Sun, M., Wu, W., Di, X., and Chen, B.
\newblock Pointcnn: Convolution on x-transformed points.
\newblock In \emph{Advances in neural information processing systems}, pp.\
  820--830, 2018{\natexlab{b}}.

\bibitem[Liu et~al.(2015)Liu, Luo, Wang, and Tang]{celeba}
Liu, Z., Luo, P., Wang, X., and Tang, X.
\newblock Deep learning face attributes in the wild.
\newblock In \emph{Proceedings of International Conference on Computer Vision
  (ICCV)}, December 2015.

\bibitem[Liu et~al.(2018)Liu, Luo, Wang, and Tang]{liu2018large}
Liu, Z., Luo, P., Wang, X., and Tang, X.
\newblock Large-scale celebfaces attributes (celeba) dataset.
\newblock \emph{Retrieved August}, 2018.

\bibitem[Maddison et~al.(2017)Maddison, Mnih, and Teh]{maddison2016concrete}
Maddison, C.~J., Mnih, A., and Teh, Y.~W.
\newblock The concrete distribution: {A} continuous relaxation of discrete
  random variables.
\newblock In \emph{5th International Conference on Learning Representations},
  2017.

\bibitem[Mentzer et~al.(2018)Mentzer, Agustsson, Tschannen, Timofte, and
  Van~Gool]{mentzer2018conditional}
Mentzer, F., Agustsson, E., Tschannen, M., Timofte, R., and Van~Gool, L.
\newblock Conditional probability models for deep image compression.
\newblock In \emph{Proceedings of the IEEE Conference on Computer Vision and
  Pattern Recognition}, pp.\  4394--4402, 2018.

\bibitem[Moenning \& Dodgson(2003)Moenning and Dodgson]{fastfps}
Moenning, C. and Dodgson, N.~A.
\newblock Fast marching farthest point sampling.
\newblock Technical report, University of Cambridge, Computer Laboratory, 2003.

\bibitem[Peters \& Schaal(2008)Peters and Schaal]{actor-critic}
Peters, J. and Schaal, S.
\newblock Natural actor-critic.
\newblock \emph{Neurocomputing}, 71\penalty0 (7-9):\penalty0 1180--1190, 2008.

\bibitem[Qi et~al.(2017{\natexlab{a}})Qi, Su, Mo, and Guibas]{pointnet}
Qi, C.~R., Su, H., Mo, K., and Guibas, L.~J.
\newblock Pointnet: Deep learning on point sets for 3d classification and
  segmentation.
\newblock In \emph{Proceedings of the IEEE conference on computer vision and
  pattern recognition}, pp.\  652--660, 2017{\natexlab{a}}.

\bibitem[Qi et~al.(2017{\natexlab{b}})Qi, Su, Mo, and Guibas]{qi2017pointnet}
Qi, C.~R., Su, H., Mo, K., and Guibas, L.~J.
\newblock Pointnet: Deep learning on point sets for 3d classification and
  segmentation.
\newblock In \emph{Proceedings of the IEEE Conference on Computer Vision and
  Pattern Recognition}, pp.\  652--660, 2017{\natexlab{b}}.

\bibitem[Qi et~al.(2017{\natexlab{c}})Qi, Yi, Su, and Guibas]{pointnet++}
Qi, C.~R., Yi, L., Su, H., and Guibas, L.~J.
\newblock Pointnet++: Deep hierarchical feature learning on point sets in a
  metric space.
\newblock In \emph{Advances in neural information processing systems}, pp.\
  5099--5108, 2017{\natexlab{c}}.

\bibitem[Ravanbakhsh et~al.(2016)Ravanbakhsh, Schneider, and
  Poczos]{ravanbakhsh2016deep}
Ravanbakhsh, S., Schneider, J., and Poczos, B.
\newblock Deep learning with sets and point clouds.
\newblock \emph{arXiv preprint arXiv:1611.04500}, 2016.

\bibitem[Ribeiro et~al.(2016)Ribeiro, Singh, and Guestrin]{lime}
Ribeiro, M.~T., Singh, S., and Guestrin, C.
\newblock ``why should i trust you?" explaining the predictions of any
  classifier.
\newblock In \emph{Proceedings of the 22nd ACM SIGKDD international conference
  on knowledge discovery and data mining}, pp.\  1135--1144, 2016.

\bibitem[Rippel \& Bourdev(2017)Rippel and Bourdev]{rippel2017real}
Rippel, O. and Bourdev, L.
\newblock Real-time adaptive image compression.
\newblock In \emph{Proceedings of the 34th International Conference on Machine
  Learning-Volume 70}, pp.\  2922--2930. JMLR. org, 2017.

\bibitem[Sannai et~al.(2019)Sannai, Takai, and Cordonnier]{sannai2019universal}
Sannai, A., Takai, Y., and Cordonnier, M.
\newblock Universal approximations of permutation invariant/equivariant
  functions by deep neural networks.
\newblock \emph{arXiv preprint arXiv:1903.01939}, 2019.

\bibitem[Sener \& Savarese(2018)Sener and Savarese]{sener2017active}
Sener, O. and Savarese, S.
\newblock Active learning for convolutional neural networks: A core-set
  approach.
\newblock In \emph{International Conference on Learning Representations}, 2018.

\bibitem[Snell et~al.(2017)Snell, Swersky, and Zemel]{protonet}
Snell, J., Swersky, K., and Zemel, R.
\newblock Prototypical networks for few-shot learning.
\newblock In \emph{Advances in neural information processing systems}, pp.\
  4077--4087, 2017.

\bibitem[Toderici et~al.(2017)Toderici, Vincent, Johnston, Jin~Hwang, Minnen,
  Shor, and Covell]{toderici2017full}
Toderici, G., Vincent, D., Johnston, N., Jin~Hwang, S., Minnen, D., Shor, J.,
  and Covell, M.
\newblock Full resolution image compression with recurrent neural networks.
\newblock In \emph{Proceedings of the IEEE Conference on Computer Vision and
  Pattern Recognition}, pp.\  5306--5314, 2017.

\bibitem[Vaswani et~al.(2017)Vaswani, Shazeer, Parmar, Uszkoreit, Jones, Gomez,
  Kaiser, and Polosukhin]{transformer}
Vaswani, A., Shazeer, N., Parmar, N., Uszkoreit, J., Jones, L., Gomez, A.~N.,
  Kaiser, L., and Polosukhin, I.
\newblock Attention is all you need.
\newblock In \emph{NIPS}, 2017.

\bibitem[Vinyals et~al.(2016)Vinyals, Blundell, Lillicrap, Wierstra,
  et~al.]{matching}
Vinyals, O., Blundell, C., Lillicrap, T., Wierstra, D., et~al.
\newblock Matching networks for one shot learning.
\newblock In \emph{Advances in neural information processing systems}, pp.\
  3630--3638, 2016.

\bibitem[Wei et~al.(2015)Wei, Iyer, and Bilmes]{wei2015submodularity}
Wei, K., Iyer, R., and Bilmes, J.
\newblock Submodularity in data subset selection and active learning.
\newblock In \emph{International Conference on Machine Learning}, pp.\
  1954--1963, 2015.

\bibitem[Yoon et~al.(2019)Yoon, Jordon, and van~der Schaar]{invase}
Yoon, J., Jordon, J., and van~der Schaar, M.
\newblock {INVASE}: Instance-wise variable selection using neural networks.
\newblock In \emph{International Conference on Learning Representations}, 2019.

\bibitem[Zaheer et~al.(2017)Zaheer, Kottur, Ravanbakhsh, Poczos, Salakhutdinov,
  and Smola]{zaheer2017deep}
Zaheer, M., Kottur, S., Ravanbakhsh, S., Poczos, B., Salakhutdinov, R.~R., and
  Smola, A.~J.
\newblock Deep sets.
\newblock In \emph{Advances in neural information processing systems}, pp.\
  3391--3401, 2017.

\end{thebibliography}
\bibliographystyle{icml2022}

\newpage
\clearpage
\appendix

\section{Organization}\label{app:organization}

\paragraph{Organization} The Appendix is organized as follows: first, we describe the pseudo-code for the Training Algorithm, then provide ablation for each components of our method and proofs for propositions described in Section~\ref{sec:approach}. Finally, we illustrate the generative process of each task using graphical models and and additional experimental results with detailed elaboration on experimental setups.

\begin{algorithm}[H]
   \caption{\small Greedy Training Algorithm}
   \label{greedy-training}
    \begin{tabularx}{\textwidth}{ll}
 	\textbf{Input} & $k$ (max subset size) \\ 
 	&$m$ (mini-batch size) \\ 
 	&$p(D)$ (distribution of sets) \\ 
 	&$\alpha$ (learning rate)\\
 	&$\ell(\cdot,D_s)$ (loss function) \\
 	\textbf{Output} & trained models with  $\theta$ and $\phi$
  \end{tabularx}

  \begin{algorithmic}[1]
     \STATE Randomly initialize parameter of SSS $\theta$ and downstream task model $\phi$. 
     \WHILE{not converged}
        \STATE Sample $m$ sets $D^{(1)},\ldots,D^{(m)}$ from $p(D)$
        \STATE Sample $Z^{(j)} = \{z^{(j)}_1, \ldots, z^{(j)}_n\} \sim p_\theta(Z|D^{(j)})$ for $j=1,\ldots, m$  
        \STATE Construct $D^{(j)}_c=\{d_i^{(j)}\in D^{(j)}:  z_i^{(j)} = 1\}$ for $j=1,\ldots, m$  
        \STATE Sample integer $l\sim\text{Unif}[1,k]$
        \STATE $D^{(j)}_s \gets$ select  $l$-elements from $D^{(j)}_c $ (with the auto-regressive model)
        \STATE $\theta \gets \theta - \alpha\nabla_\theta \frac{1}{m} \sum_{j=1}^m\ell(\cdot,D_s^{(j)})$ 
        \STATE $\phi \gets \phi - \alpha\nabla_\phi \frac{1}{m} \sum_{j=1}^m\ell(\cdot, D^{(j)}_s)$
    \ENDWHILE
  \end{algorithmic}
\end{algorithm}


\section{Training Algorithm}\label{greedy} In order to reduce the computational cost at training time, we use a greedy training algorithm with stochastic gradient descent as described in  Algorithm~\ref{greedy-training}. First, we uniformly sample an integer $l$ from $\{1,\ldots, k\}$, which is the subset size for a given mini-batch. Then we select $l$ elements for $D_S$ from the candidate set \emph{at once} using the autoregressive selection model. Finally, we perform gradient descent with respect to the parameters of SSS model $\theta$ and target task network $\phi$ to minimize the target task loss on the selected subset $D_s$. As a result, we do not have to run the auto-regressive model $k$ time during training, which significantly reduces the computational cost during training.

\begin{algorithm}[ht]
    \caption{Fixed Size Subsampling. $k$ is the subset size. $l$ is the number of elements to select at each iteration. $D$ is the full set and $D_s \subset D$ is the final subset after running SSS.}
    \label{fixed-size-selection}
    \begin{algorithmic}[1]
    	\STATE \textbf{Input:}  $k$, $l$, $D=\{d_1,\ldots,d_n\}$
    	\STATE \textbf{Output:} $D_s=\{s_1,\ldots, s_k\}$ 
    	
	\FUNCTION{SSS$(k, l, D)$}
    	    \STATE $D_e \leftarrow \frac{1}{n}\sum_{i=1}^n g(d_i)$
    	    \STATE $\overline{d_i} \leftarrow \text{Concat}(g(d_i), D_e)$
            \STATE $z_i \sim \text{Ber}(z_i;\rho(\overline{d_i}) \text{ for } i=1, \ldots, n$
            \STATE $D_c \leftarrow \{d_i \in D\mid z_i = 1, 1\leq i\leq n  \}$
            \STATE $D_s  \leftarrow \emptyset$
            \FOR{$t=1$ {\bfseries to} $k/l$}
                \STATE $D_s \leftarrow D_s \cup \text{AUTOSELECT}(l,D_s, D_c)$
            \ENDFOR
        \ENDFUNCTION
        
        \FUNCTION{AUTOSELECT$(l, D^{(t-1)}_s, D_c)$}
            \STATE $D_c^{(t)} = \{s^{(t)}_1,\ldots, s^{(t)}_{m_t}\} \leftarrow D_c \setminus D^{(t)}_s$
            \STATE $\Tilde{\pi}^{(t)}_i \leftarrow \sigma(\varphi \circ f(s_i, D^{(t-1)}_s))$
            \STATE $(\pi^{(t)}_1, \ldots, \pi^{(t)}_{m_t}) \leftarrow \frac{1}{\sum_{j=1}^{m_t} \Tilde{\pi}^{(t)}_j}(\Tilde{\pi}^{(t)}_1, \ldots, \Tilde{\pi}^{(t)}_{m_t})$
            \STATE \begin{varwidth}[t]{\linewidth} $Q \leftarrow \text{Sample } l \text{ elements from } D^{(t)}_c$ \par $\text{ with the probability } \pi^{(t)}$
            \end{varwidth}
            \STATE $\text{return } Q$
        \ENDFUNCTION
    \end{algorithmic}
\end{algorithm}

\section{Fixed-size Subset Selection}\label{fixed-size-ss}
At test time,  we run the fixed size subset selection algorithm to choose the most task relevant elements from the  set $D$ as described in Algorithm~\ref{fixed-size-selection}. We do not use the greedy training algorithm. Instead, we autoregressively select $k$ elements from the candidate set $D_c$ as described in line 13 from Algorithm~\ref{fixed-size-selection} to construct the representative subset $D_s$. 

\section{Ablation}\label{app:ablation}
We perform extensive ablation studies on the two-stage Set based Stochastic Susbsampling method using the function reconstruction tasks presented in Section ~\ref{sec:function}. First, we explore 
the contribution of the candidate selection and autoregressive subset selection stages. Then, we verify the importance of stochasticity in SSS. 

\paragraph{Random Selection with Autoregressive Subset Selection} To show the importance of the candidate selection stage of SSS, we replace it with random selection (labelled \textbf{Random + Stage 2} in Fig.~\ref{table:acc_minitable}). As shown in Fig.~\ref{fig:ablation}, we find that while this model performs better than the model with only candidate selection, it performs worse than SSS (red line in Fig.~\ref{fig:ablation}) and the autoregressive subset selection stage used alone (green line in Fig.~\ref{fig:ablation}). We provide visualizations of the reconstructed functions in the second column of  Fig.~\ref{function_reconstruction_ablation}.  Random selection in place of the candidlate selection can ignore elements from certain parts of the function and hence the autoregressive selection model  cannot select elements from those regions for reconstruction. In short, filtering elements with candidate selection helps the autoregressive selection model to choose more informative instances from the input set than random selection.

\paragraph{AutoRegressive Subset Selection Only} 
As shown in Fig.~\ref{fig:ablation} we observe that the autoreressive subset selection model (labelled \textbf{Stage 2 Only} in Fig.~\ref{fig:ablation}) performs significantly better than the SSS model with candidate selection and autoregressive selection stage. Qualitative results are provided in third column of Fig.~\ref{function_reconstruction_ablation}. While this model performs well, it is not very practical due to the high computational cost when the size of the set becomes large.

\paragraph{Candidate Selection Only} In order to validate the importance of the autoregressive selection stage, we construct a subset using only the candidate selection stage. As shown in  Fig.~\ref{fig:ablation} (labelled \textbf{Stage 1 Only}), removing the autoregressive selection stage significantly degrades the performance of the SSS model. In the first column of Fig.~\ref{function_reconstruction_ablation}, the model without autoregressive selection significantly underperforms compared to SSS since it heavily focuses on the drifting parts of the function and ignores the other parts of curve. In sum, it is not always desirable to select only highly activating samples in the set without considering any dependencies among the others since it may choose redundant elements. 

\begin{figure}[t]
    \centering
    \includegraphics[width=\linewidth]{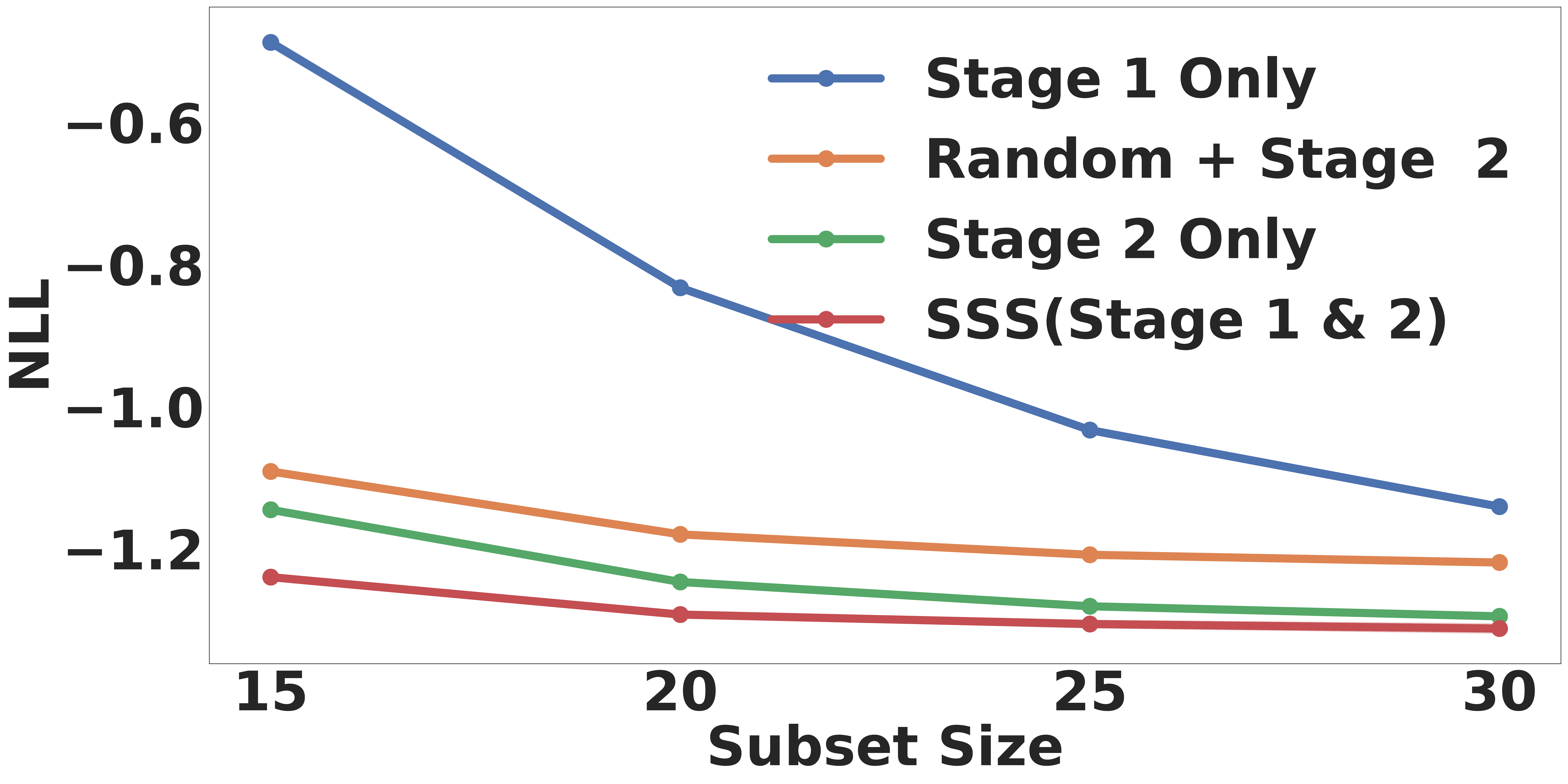}
    \vspace{-0.2in}
    \caption{\small Ablation on SSS.}
    \label{fig:ablation}
\end{figure}

Generally, we find that SSS performs better than the variants considered here and provide a better tradeoff between model performance and computational requirements.

\begin{table}[t]
\begin{center}
\caption{\small CelebA Attributes Classification.}\label{tab:celeba-classification}
\begin{tabular}{cccc}
    \toprule
    Model  & \# Pixels & Storage  & mAUC   \\
    \midrule
    Full Image & All 38804 & 114KB &  0.9157  \\
    RS &  500 & 5KB &  0.8471  \\
    SSS(rec) &  500 & 5KB &  0.8921  \\
    SSS(MC) &  500 & 5*5KB &  0.9132  \\
    SSS(ours) &  500 & 5KB & 0.9093  \\
    \bottomrule
  \end{tabular}
\end{center}
\end{table}
\begin{figure}[t]
    \centering
    \includegraphics[width=0.7\linewidth]{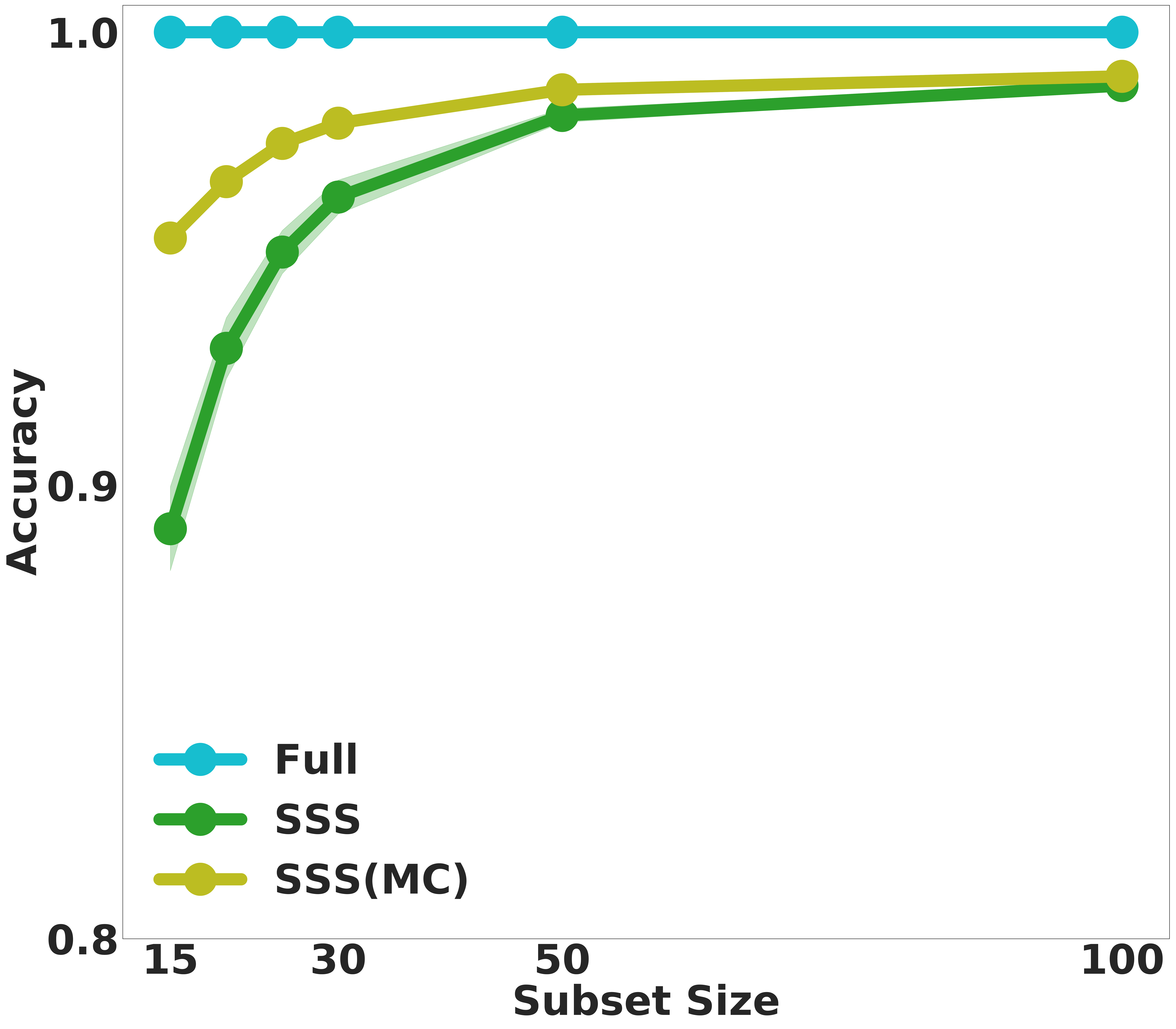}
    \vspace{-0.15in}
    \caption{\small Accuracy with varying subset size and the number of particles for MCMC.}
    \label{fig:mnist_mcmc}
\end{figure}
\paragraph{Stochasticity of SSS} Since our method is stochastic with the following predictive distribution: $\mathbb{E}_{p_\theta(D_s|D)} [p_\phi(y_D|D_s)]$, we approximate it with Monte Carlo sampling as follows:
\begin{equation}
\begin{gathered}
    \mathbb{E}_{p_\theta(D_s|D)} [p_\phi(y_D|D_s)] \approx \frac{1}{n}\sum_{i=1}^n p_\phi (y_D|D^{(i)}_s), \\ 
    \text{where } D^{(i)}_s \stackrel{i.i.d.}{\sim} p_\theta (D_s|D),
\end{gathered}
\end{equation}
However in \textit{all} experiments, we only report the result with one sampled subset, since it gives the best trade-off between computational cost and accuracy. We compare SSS against another variant, \textbf{SSS-MC}, which use 5 sampled subsets for MC sampling. As shown in Fig.~\ref{celeba_classification}, it obtains a mean AUC of $91.32\%$, which is slightly better than SSS which achieves $90.93\%$. Note SSM-MC increases the computational cost (inference) and memory requirement up to 5 times as shown in Table~\ref{tab:celeba-classification}. The result justifies the inference procedure of SSS which achieves good performance for the target tasks with memory and computational efficiency. 

Additionally, in Figure~\ref{fig:mnist_mcmc} we show how the uncertainty of SSS decreases as we increases the subset size. For each subset size, we draw 5 different subsets from the subsampling model trained on the MNIST classification task described in Section~\ref{sec411} and average the predictions of the sampled subsets. This requires 5 forward pass of the model. At a subsampling rate of 50, the model with a single subset shows similar performance with the one with multiple draws of subsets.

\begin{figure*}
\vspace{-0.1in}
    \centering
    \begin{subfigure}{0.5\textwidth}
		\centering
		\includegraphics[width=\linewidth]{images/function/plots/legend_func.pdf}
		\vspace{-0.2in}
	\end{subfigure}
    \begin{subfigure}{0.25\textwidth}
		\centering
		\includegraphics[width=\linewidth]{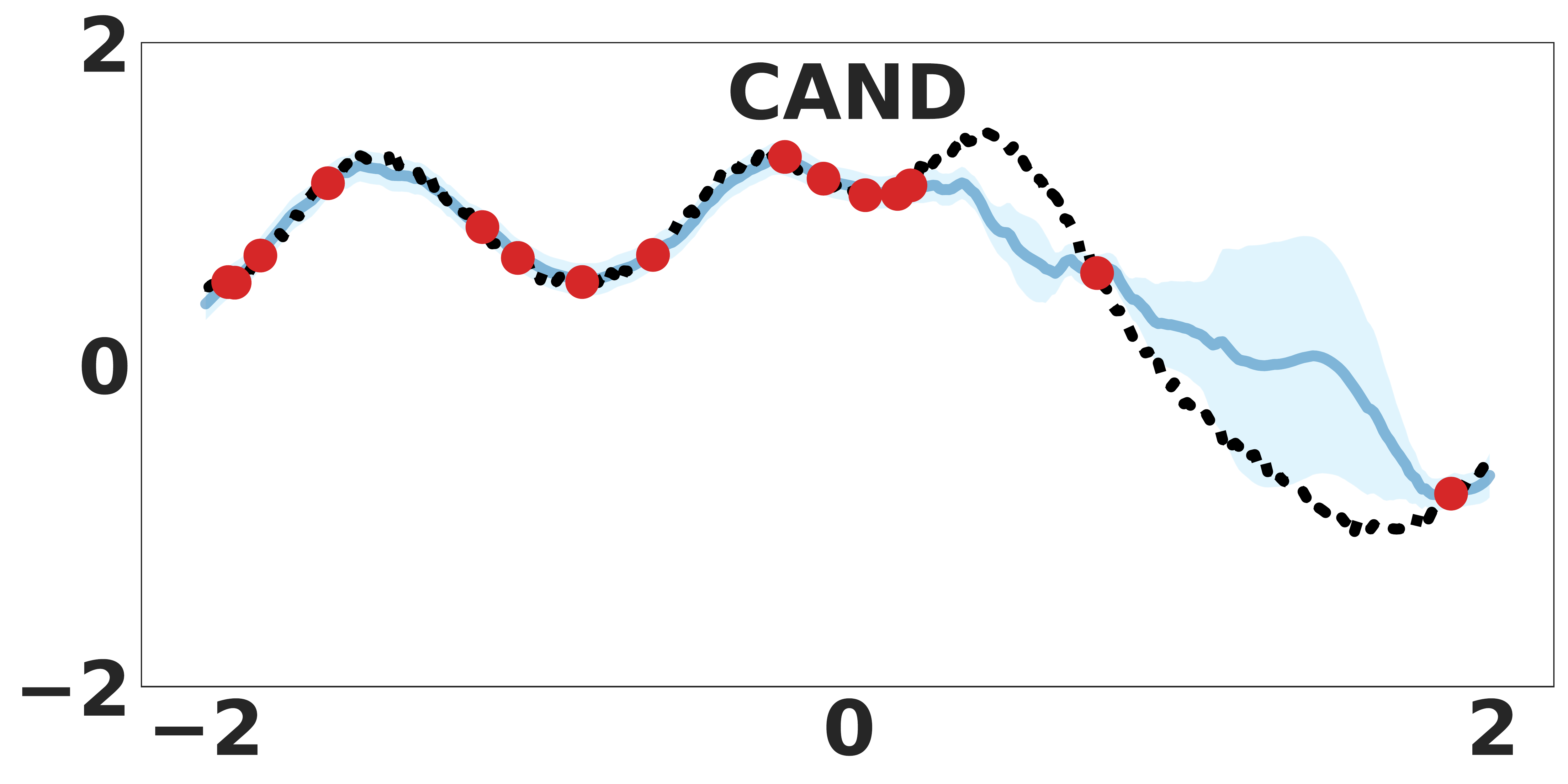}
	\end{subfigure}%
    \begin{subfigure}{0.25\textwidth}
		\centering
		\includegraphics[width=\linewidth]{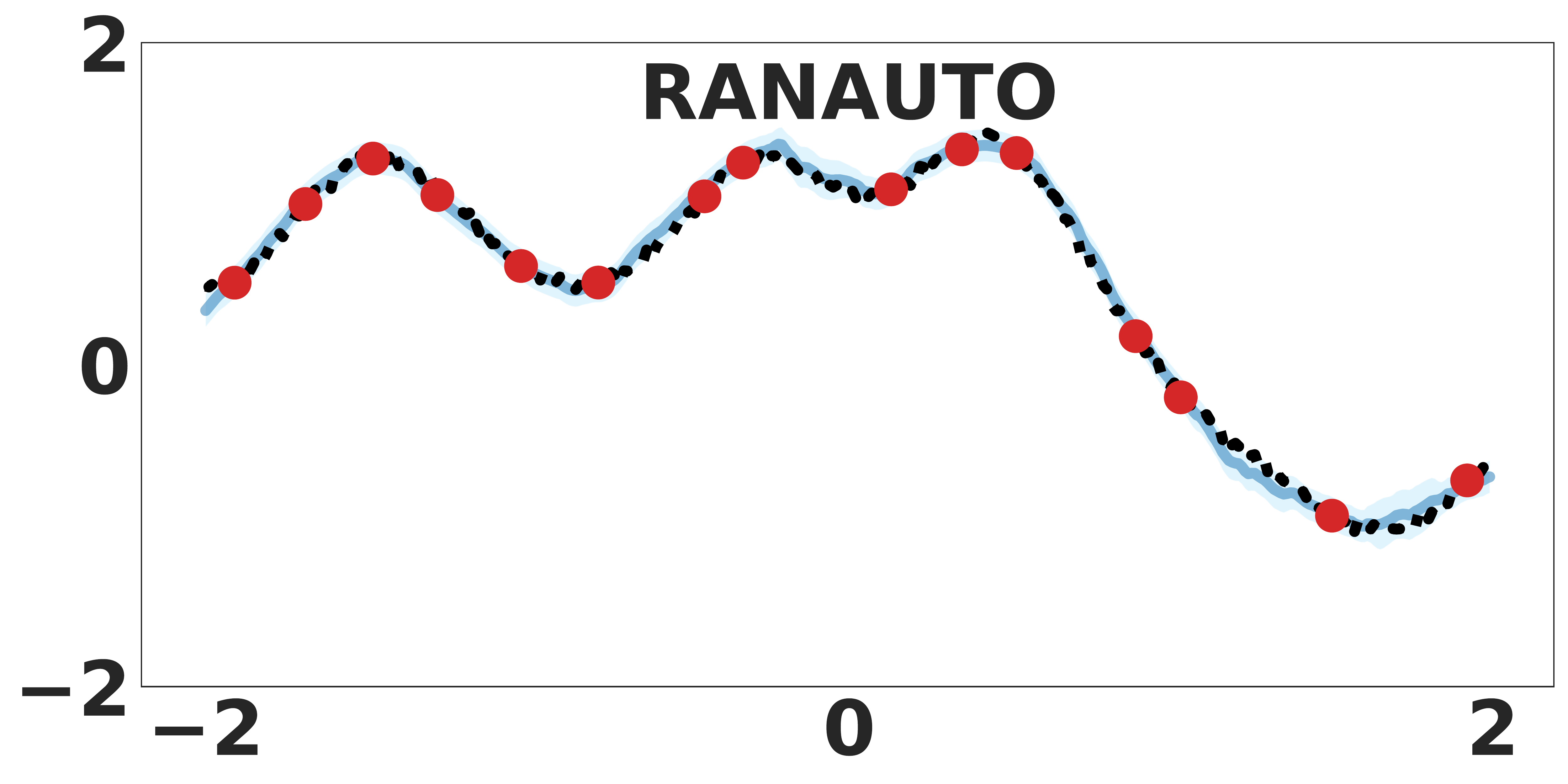}
	\end{subfigure}%
    \begin{subfigure}{0.25\textwidth}
		\centering
		\includegraphics[width=\linewidth]{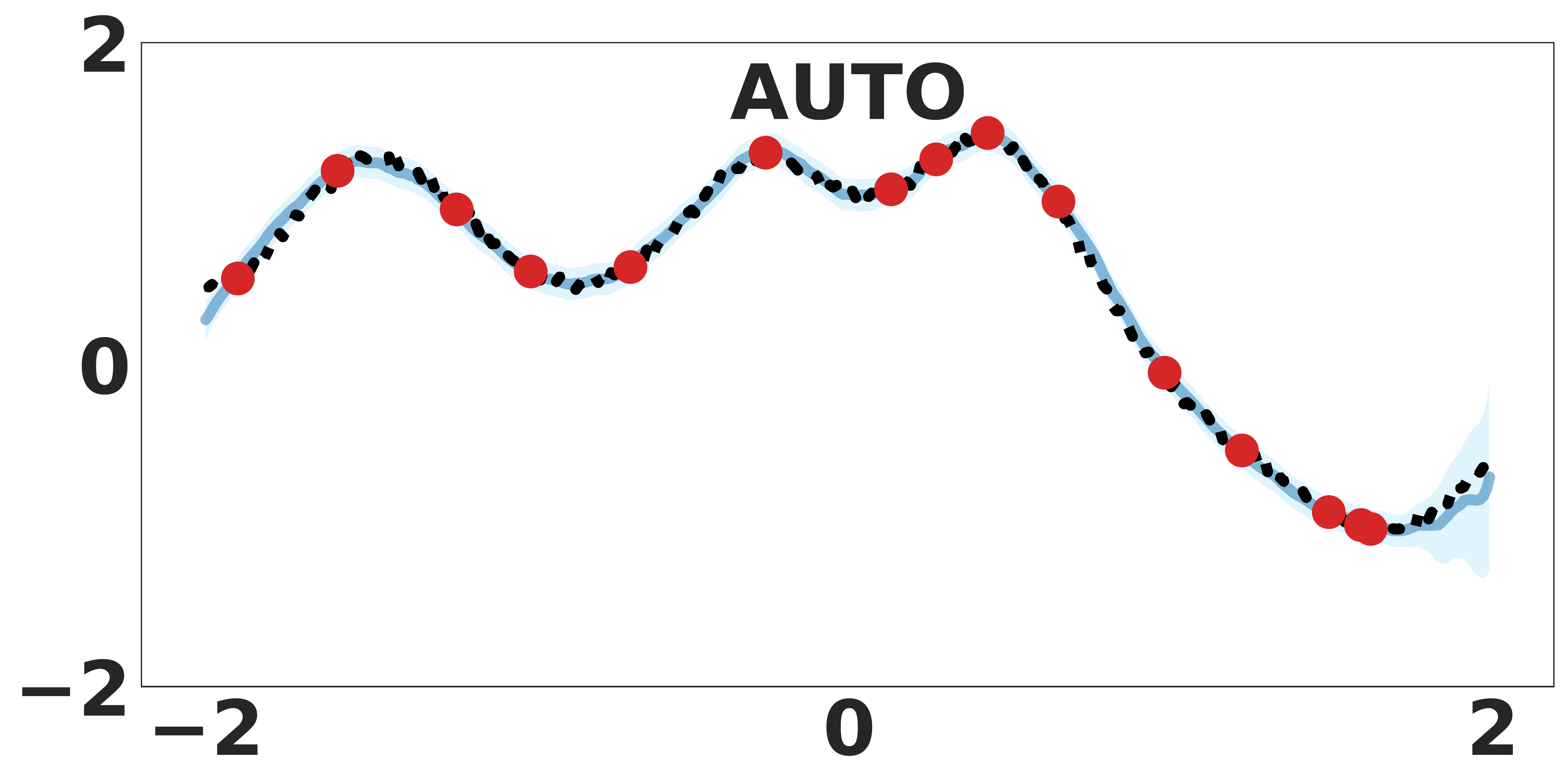}
	\end{subfigure}%
	\begin{subfigure}{0.25\textwidth}
		\centering
		\includegraphics[width=\linewidth]{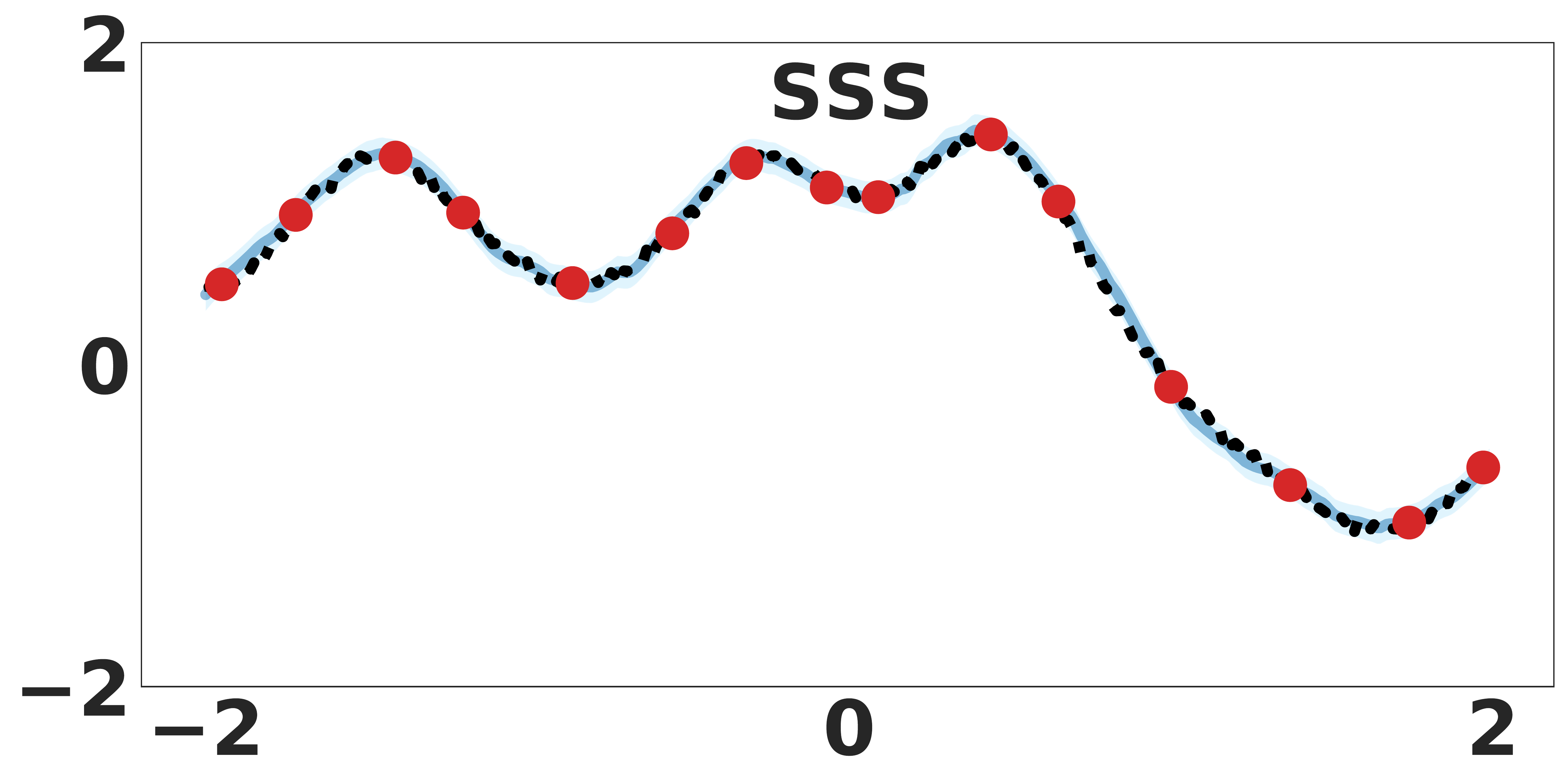}
	\end{subfigure}
    \begin{subfigure}{0.25\textwidth}
		\centering
		\includegraphics[width=\linewidth]{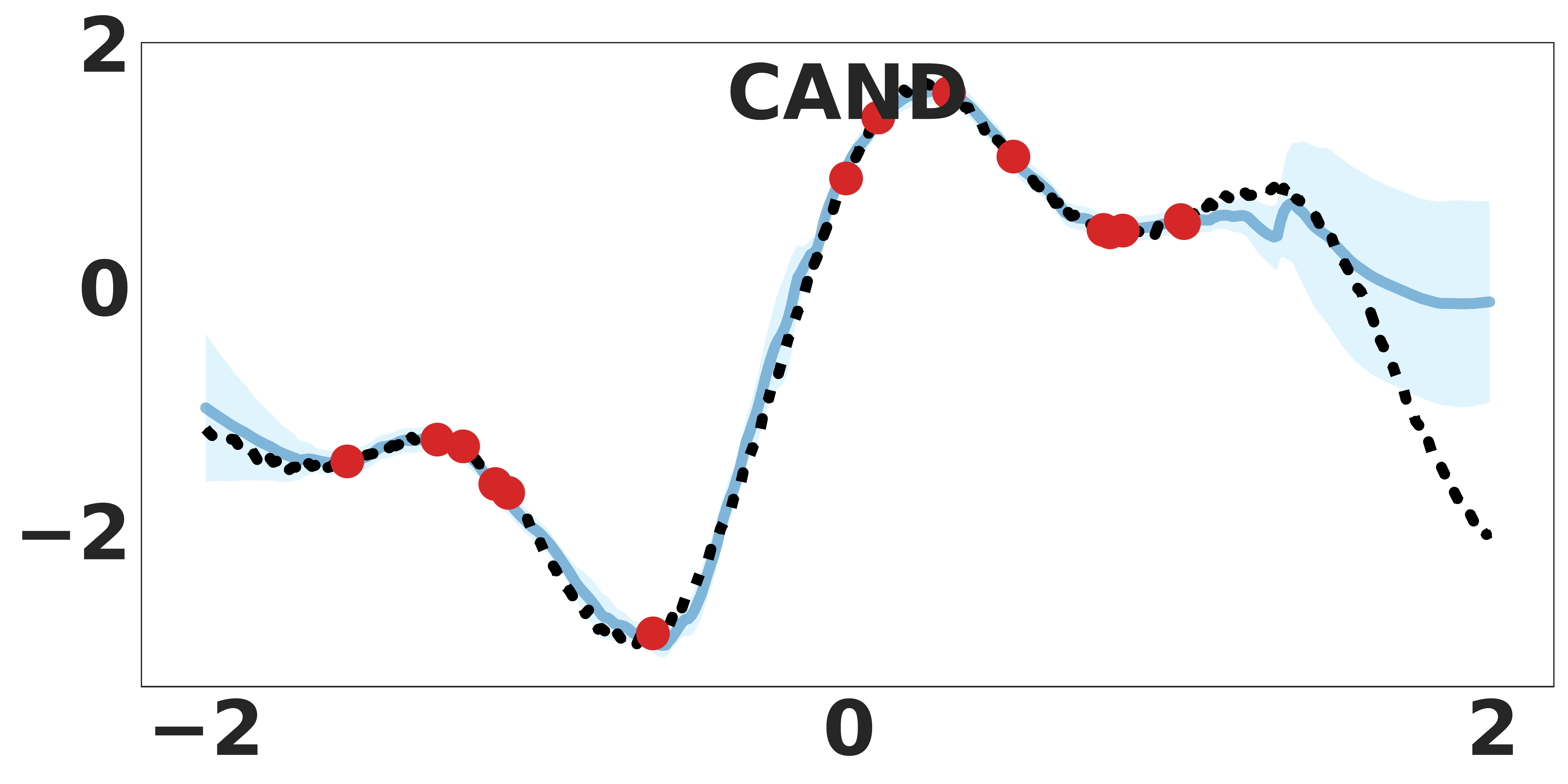}
	\end{subfigure}%
    \begin{subfigure}{0.25\textwidth}
		\centering
		\includegraphics[width=\linewidth]{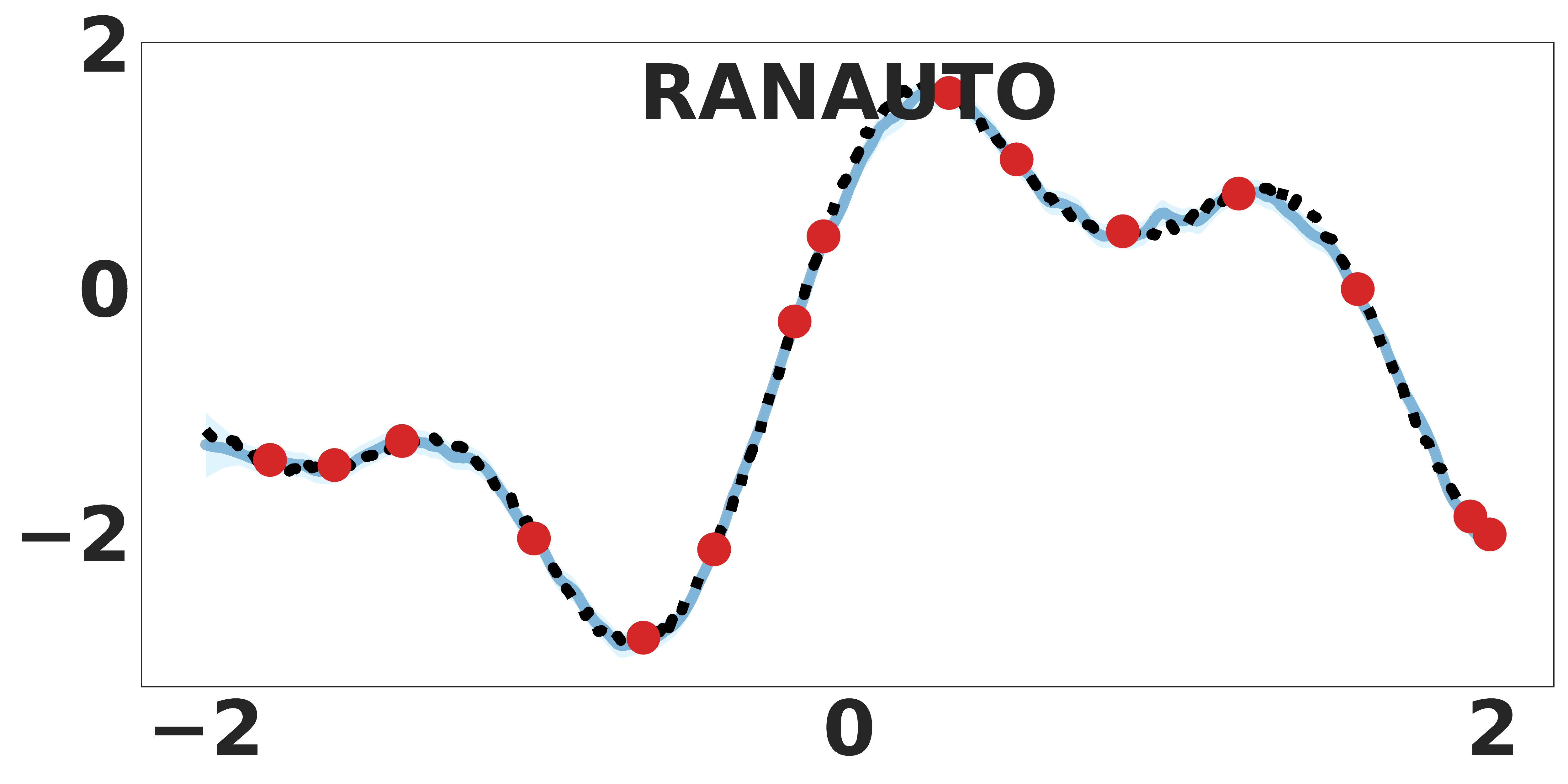}
	\end{subfigure}%
    \begin{subfigure}{0.25\textwidth}
		\centering
		\includegraphics[width=\linewidth]{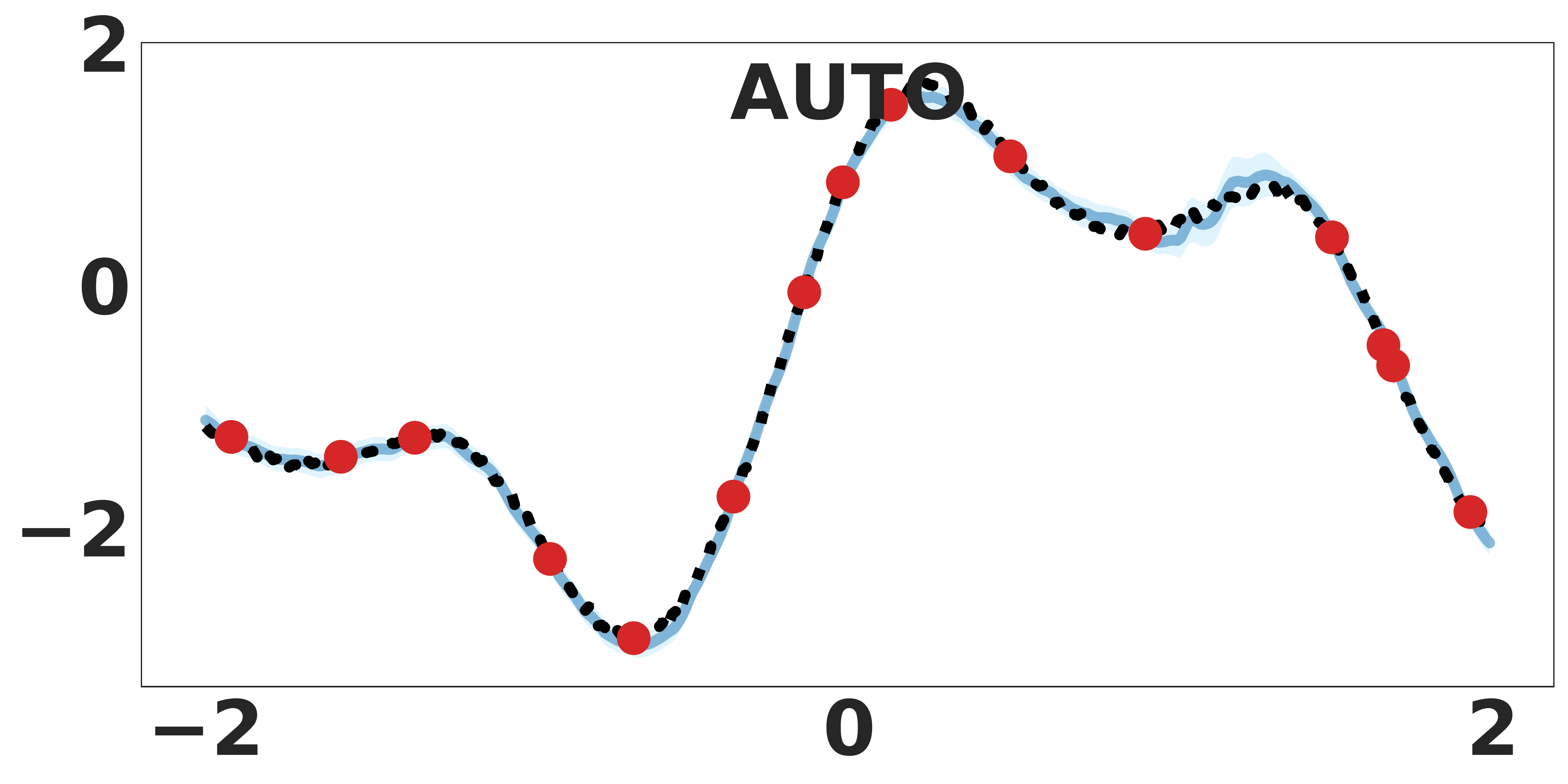}
	\end{subfigure}%
	\begin{subfigure}{0.25\textwidth}
		\centering
		\includegraphics[width=\linewidth]{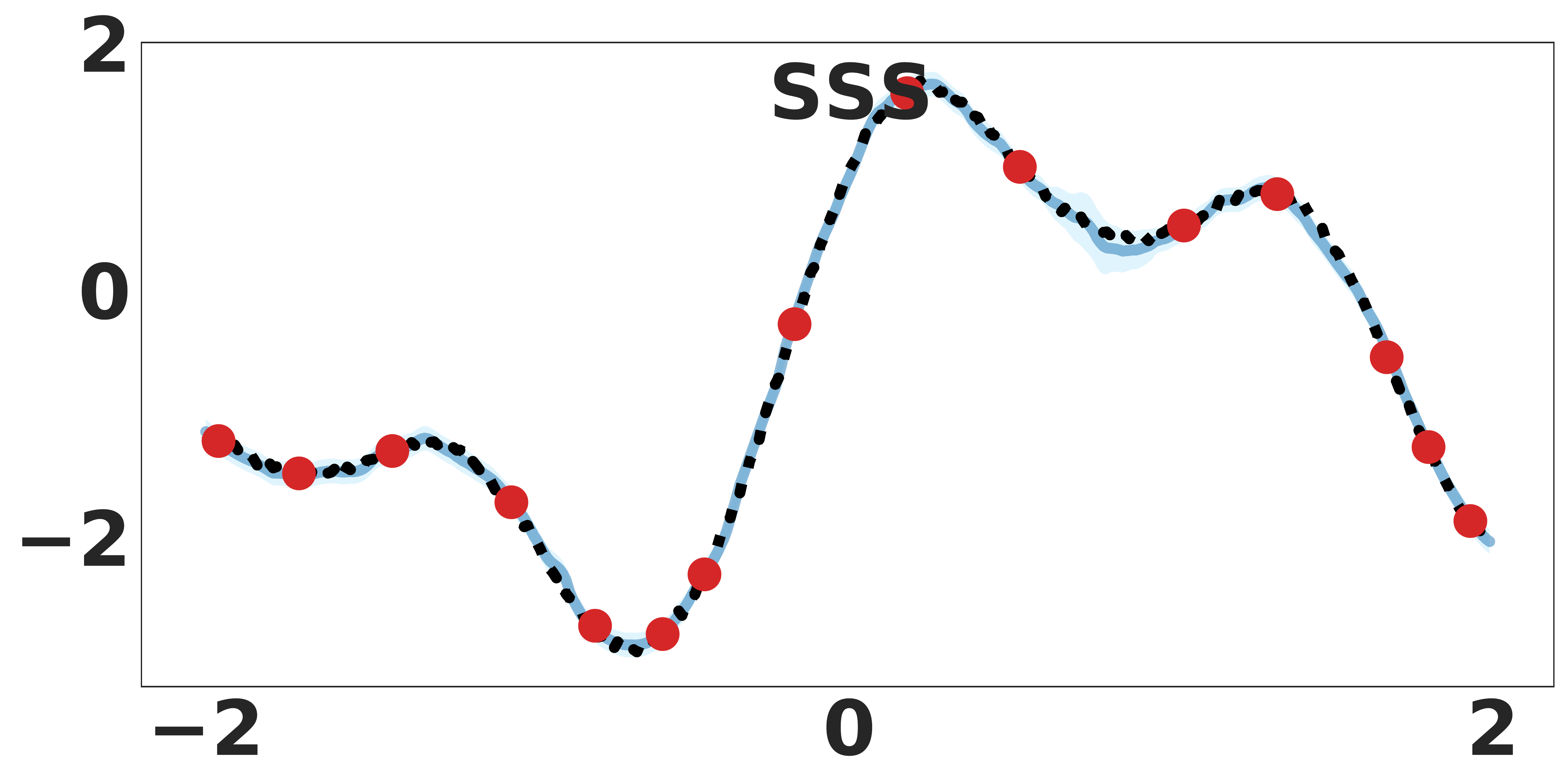}
	\end{subfigure}	
    \begin{subfigure}{0.25\textwidth}
		\centering
		\includegraphics[width=\linewidth]{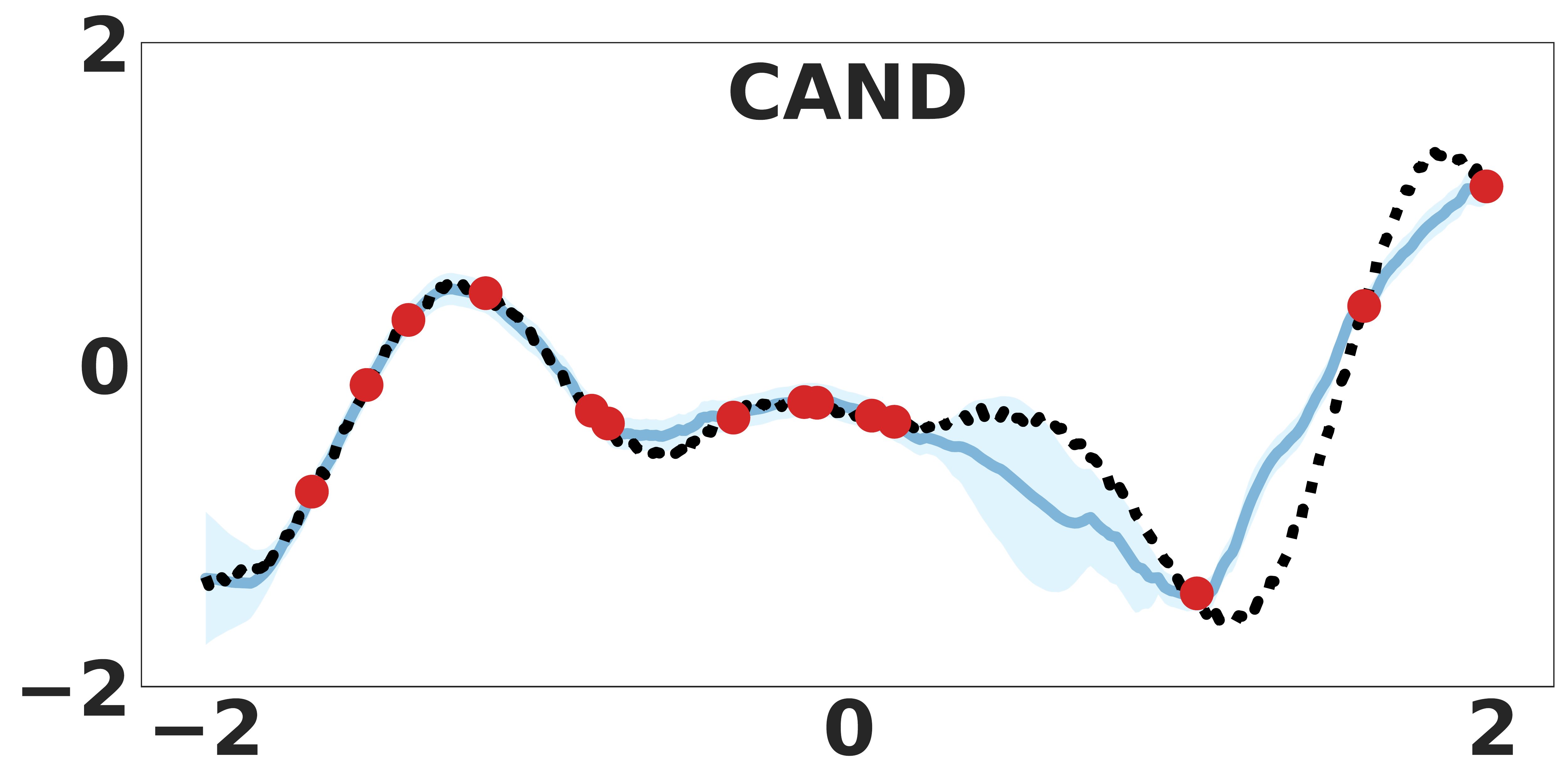}
	\end{subfigure}%
    \begin{subfigure}{0.25\textwidth}
		\centering
		\includegraphics[width=\linewidth]{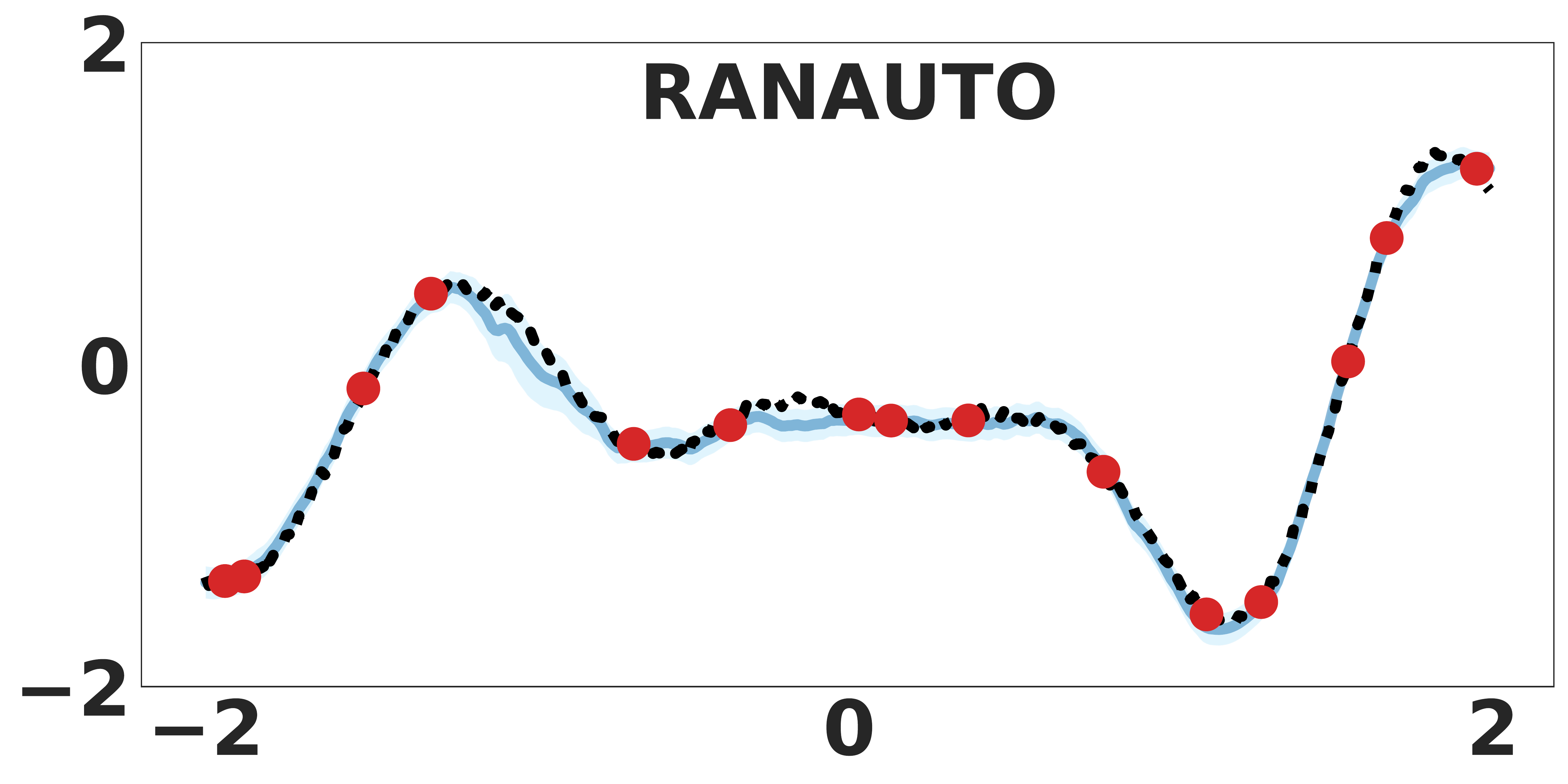}
	\end{subfigure}%
    \begin{subfigure}{0.25\textwidth}
		\centering
		\includegraphics[width=\linewidth]{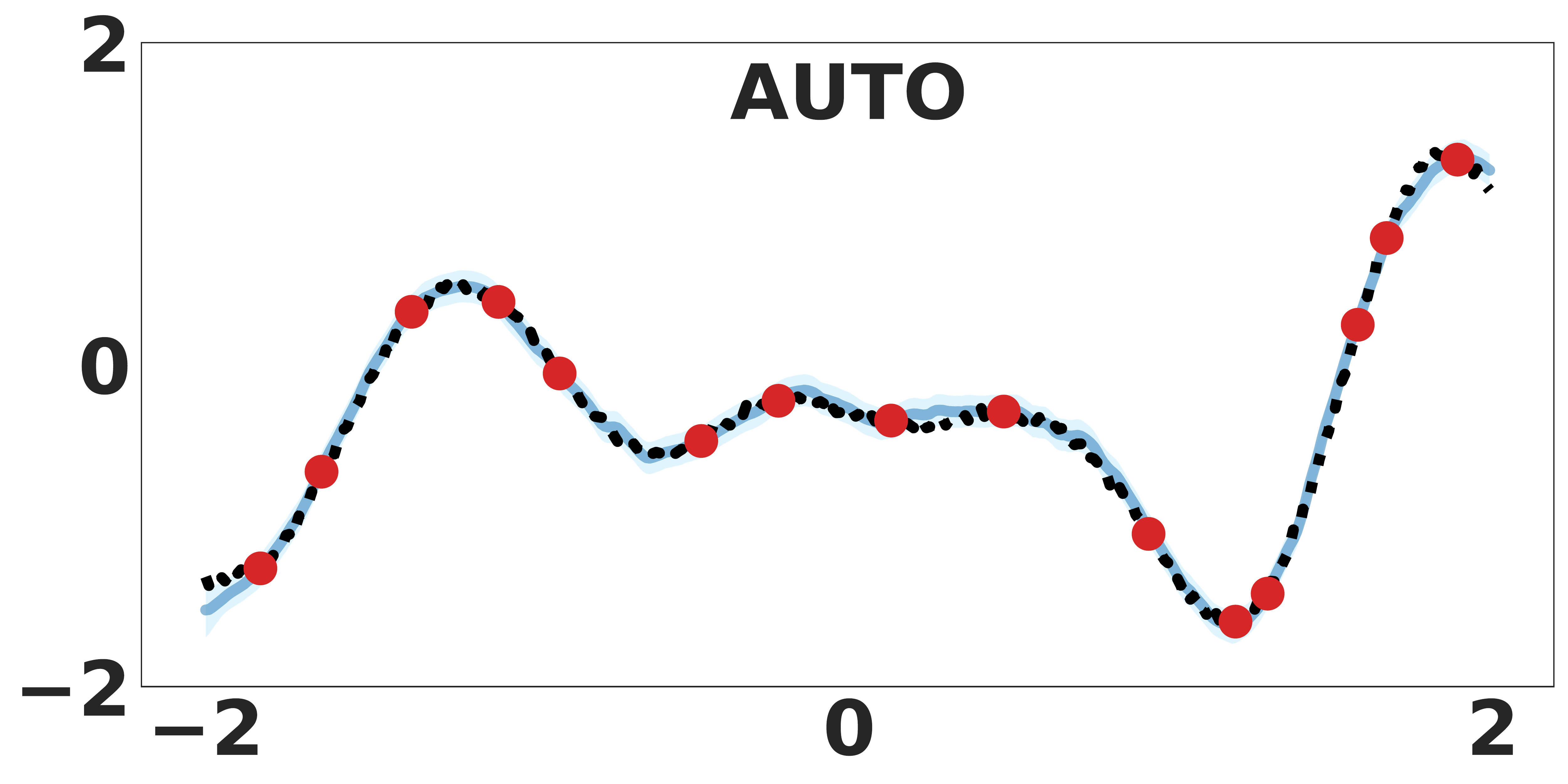}
	\end{subfigure}%
	\begin{subfigure}{0.25\textwidth}
		\centering
		\includegraphics[width=\linewidth]{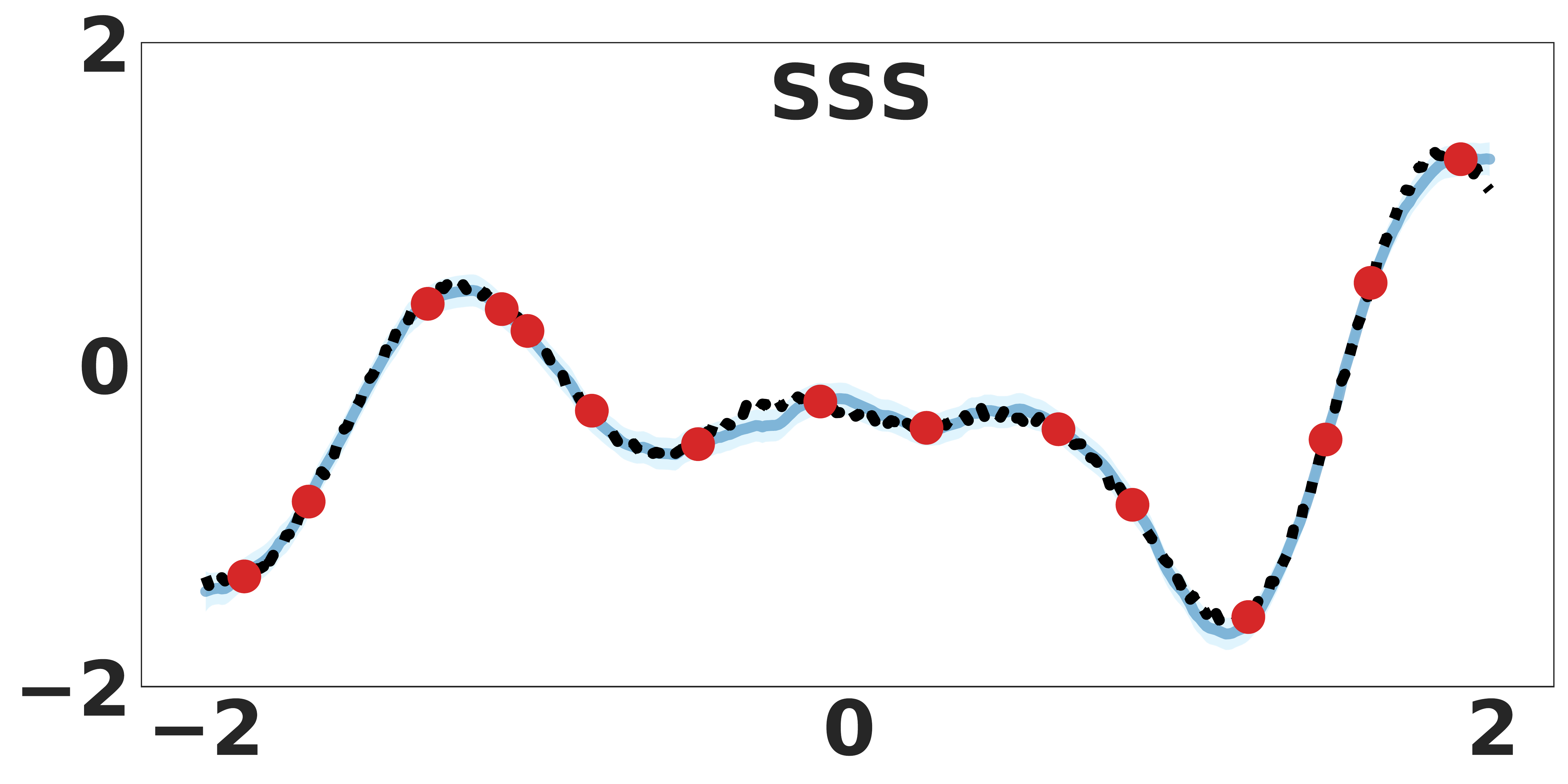}
	\end{subfigure}	
    \begin{subfigure}{0.25\textwidth}
		\centering
		\includegraphics[width=\linewidth]{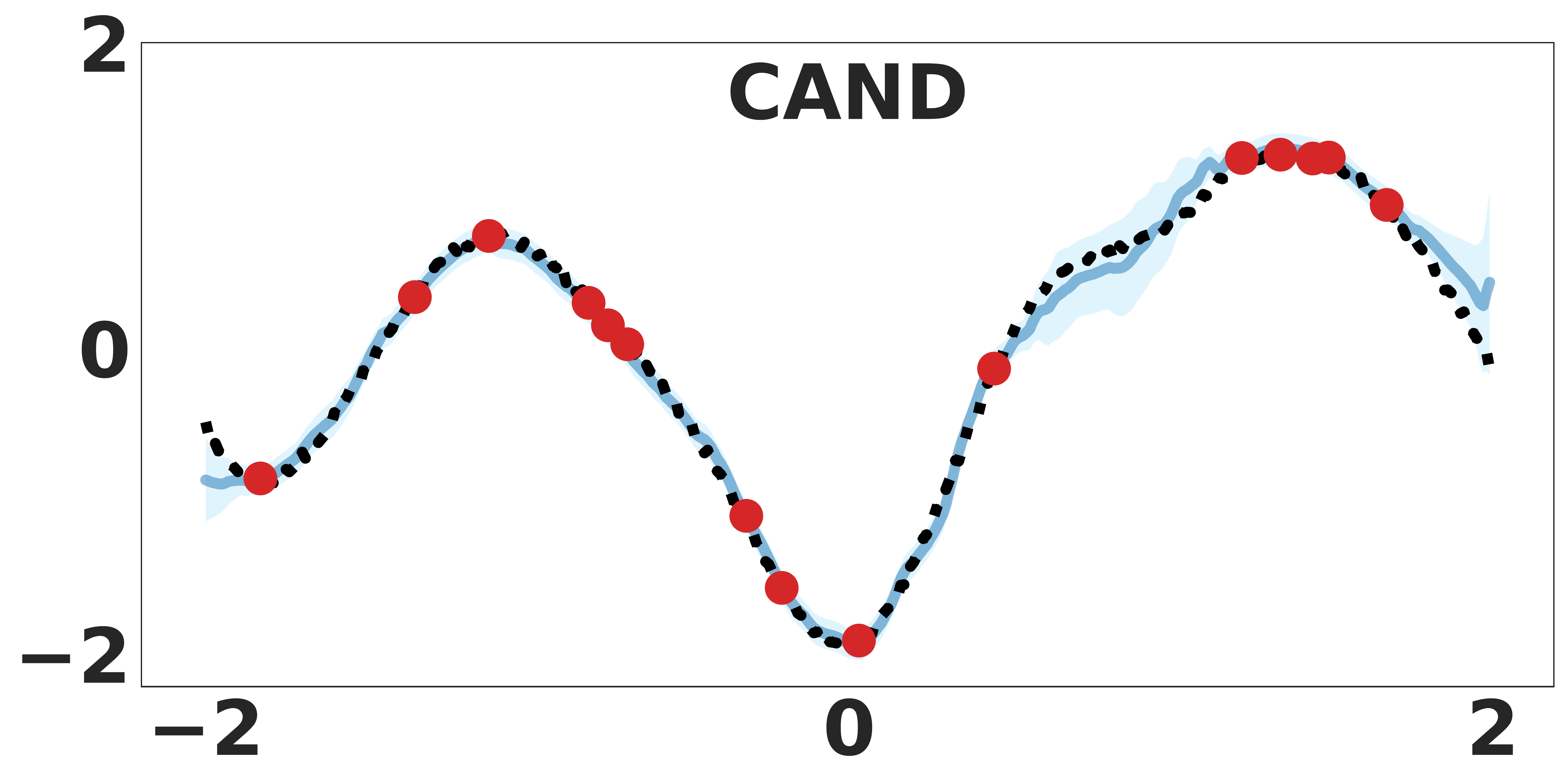}
	\end{subfigure}%
    \begin{subfigure}{0.25\textwidth}
		\centering
		\includegraphics[width=\linewidth]{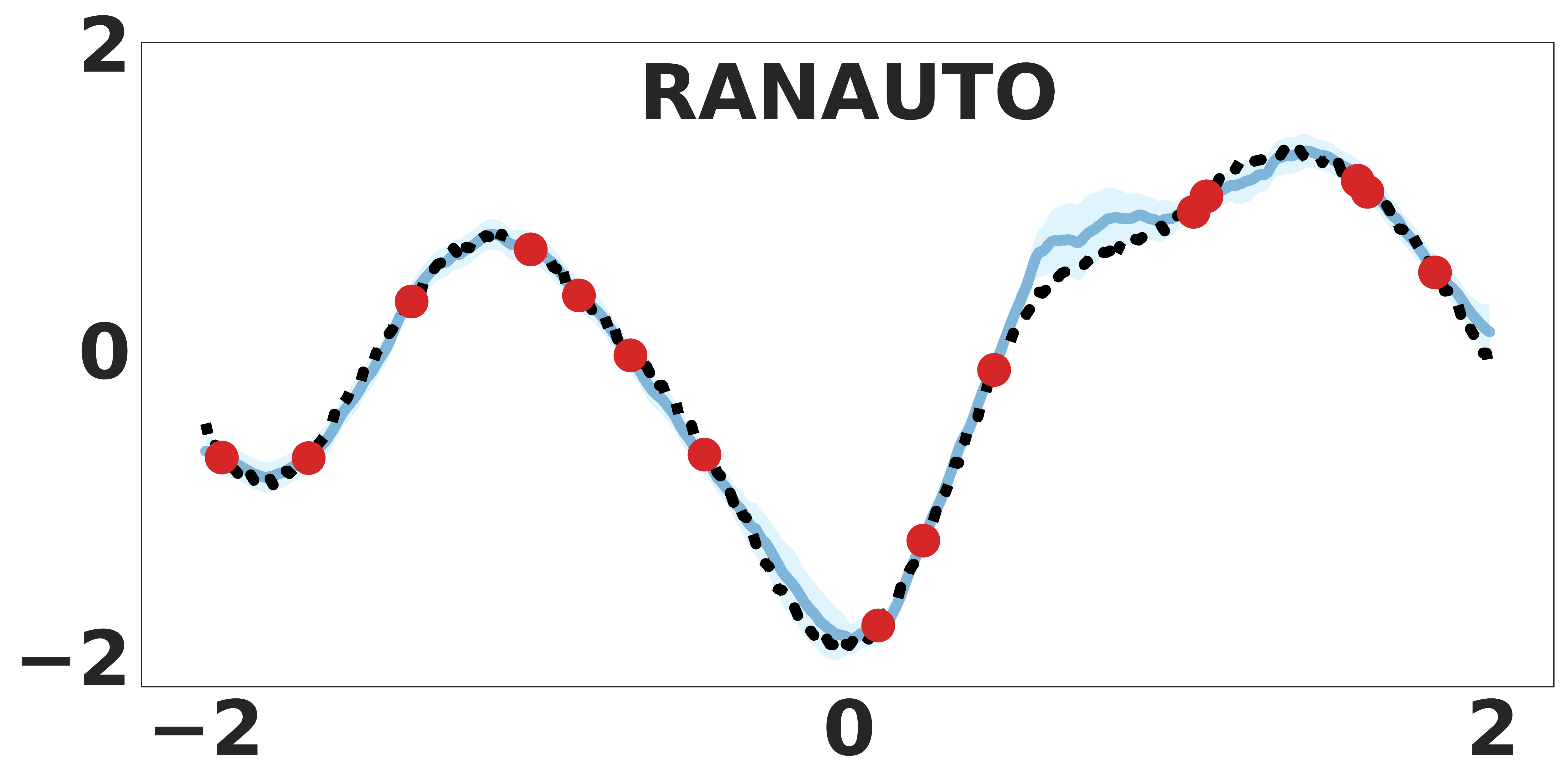}
	\end{subfigure}%
    \begin{subfigure}{0.25\textwidth}
		\centering
		\includegraphics[width=\linewidth]{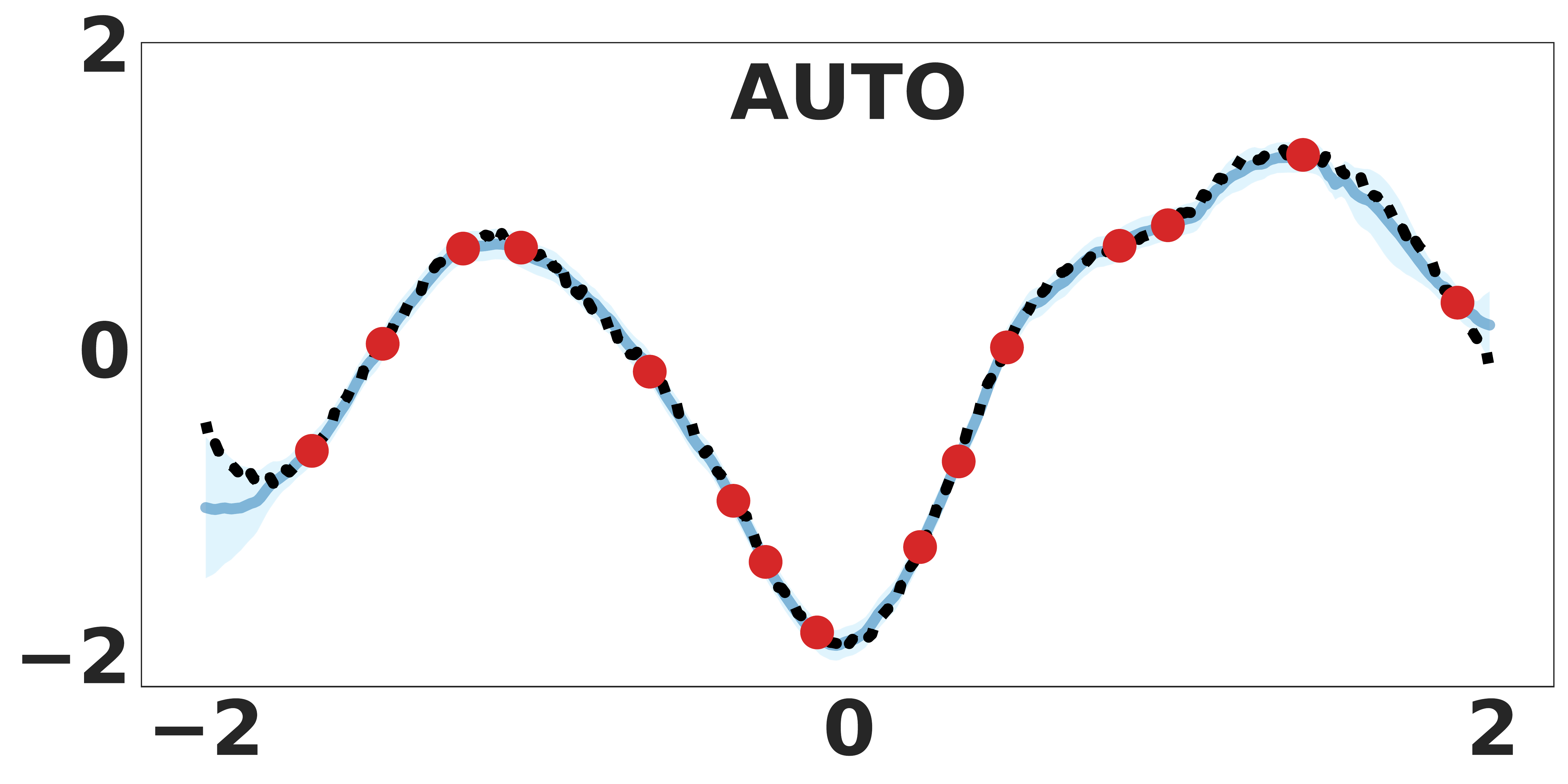}
	\end{subfigure}%
	\begin{subfigure}{0.25\textwidth}
		\centering
		\includegraphics[width=\linewidth]{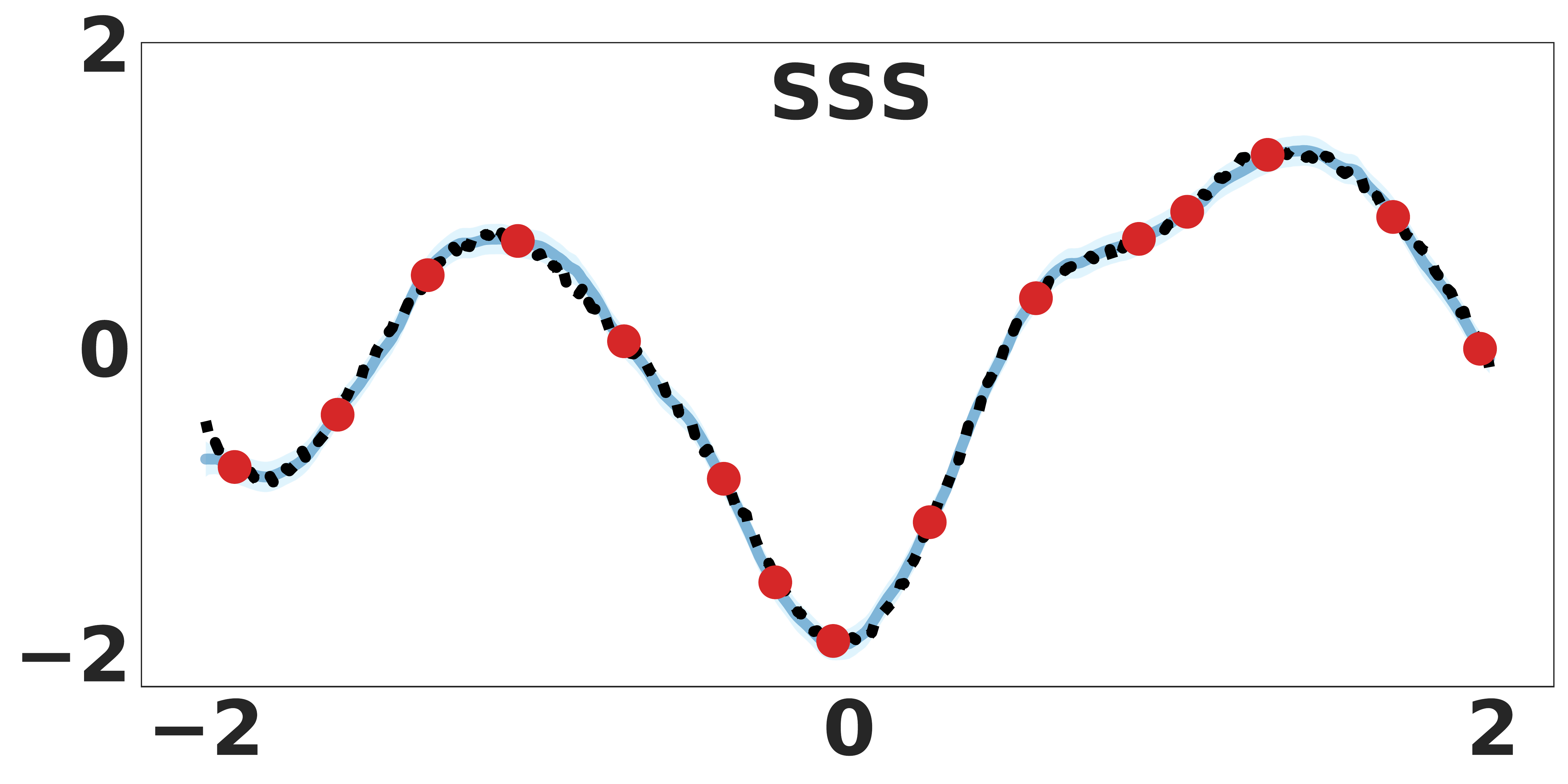}
	\end{subfigure}	
    \begin{subfigure}{0.25\textwidth}
		\centering
		\includegraphics[width=\linewidth]{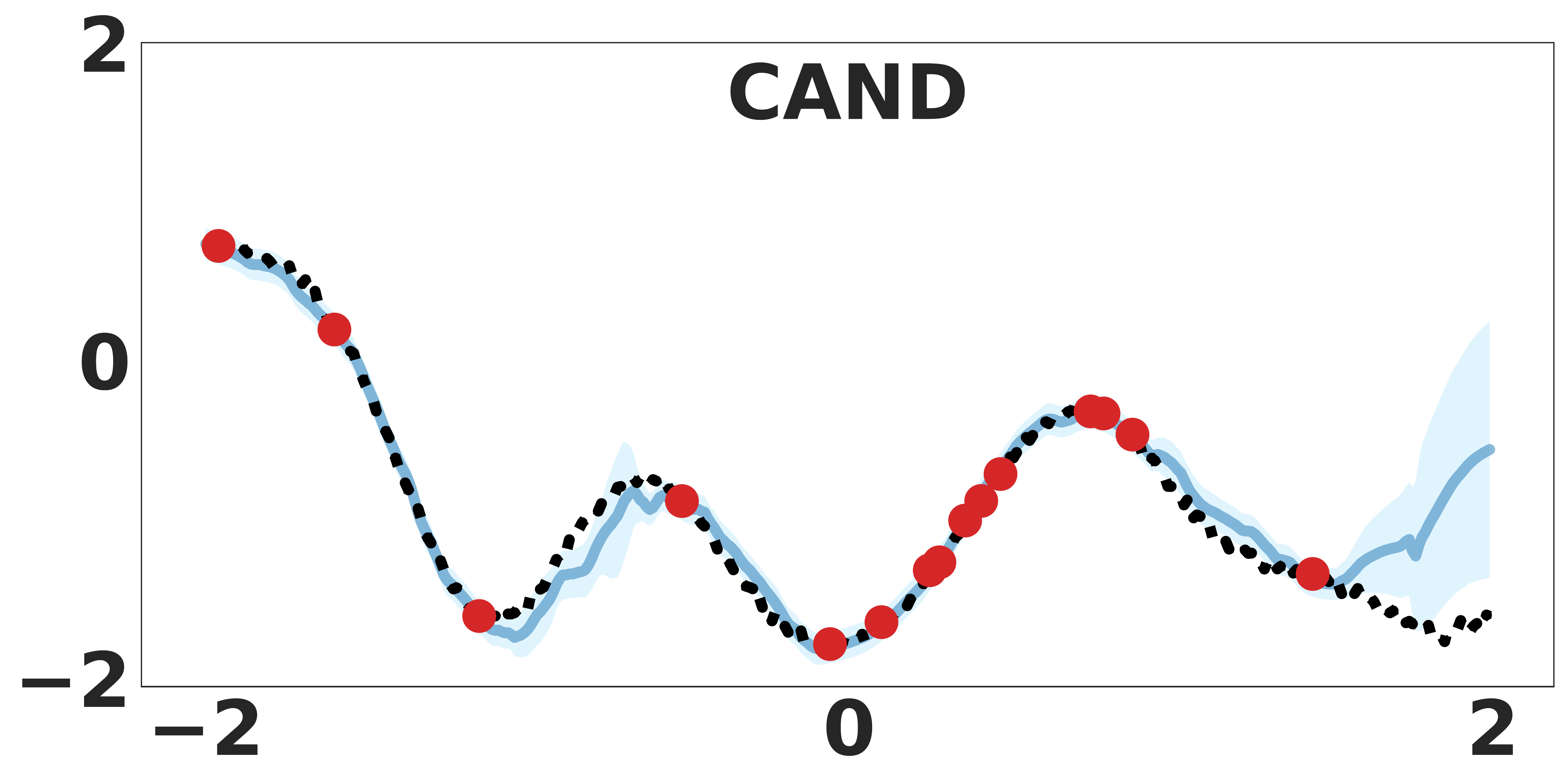}
	\end{subfigure}%
    \begin{subfigure}{0.25\textwidth}
		\centering
		\includegraphics[width=\linewidth]{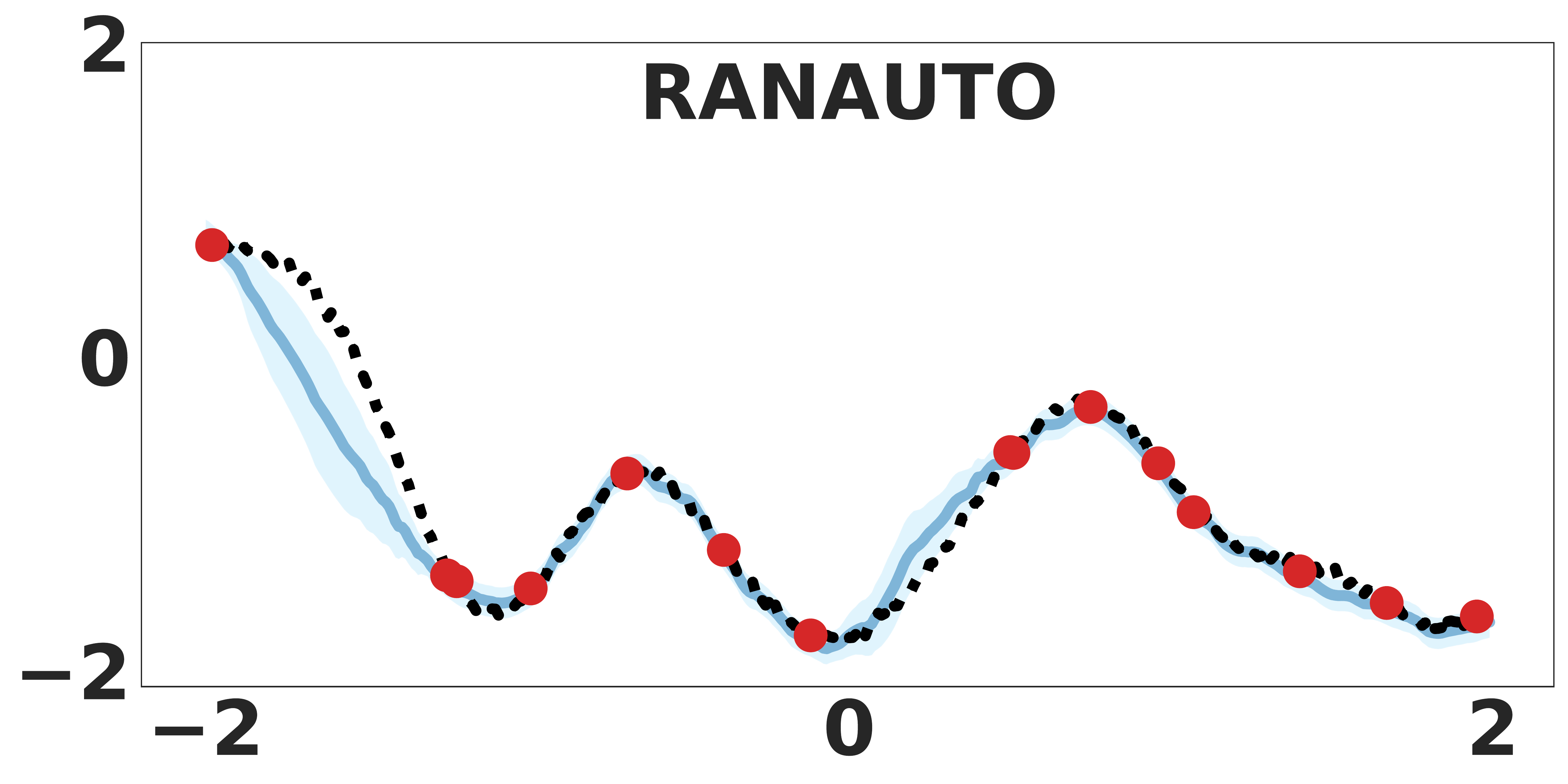}
	\end{subfigure}%
    \begin{subfigure}{0.25\textwidth}
		\centering
		\includegraphics[width=\linewidth]{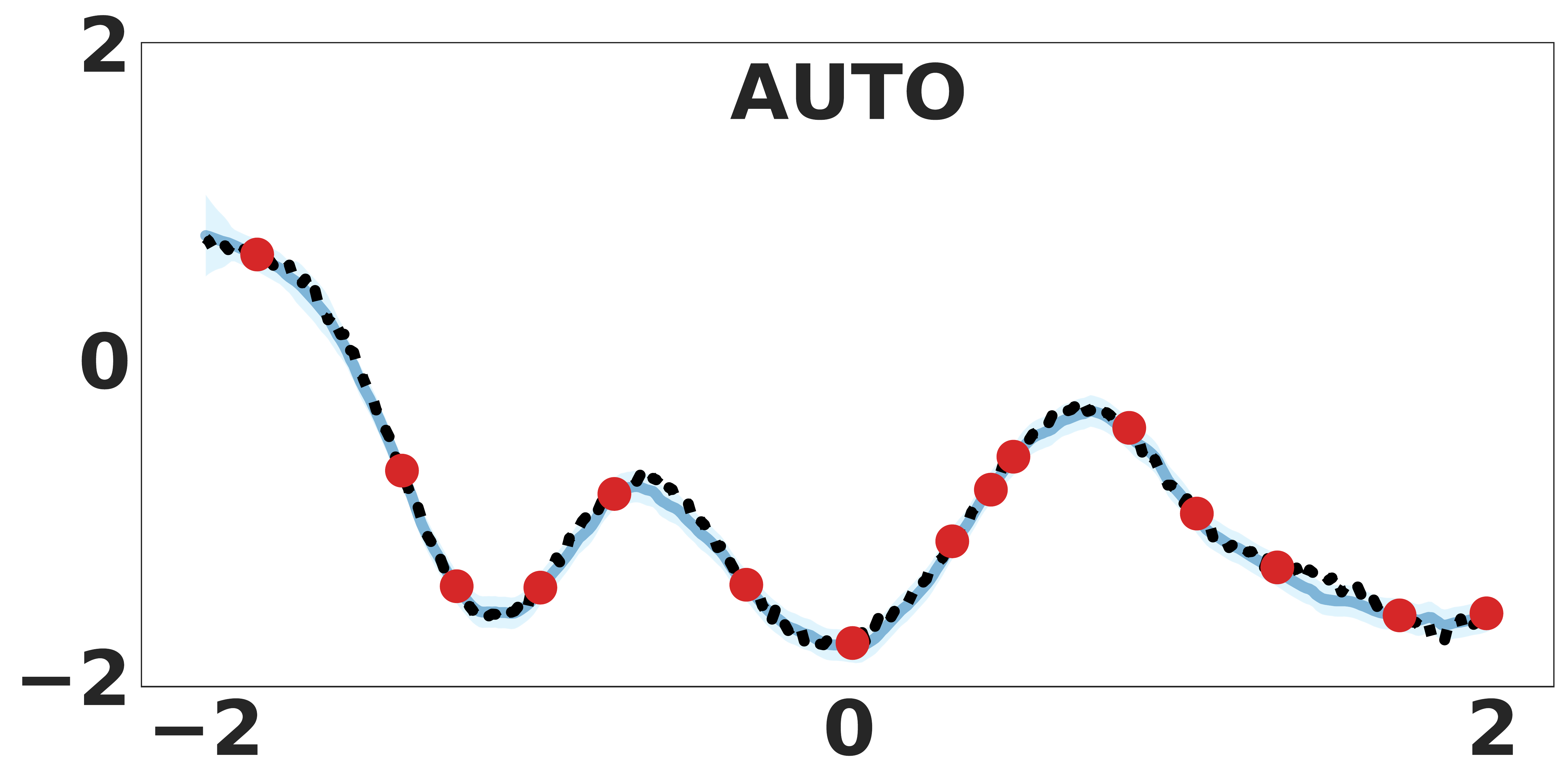}
	\end{subfigure}%
	\begin{subfigure}{0.25\textwidth}
		\centering
		\includegraphics[width=\linewidth]{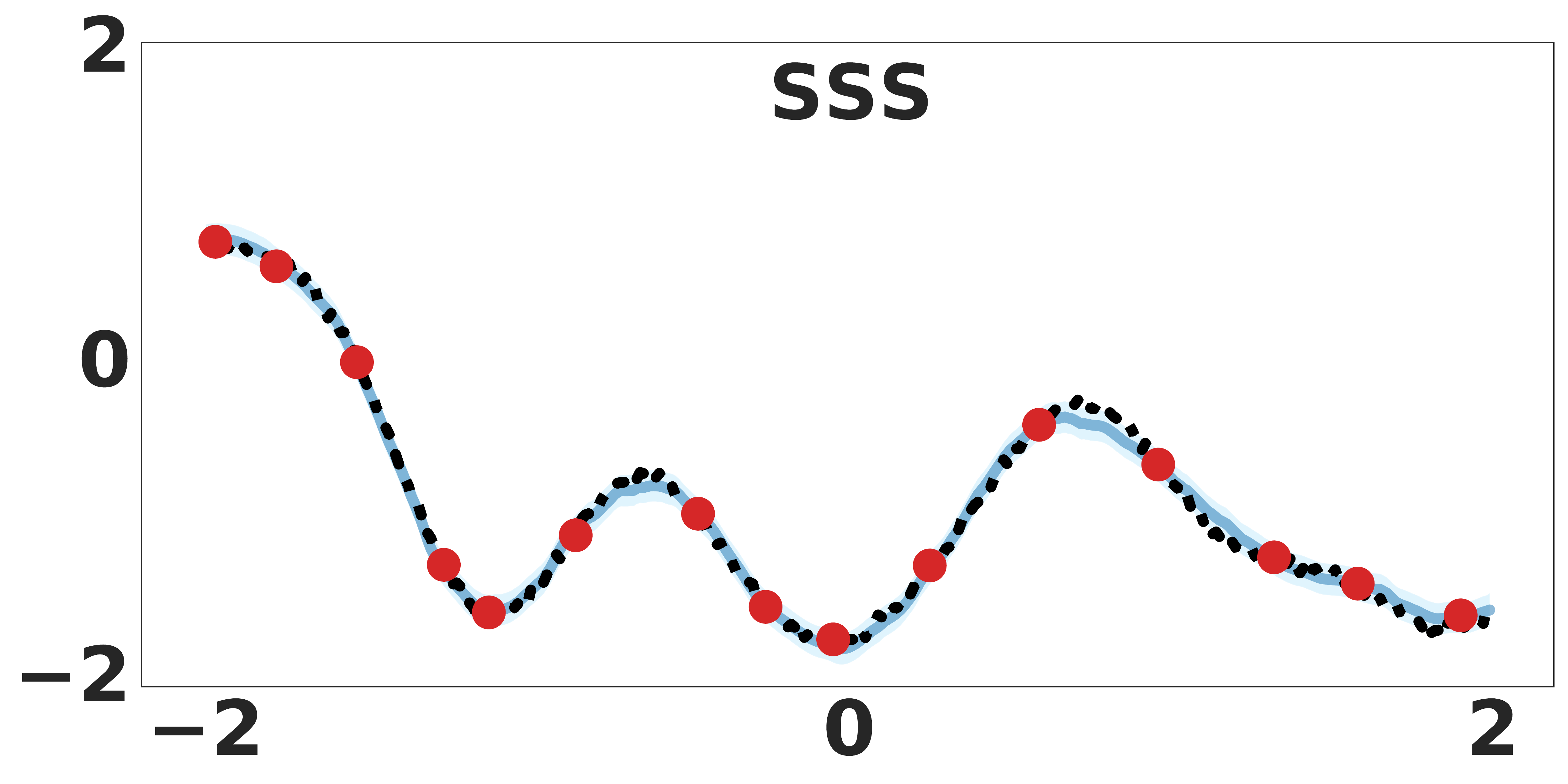}
	\end{subfigure}	
    \begin{subfigure}{0.25\textwidth}
		\centering
		\includegraphics[width=\linewidth]{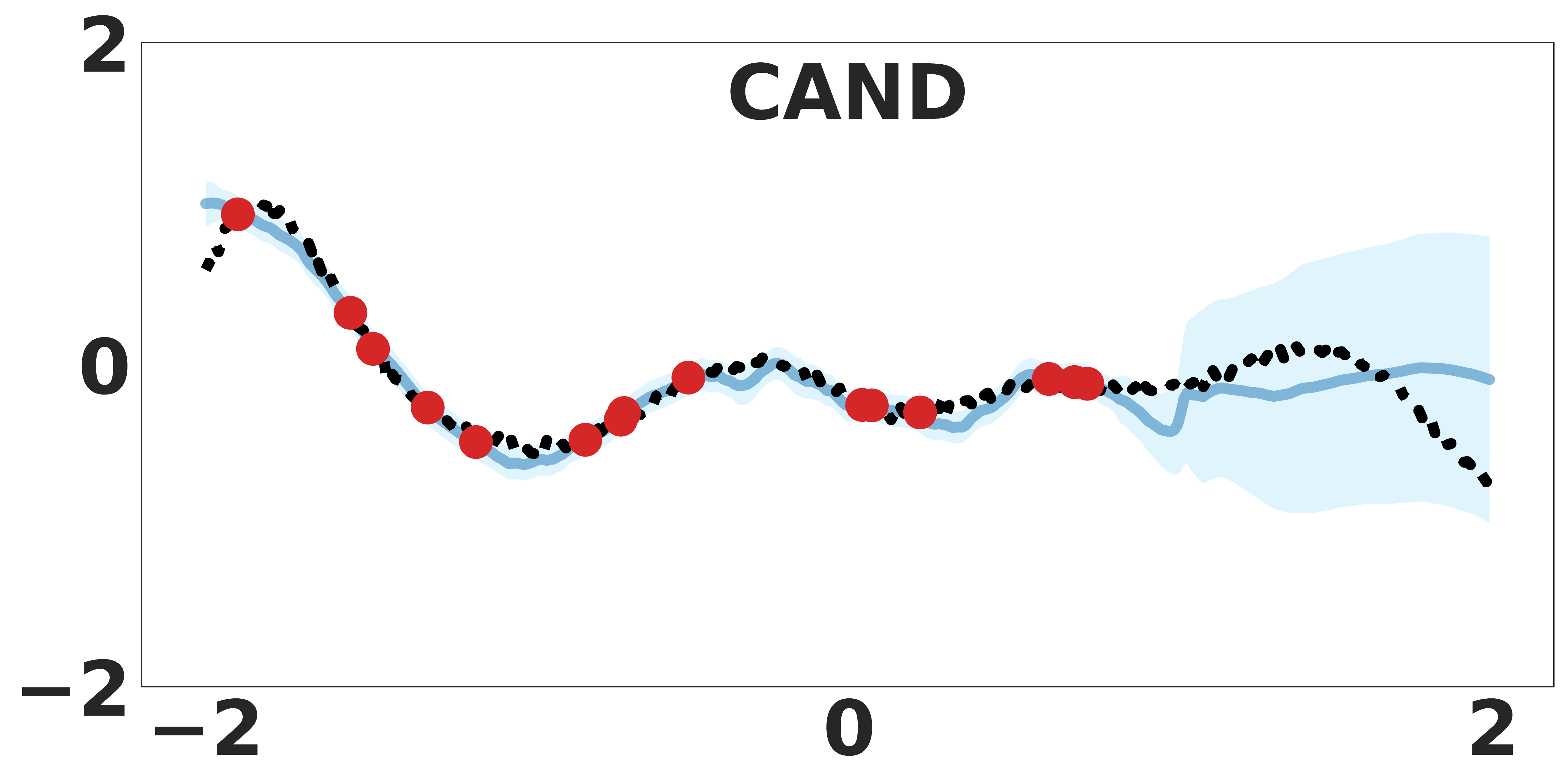}
	\end{subfigure}%
    \begin{subfigure}{0.25\textwidth}
		\centering
		\includegraphics[width=\linewidth]{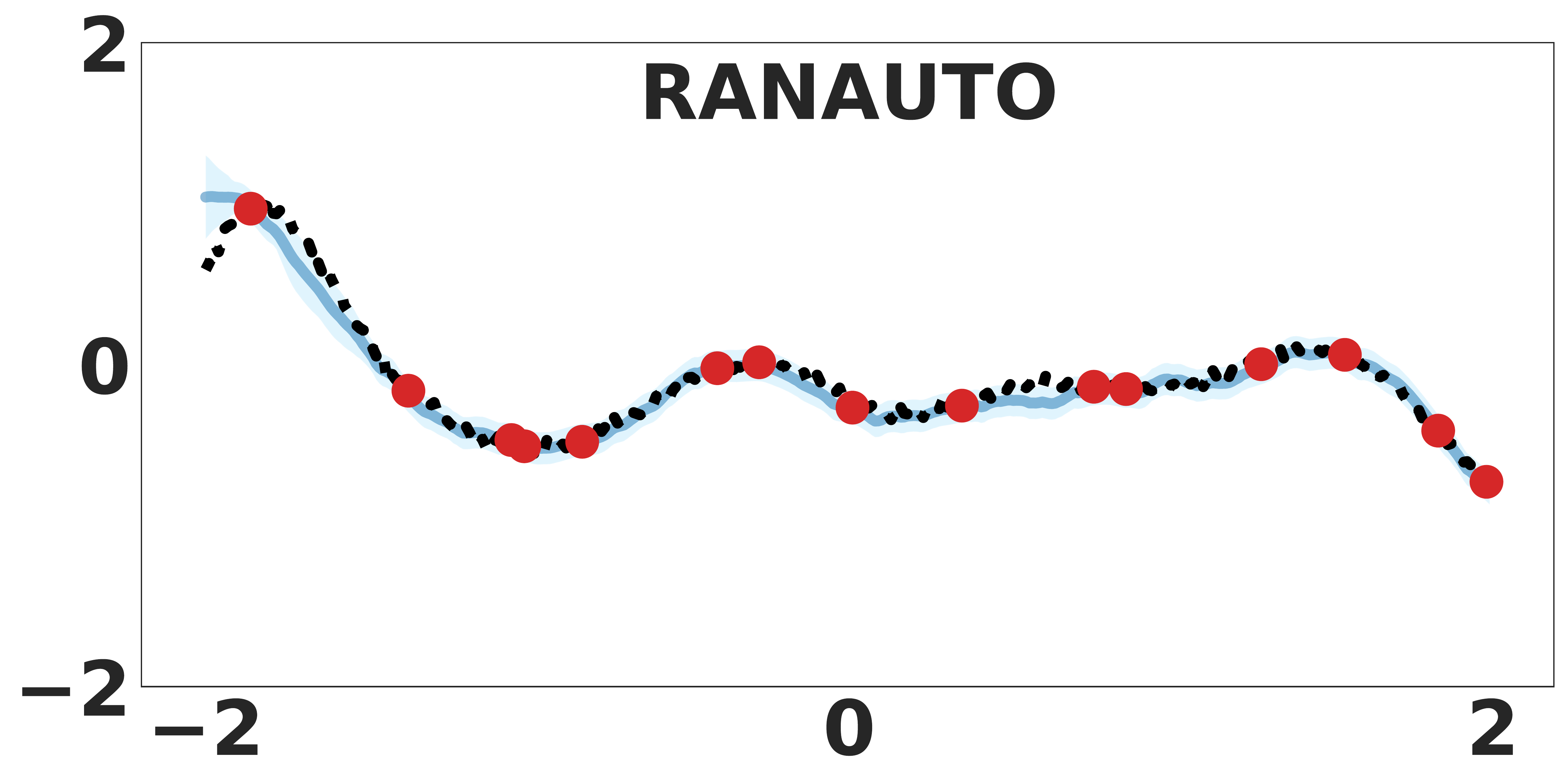}
	\end{subfigure}%
    \begin{subfigure}{0.25\textwidth}
		\centering
		\includegraphics[width=\linewidth]{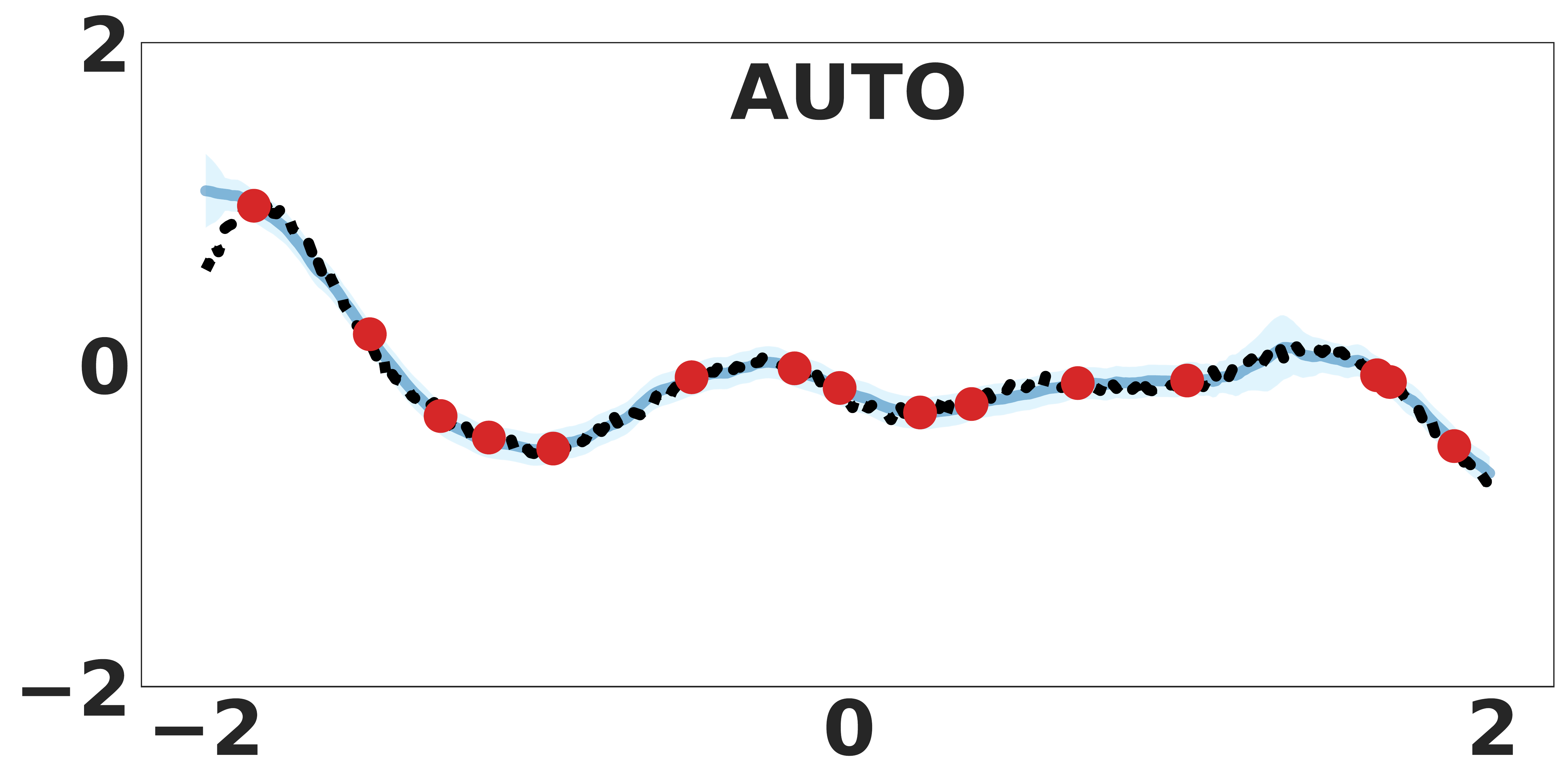}
	\end{subfigure}%
	\begin{subfigure}{0.25\textwidth}
		\centering
		\includegraphics[width=\linewidth]{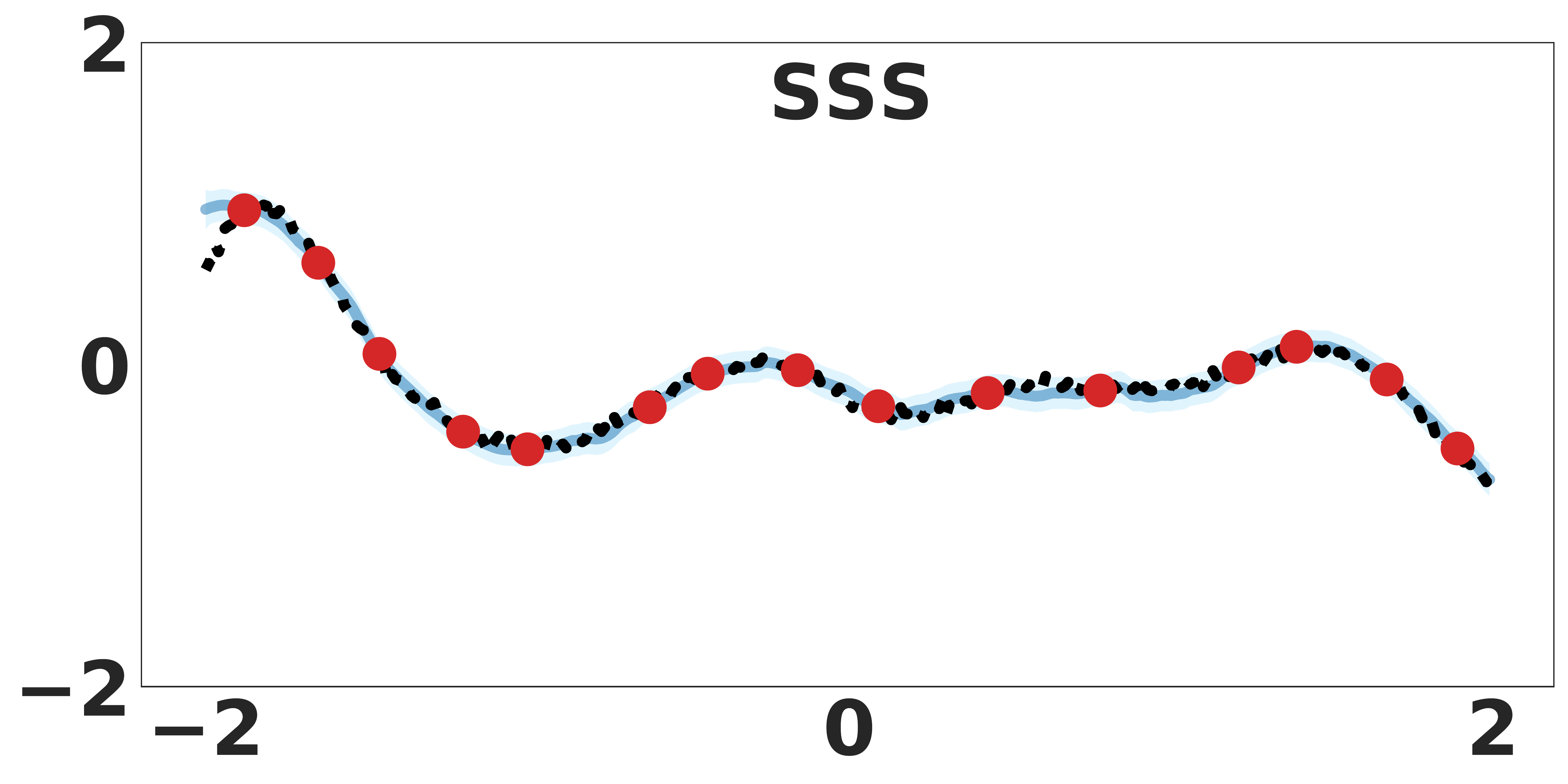}
	\end{subfigure}	
    \begin{subfigure}{0.25\textwidth}
		\centering
		\includegraphics[width=\linewidth]{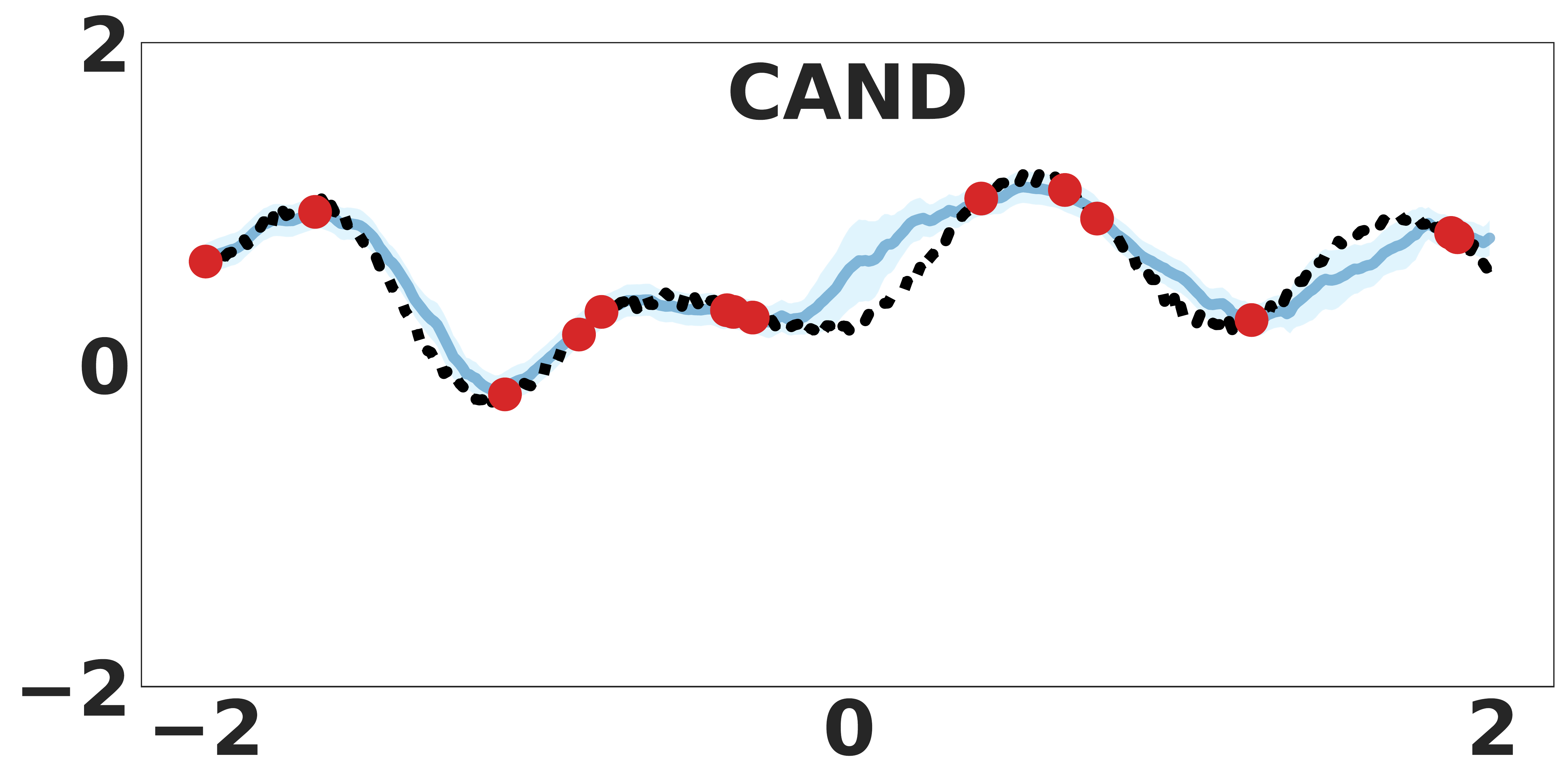}
	\end{subfigure}%
    \begin{subfigure}{0.25\textwidth}
		\centering
		\includegraphics[width=\linewidth]{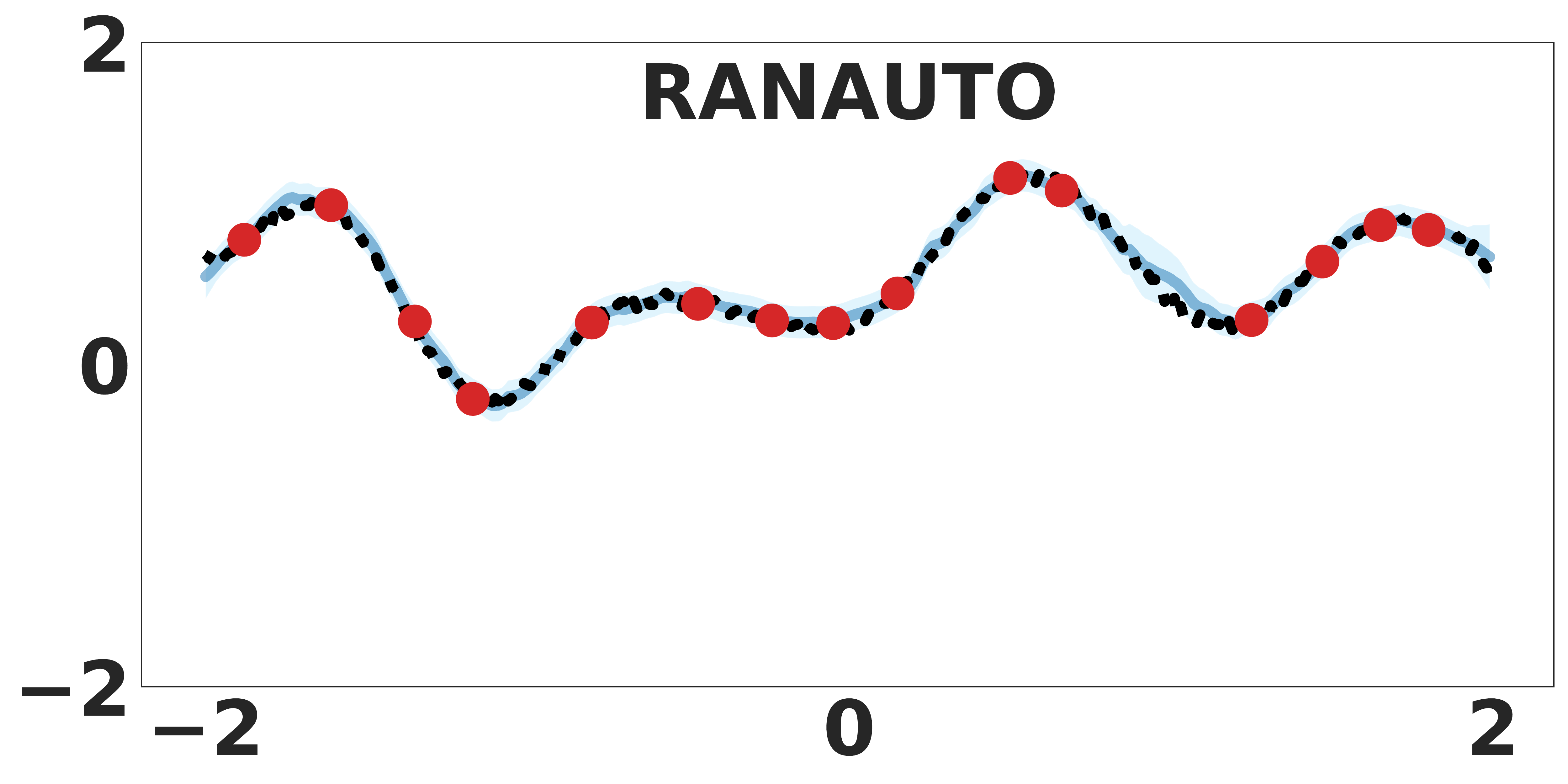}
	\end{subfigure}%
    \begin{subfigure}{0.25\textwidth}
		\centering
		\includegraphics[width=\linewidth]{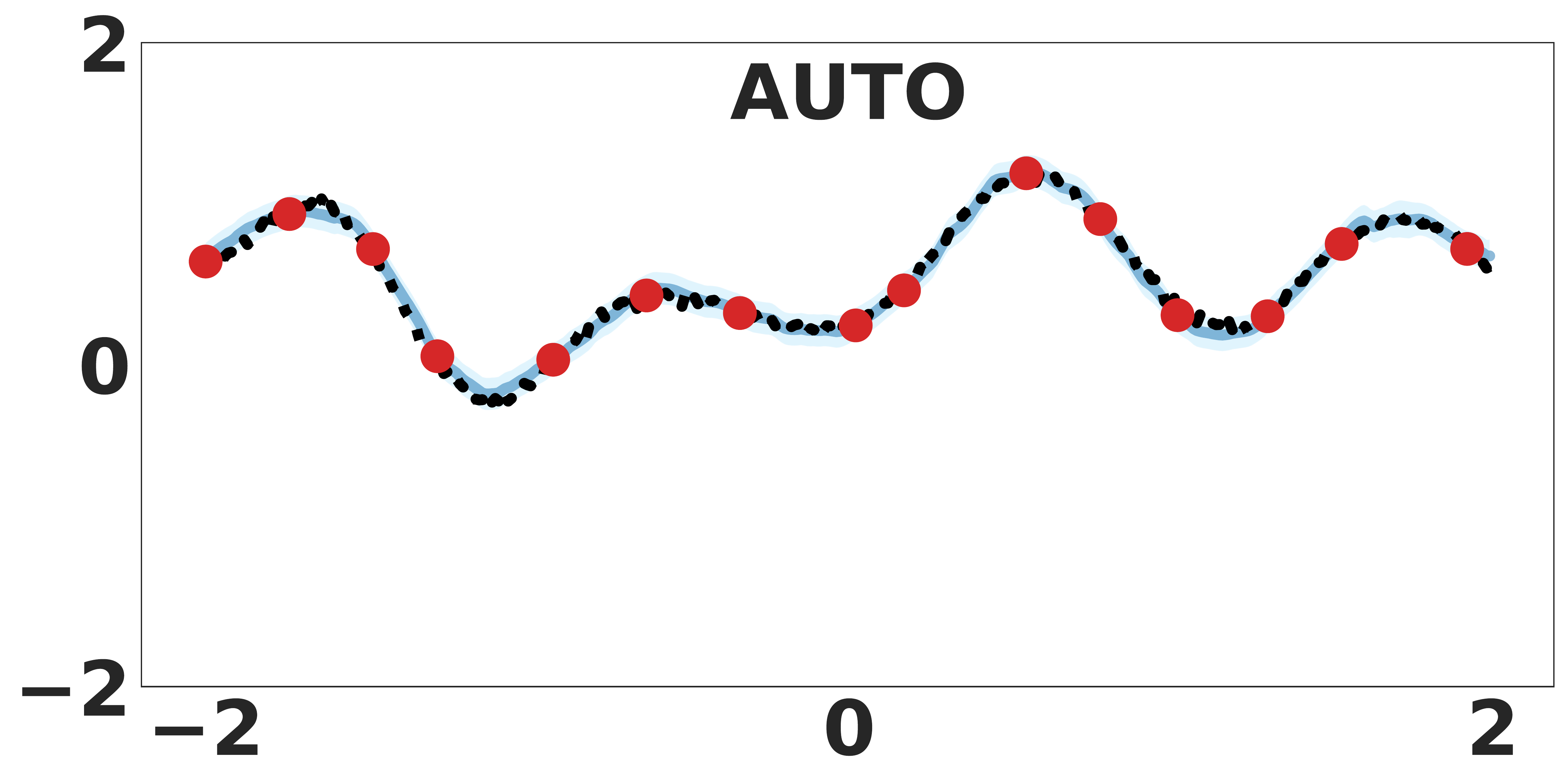}
	\end{subfigure}%
	\begin{subfigure}{0.25\textwidth}
		\centering
		\includegraphics[width=\linewidth]{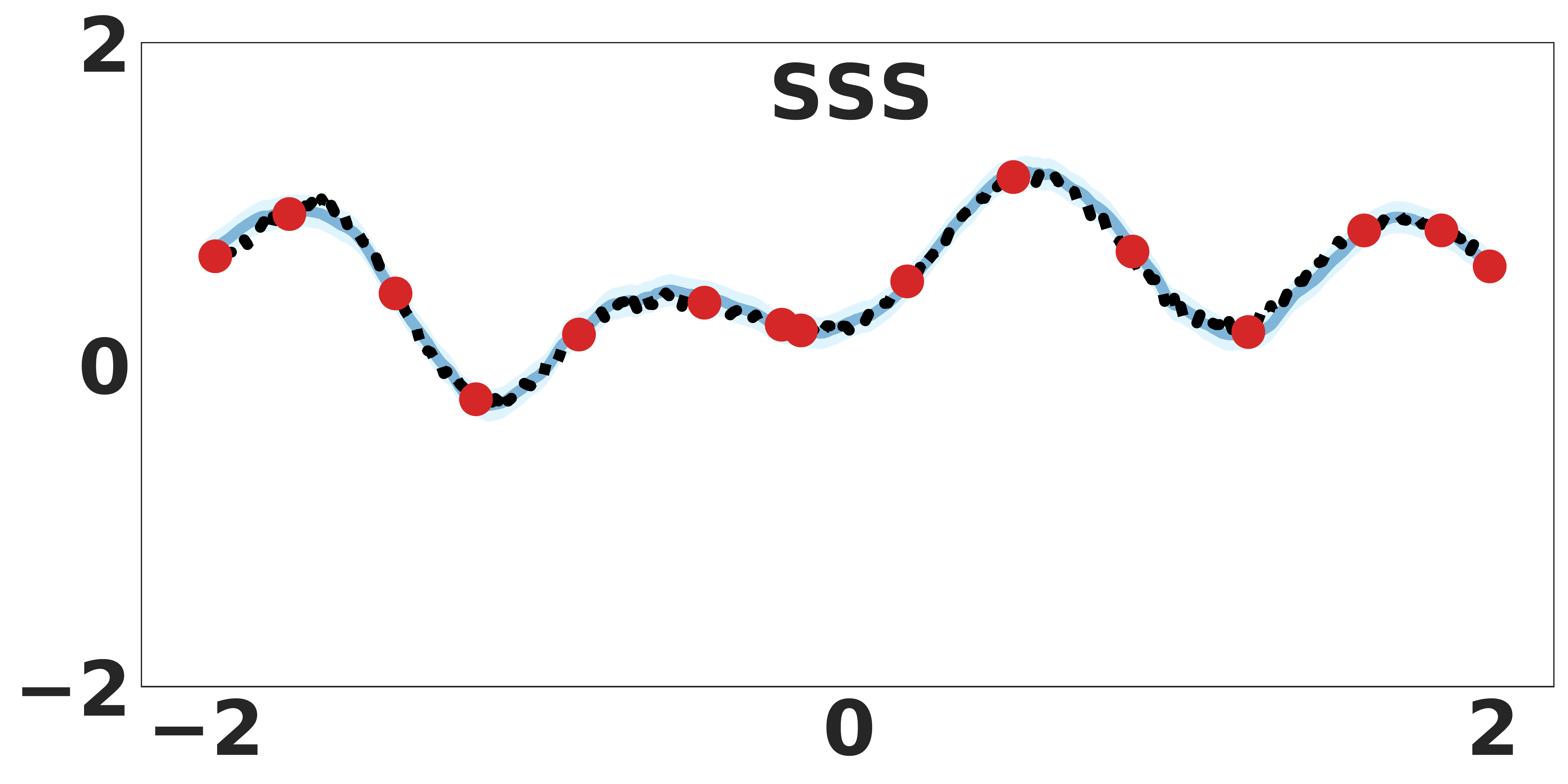}
	\end{subfigure}	
    \begin{subfigure}{0.25\textwidth}
		\centering
		\includegraphics[width=\linewidth]{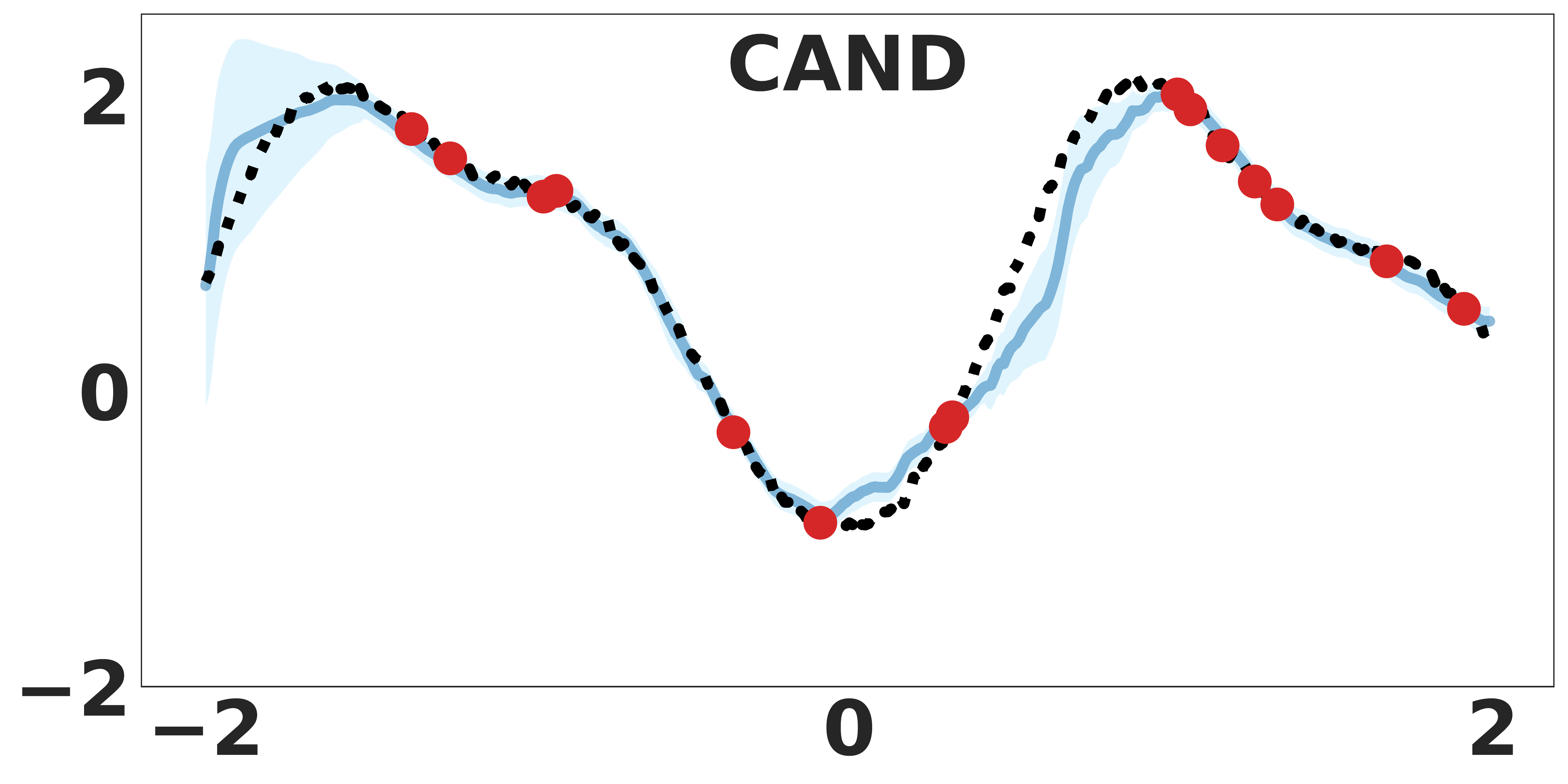}
	\end{subfigure}%
    \begin{subfigure}{0.25\textwidth}
		\centering
		\includegraphics[width=\linewidth]{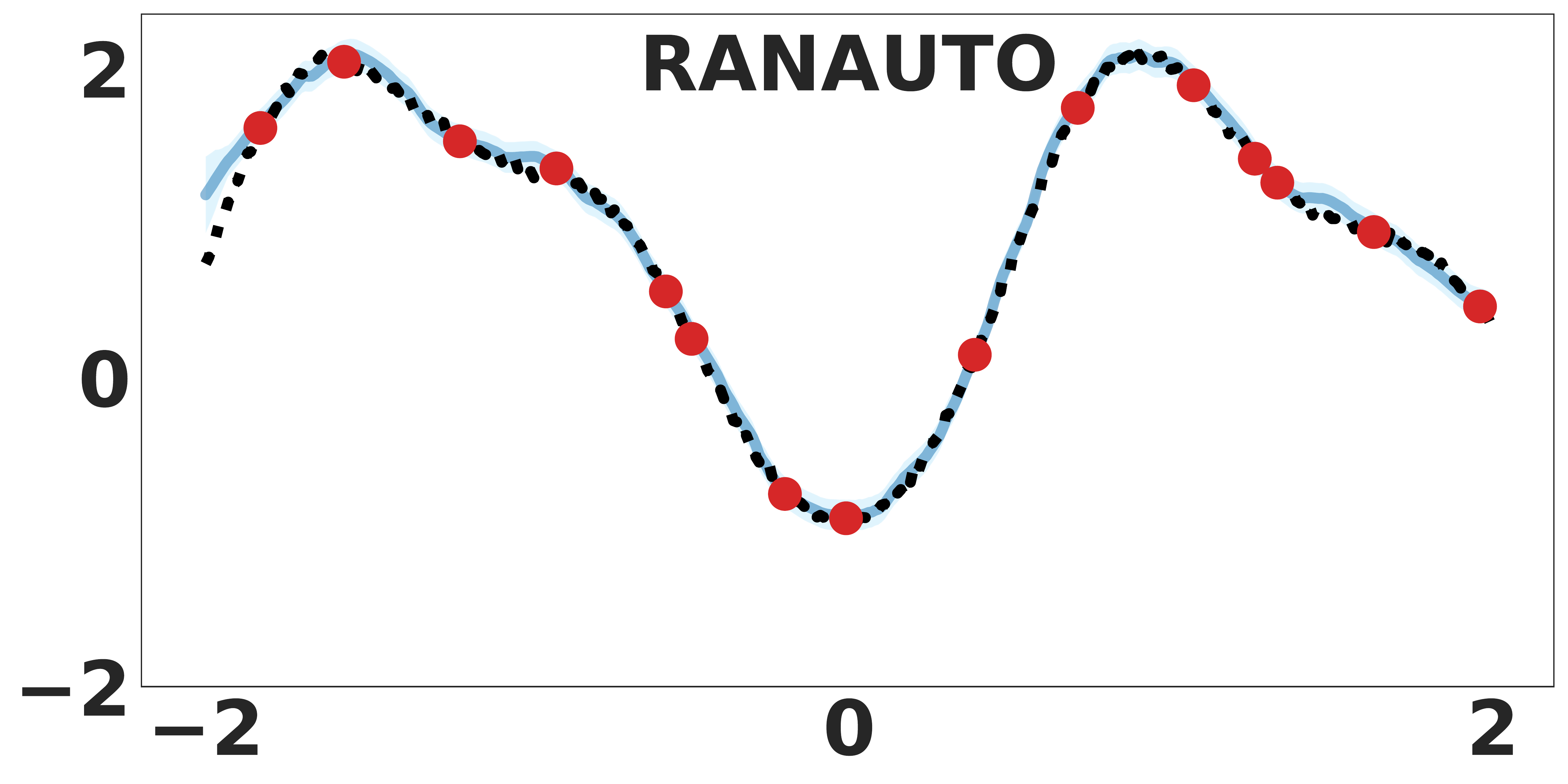}
	\end{subfigure}%
    \begin{subfigure}{0.25\textwidth}
		\centering
		\includegraphics[width=\linewidth]{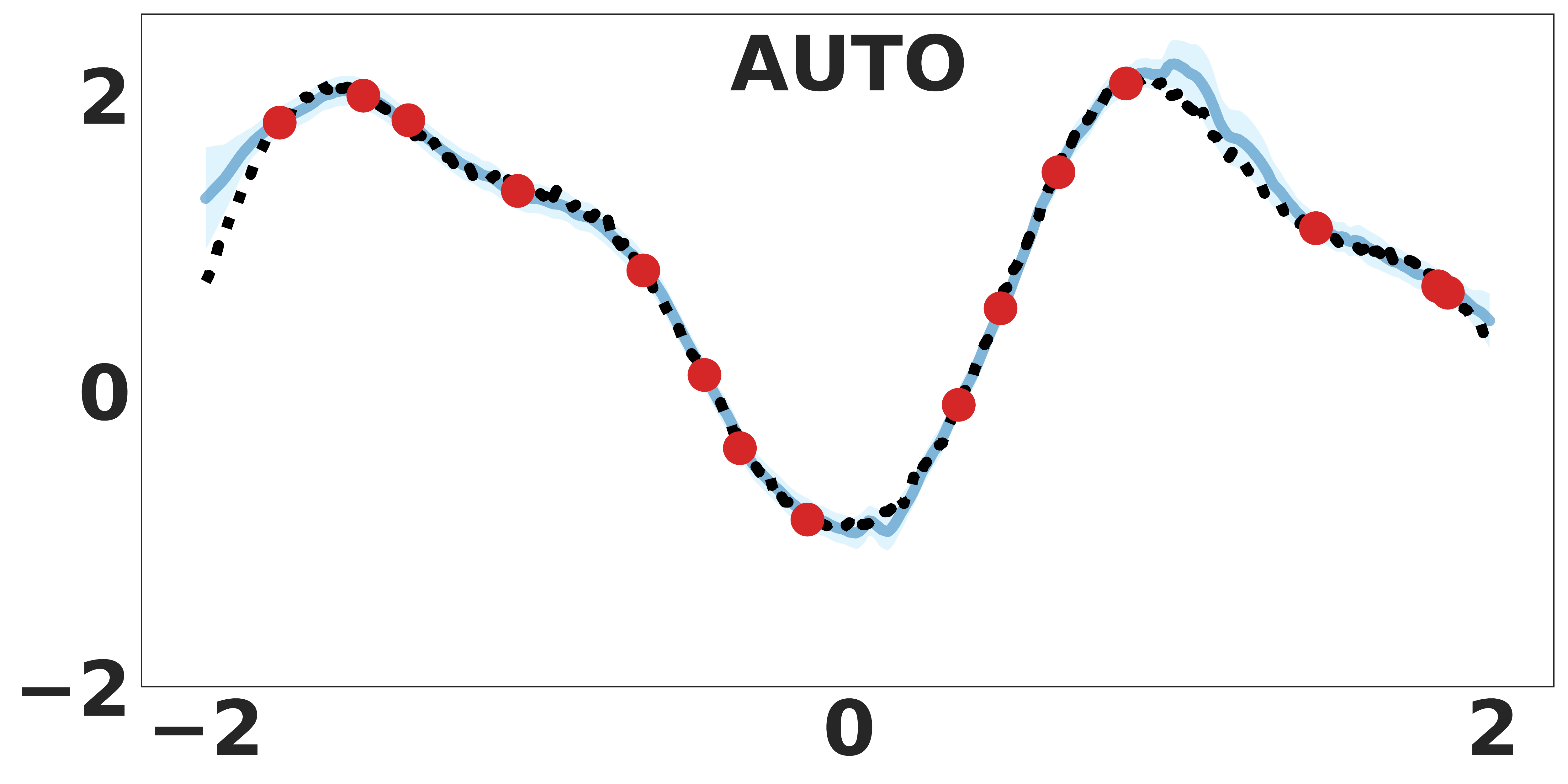}
	\end{subfigure}%
	\begin{subfigure}{0.25\textwidth}
		\centering
		\includegraphics[width=\linewidth]{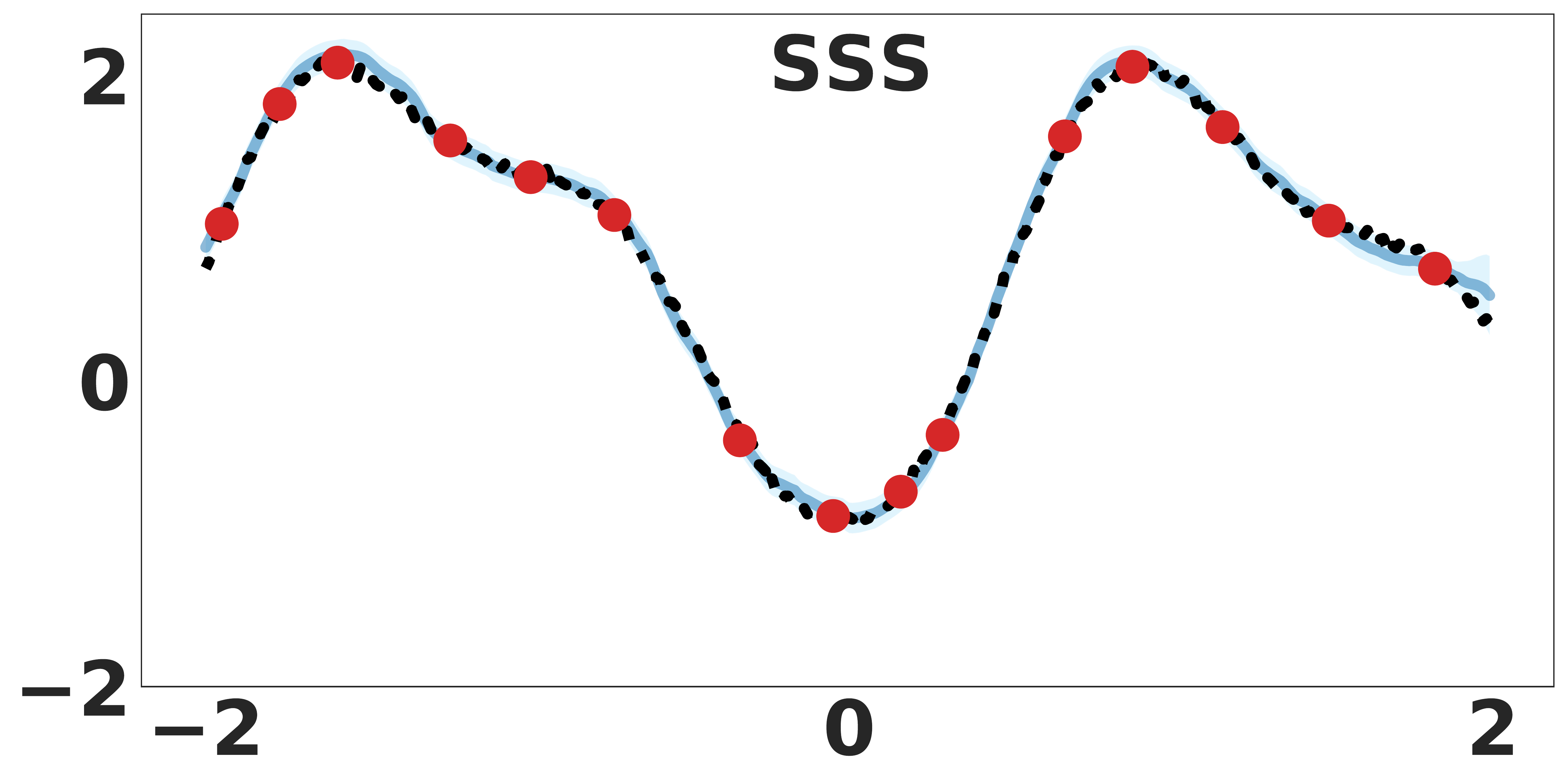}
	\end{subfigure}
    \begin{subfigure}{0.25\textwidth}
		\centering
		\includegraphics[width=\linewidth]{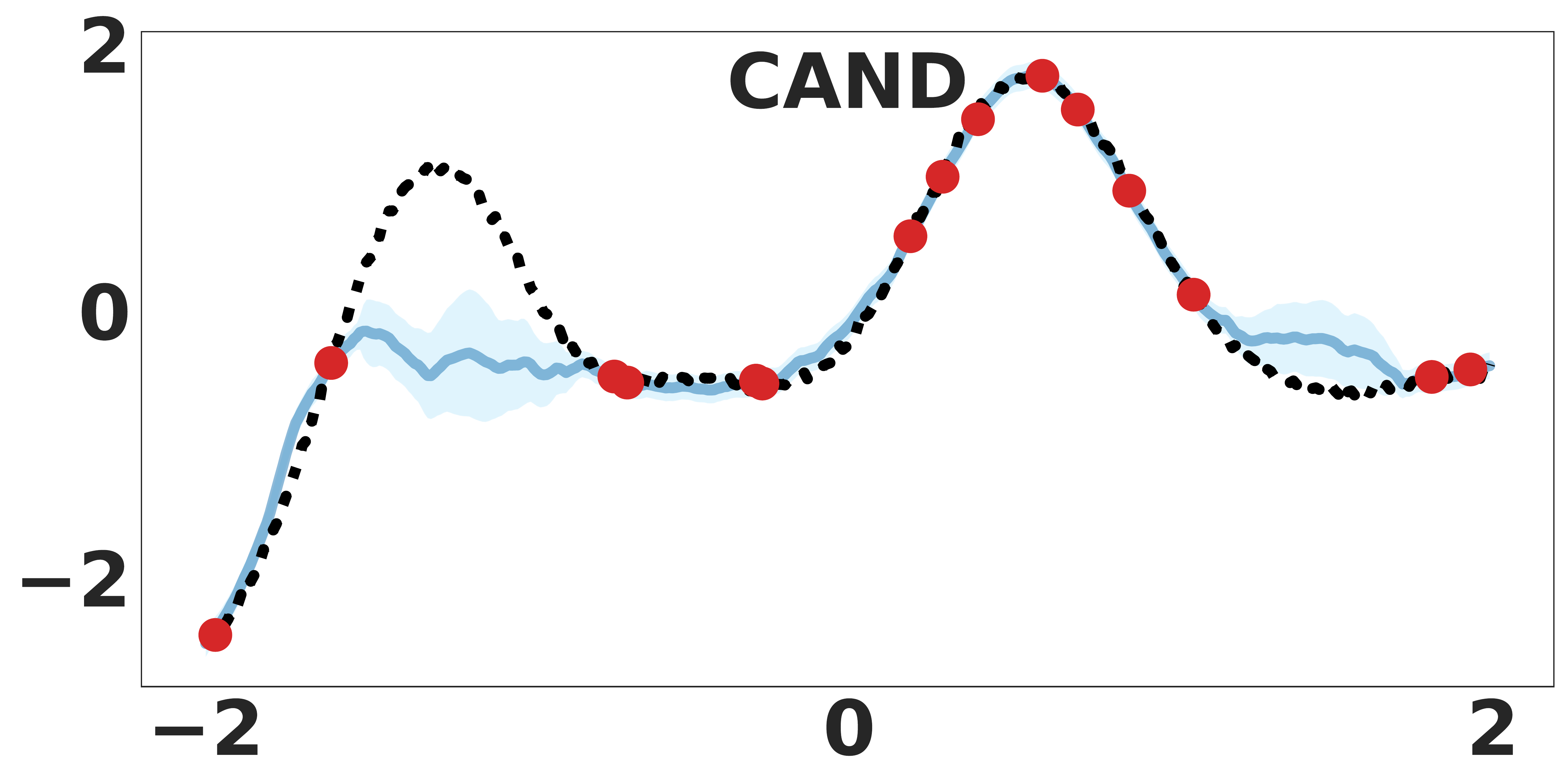}
	\end{subfigure}%
    \begin{subfigure}{0.25\textwidth}
		\centering
		\includegraphics[width=\linewidth]{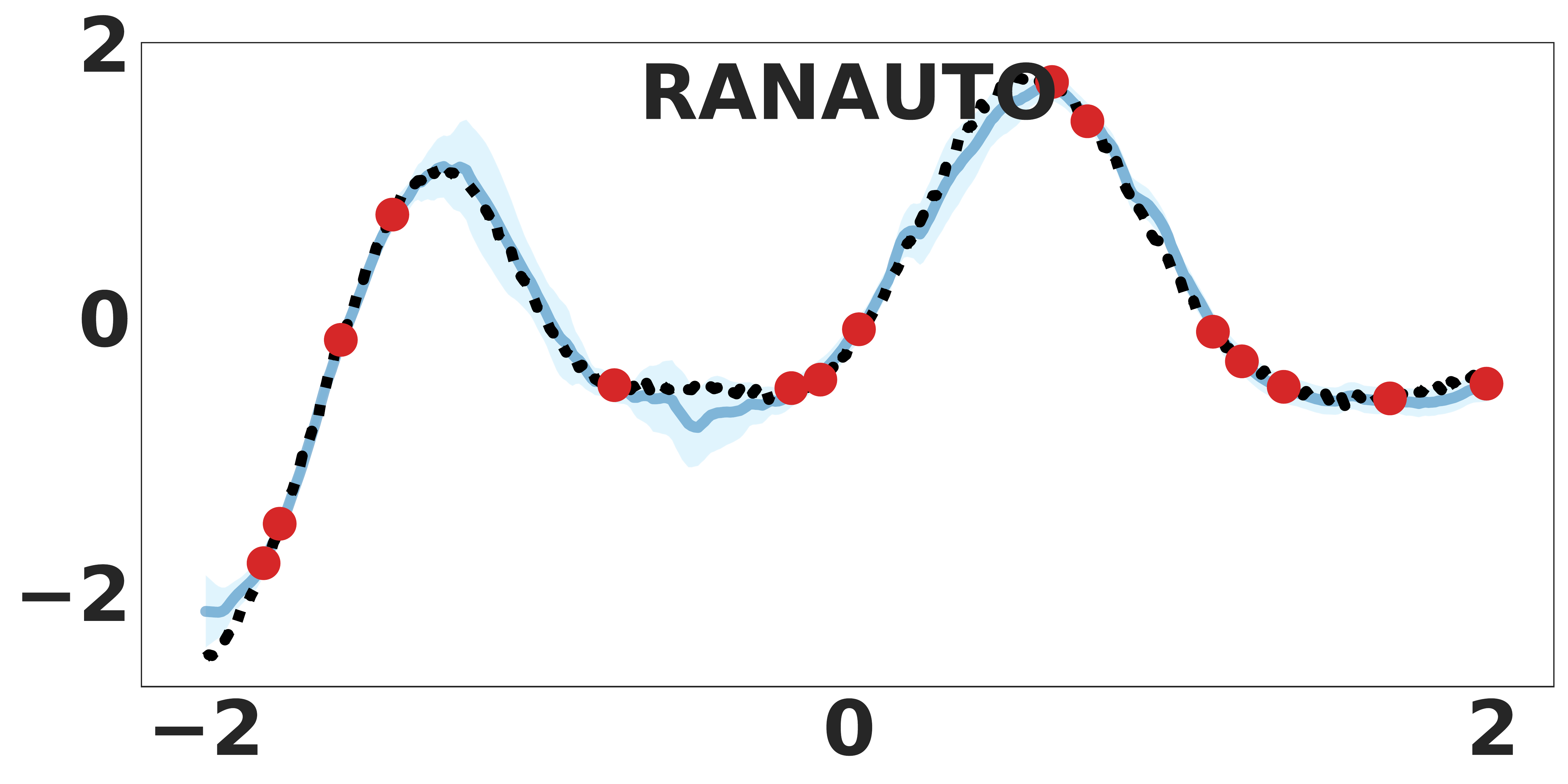}
	\end{subfigure}%
    \begin{subfigure}{0.25\textwidth}
		\centering
		\includegraphics[width=\linewidth]{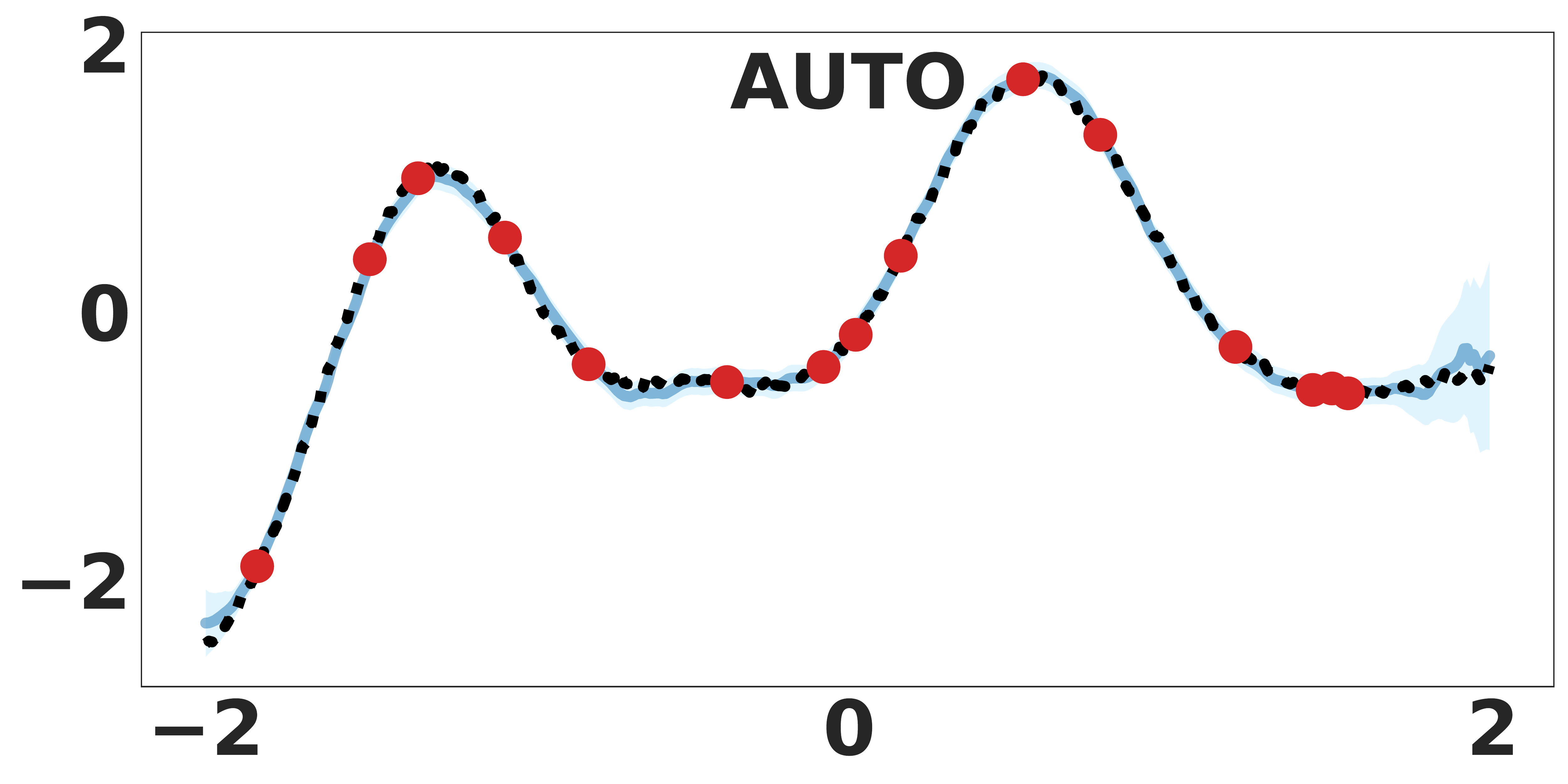}
	\end{subfigure}%
	\begin{subfigure}{0.25\textwidth}
		\centering
		\includegraphics[width=\linewidth]{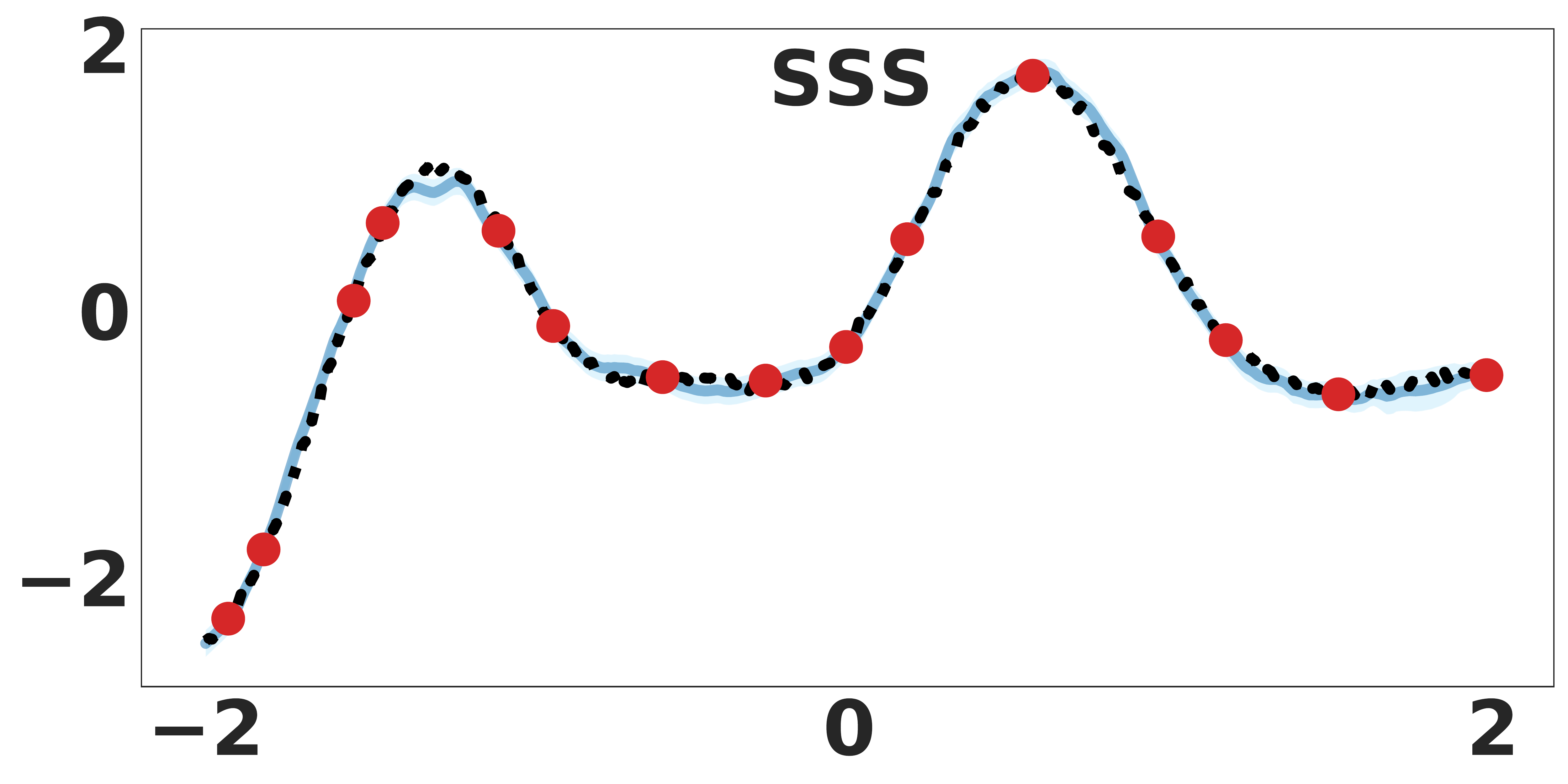}
	\end{subfigure}
	\caption{\small \textbf{Ablation: }Visualization of 1D function reconstruction. In the ablation studies, we compare the first stage
	 (\textbf{CAND}) and the second stage (\textbf{AUTO}) with \textbf{SSS}. Additionally, we replace the Candidate Selection stage (Stage 1) in 
	 SSS with random selection (\textbf{RANAUTO}) and compare the performance of these models. As can be seen from the visualized reconstructed 
	 outputs, the combination of the Candidate Seletion stage with the Autoregressive stage results in the best subset selection for the 
	 reconstruction task.}
	\label{function_reconstruction_ablation}
	\vspace{-0.1in}
\end{figure*}
\paragraph{Cost of Running SSS} We report the cost of running SSS for pixel subsampling, which has the largest set size (38804 pixels).
In this experiment, we select 500 pixels in total with $l=20$, i.e., we select 20 pixels at once in the second stage described in Section~\ref{sec3-4}. We measure the FLOPS and memory requirements for the forward pass of SSS. We find that the computational cost of running SSS is 8.38 GMac (40\% of the FLOPS for the full model
which is 20 GMac) with 217.09k memory requirement (compared to 958.85k for the full model) which shows that SSS is computationally cheap to run.

\section{Proofs}\label{proofs}

\begin{definition}
Let $\mathfrak{S}_n \coloneqq \{g:[n]\rightarrow [n]\mid g\text{ is bijective} \}$ be a set of all permutation on $n$, where $[n]\coloneqq \{1,\ldots,n\}$. We say $f: X^n \rightarrow Y$ is permutation invariant if and only if for any permutation $\pi \in \mathfrak{S}_n$, $f(\pi(\rvx)) = f(\rvx)$ for all $\rvx \in X^n$.
\end{definition}

\begin{definition}
Let $\mathfrak{S}_n \coloneqq \{g:[n]\rightarrow [n]\mid g\text{ is bijective} \}$ be a set of all permutation on $n$, where $[n]\coloneqq \{1,\ldots,n\}$. We say $f: X^n \rightarrow Y^n$ is permutation equivariant if and only if for any permutation $\pi \in \mathfrak{S}_n$, $\pi(f(\rvx)) = f(\pi(\rvx))$ for all $\rvx \in X^n$.
\end{definition}

\begin{proposition}
Let $D=\{d_i\in\mathbb{R}^d\mid i=1,\ldots,n \}$ be a finite set and let $g$ be an affine transformation with non-linearity. Define a function $k(d_1, \ldots, d_n) \coloneqq \frac{1}{n}\sum_{i=1}^n g(d_i)$. Then the set encoding $D_e=k(d_1, \ldots, d_n)$ is permutation invariant.
\end{proposition}

\begin{proof}
Let a permutation $\pi \in \mathfrak{S}_n$ be given. 
\begin{align}
\begin{split}
    k(d_1,\ldots, d_n) &= \frac{1}{n} \sum_{i=1}^n g(d_i) \\
    &= \frac{1}{n} \sum_{i=1}^n g(d_{\pi(i)}) \\
    &= k(d_{\pi(1)}, \ldots, d_{\pi(n)}) 
\end{split}
\end{align}
\end{proof}

\begin{proposition}
Let $f:X^n \rightarrow Y^n$ and $g:Y^n\rightarrow Z^n$ be permutation equivarant functions. Then the composition of two function $g\circ f$ is permutation equivariant function.
\label{prop-comp-appendix}
\end{proposition}

\begin{proof}
Suppose that $f$ and $g$ are permutation equivariant.  Let $\pi \in \mathfrak{S_n}$ be permutation. We want to show $\pi((g\circ f)(\rvx)) = (g\circ f)(\pi(\rvx))$ for all $\rvx \in X^n$. Let $\rvx \in X^n$ be given.
\begin{align}
\begin{split}
    (g\circ f)(\pi(\rvx)) &= g(f(\pi(\rvx))) \\
    &= g(\pi(f(\rvx))) \\
    &= \pi(g(f(\rvx))) \\
    &= \pi((g\circ f)(\rvx))
\end{split}
\end{align}
The second and third equality holds since $f$ and $g$ are permutation equivariant.
\end{proof}

\begin{proposition}
Let $\zeta: X^n \rightarrow Y^n$ be a function, mapping $(d_1, \ldots, d_n)\mapsto(\rho(\overline{d}_1), \ldots, \rho(\overline{d}_n))$ for candidate selection. Then $\zeta$ is permutation equivariant. 
\label{lemma-1}
\end{proposition}

\begin{proof}
Let a permutation $\pi \in \mathfrak{S}_n$ be given and let $\rvx=(d_1, \ldots, d_n)$ be given. 
\begin{align}
\begin{split}
    \zeta(\pi(\rvx)) &= \left(\rho(\overline{d}_{\pi(1)}), \ldots, \rho(\overline{d}_{\pi(n)})\right) \\
    &= \left(\sigma(h(\overline{d}_{\pi(1)})), \ldots, \sigma(h(\overline{d}_{\pi(n)}))  \right) \\ 
    &= \left(\sigma(h(d_{\pi(1)};D_e)), \ldots, \sigma(h(d_{\pi(n)};D_e))  \right) \\ 
    &= \pi\left(\sigma(h(d_{1};D_e)), \ldots, \sigma(h(d_{n};D_e)) \right) \\
    &= \pi\left(\zeta(\rvx)\right)
\end{split}
\end{align}
where $;$ denotes concatenation of two vectors  and $D_e = \frac{1}{n}\sum_{i=1}^n g(d_i)$. $g(\cdot)$ and $h(\cdot)$ affine transformation followed by non-linear activation. Since element-wise operations $h(\cdot), \sigma(\cdot)$ are permutation equivariant and $(d_1,\ldots, d_n)\mapsto(d_1;D_e, \ldots, d_n;D_e)$ is permutation equivariant, composition of those functions is also permutation equivariant by Proposition~\ref{prop-comp-appendix}. Therefore, fourth equality holds, i.e., $\zeta$ is permutation equivariant.
\end{proof}

\begin{proposition}
The probability $p_\theta(Z|D)$ induced by candidate selection function is exchangeable.
\label{prop-can}
\end{proposition}

\begin{proof}
Let a permutation $\pi \in \mathfrak{S}_n$ be given.
\begin{align}
    p_\theta(Z|D) &= \prod_{i=1}^n p_\theta(z_i|d_i, D) \\
    &= \prod_{i=1}^n \rho(d_i;D_e) \\
    &= \prod_{i=1}^n \rho(d_{\pi(i)};D_e) \label{eq:15} \\
    &= \prod_{i=1}^n p_\theta(z_{\pi(i)}|d_{\pi(i)}, D) \\
    &= p_\theta(\pi(Z)|D)
\end{align}
where $D_e = \frac{1}{n} \sum_{i=1}^n g(d_i)$ and ; denotes concatenation of two vectors. Equality in~\ref{eq:15} holds since $D_e$ is permutation invariant and $\rho$ is element-wise operation. 
\end{proof}

\begin{proposition}
Let $f$ be a stack of multi-head attention blocks from Set Transformer~\citep{lee2018set} and let $\varphi$ be affine transformation. For all time step $t$ in autoregressive selection, the functions $f, \varphi,$ and $\sigma \circ \varphi \circ f$ are permutation equivariant, where $\sigma$ is sigmoid function.
\end{proposition}

\begin{proof}
Since each multi-head attention block in $f$ is permutation equivariant, a stack of the blocks is also permutation equivariant by Proposition~\ref{prop-comp-appendix}. Since we apply $\varphi$ independently to each element in a set, $\varphi$ is permutation equivariant. Similarly, $\sigma$ is permutation equivariant since it is an element-wise operation. As a result, $\sigma \circ \varphi \circ f$ is permutation equivariant again by Proposition~\ref{prop-comp-appendix}. 
\end{proof}

\begin{figure*}[ht]
\centering
	\begin{subfigure}{.2\textwidth}
		\centering
		\includegraphics[width=0.6\linewidth]{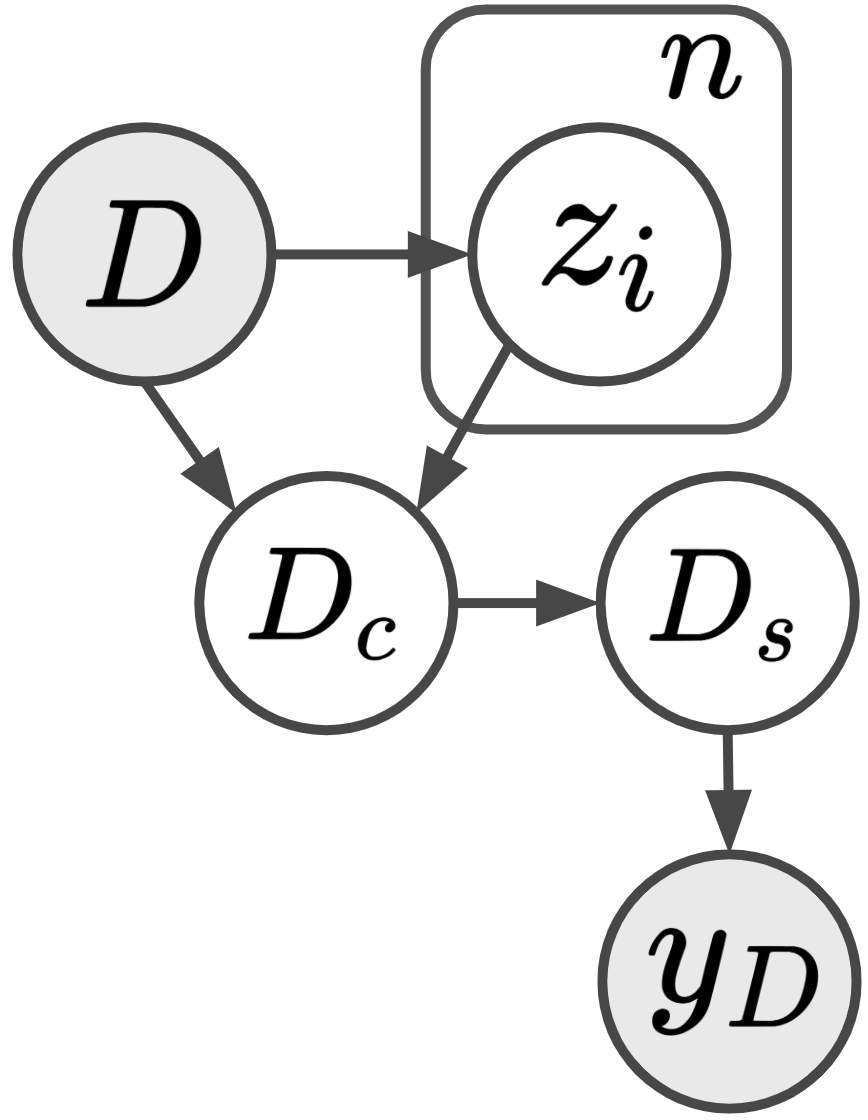}
		\captionsetup{justification=centering,margin=0.5cm}
		\caption{\small}
		\label{gm_pred}
	\end{subfigure}%
	\begin{subfigure}{.2\textwidth}
		\centering
		\includegraphics[width=0.8\linewidth]{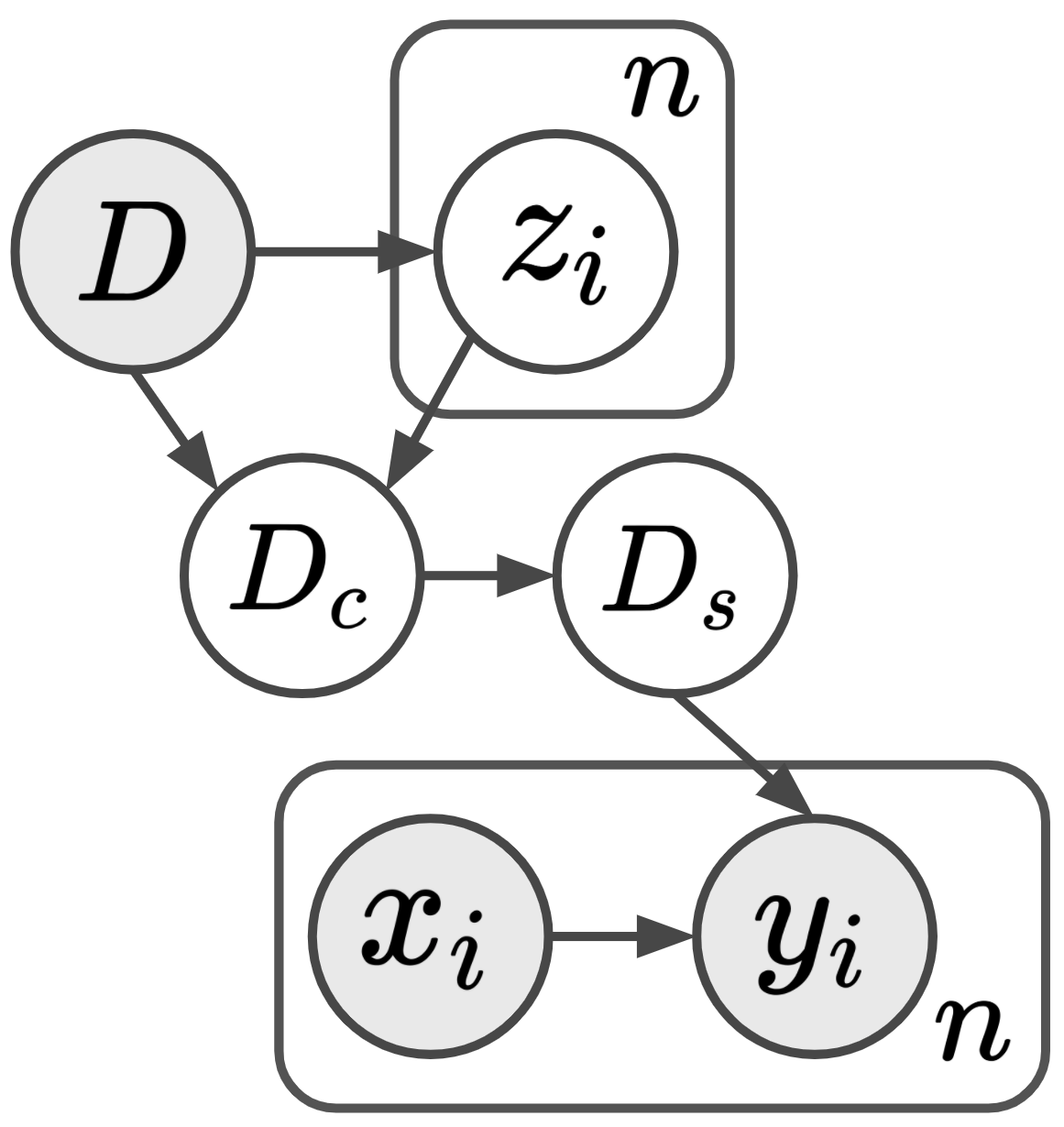}
		\captionsetup{justification=centering,margin=0.5cm}
		\caption{\small}
		\label{gm_rec}
	\end{subfigure}%
	\begin{subfigure}{.3\textwidth}
		\centering
		\includegraphics[width=0.8\linewidth]{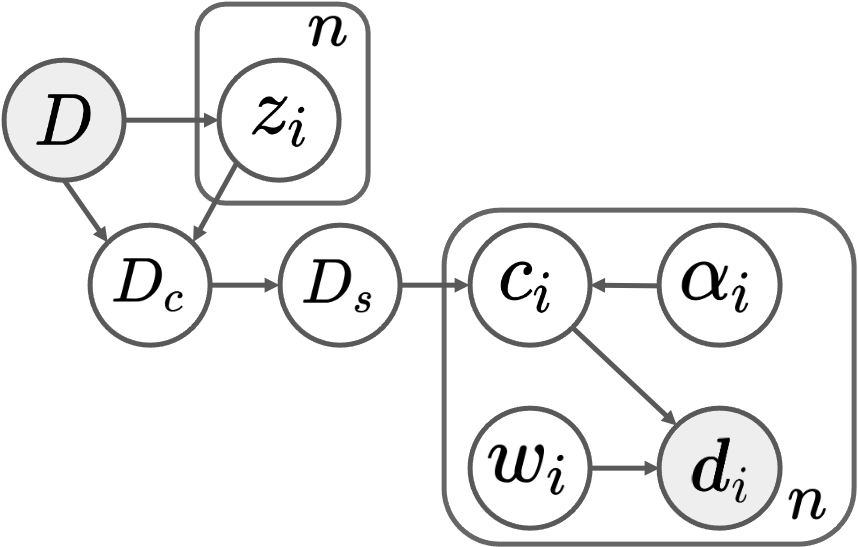}
		\captionsetup{justification=centering,margin=0.5cm}
		\caption{\small}
		\label{celeba-distillation}
	\end{subfigure}
	\begin{subfigure}{.2\textwidth}
		\centering
		\includegraphics[width=0.8\linewidth]{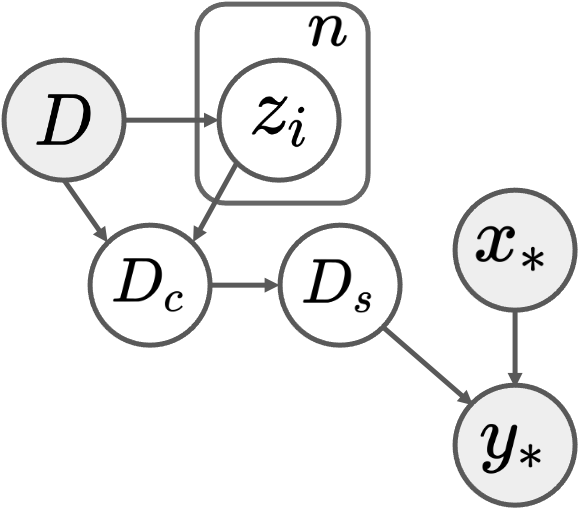}
		\captionsetup{justification=centering,margin=0.5cm}
		\caption{\small}
		\label{mini-distillation}
	\end{subfigure}
	\caption{\small \textbf{Graphical Models:} \textbf{(a)} Feature selection for reconstruction. \textbf{(b)} Feature Selection for prediction
	task. \textbf{(c)} Instance selection for representative data points. \textbf{(d)} Instance selection for few-shot classification.}
	\label{zm}
\end{figure*}

\section{Graphical Model}\label{graphical-model}
In Figure~\ref{zm}, we illustrate the generative process using graphical models for each tasks --- (a) feature selection for set reconstruction, (b) feature selection for prediction (c) Instance selection for representative data points and (d) instance selection for few-shot classification. The shaded circles denote observed variables and the others latent variables.

\section{Instance Selection Samples}
In this section, we show more qualitative examples for the 1D and CelebA experiments on how the models subsamples elements of the given set for the target task.

\subsection{1D Function - Reconstruction}
\begin{figure*}
\vspace{-0.1in}
    \centering
    \begin{subfigure}{0.4\textwidth}
		\centering
		\includegraphics[width=\linewidth]{images/function/plots/legend_func.pdf}
		\vspace{-0.2in}
	\end{subfigure}
    \begin{subfigure}{0.25\textwidth}
		\centering
		\includegraphics[width=\linewidth]{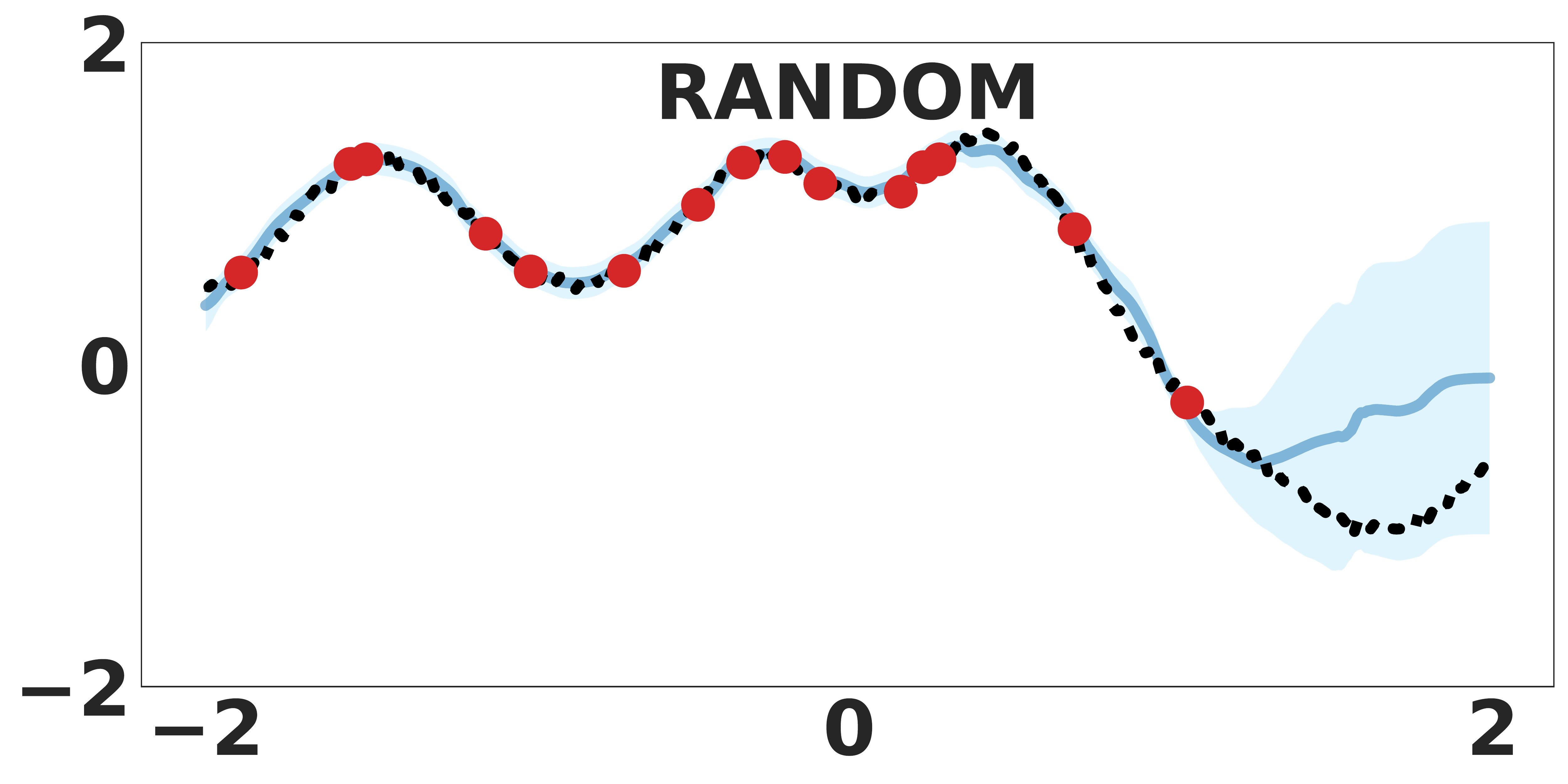}
	\end{subfigure}%
    \begin{subfigure}{0.25\textwidth}
		\centering
		\includegraphics[width=\linewidth]{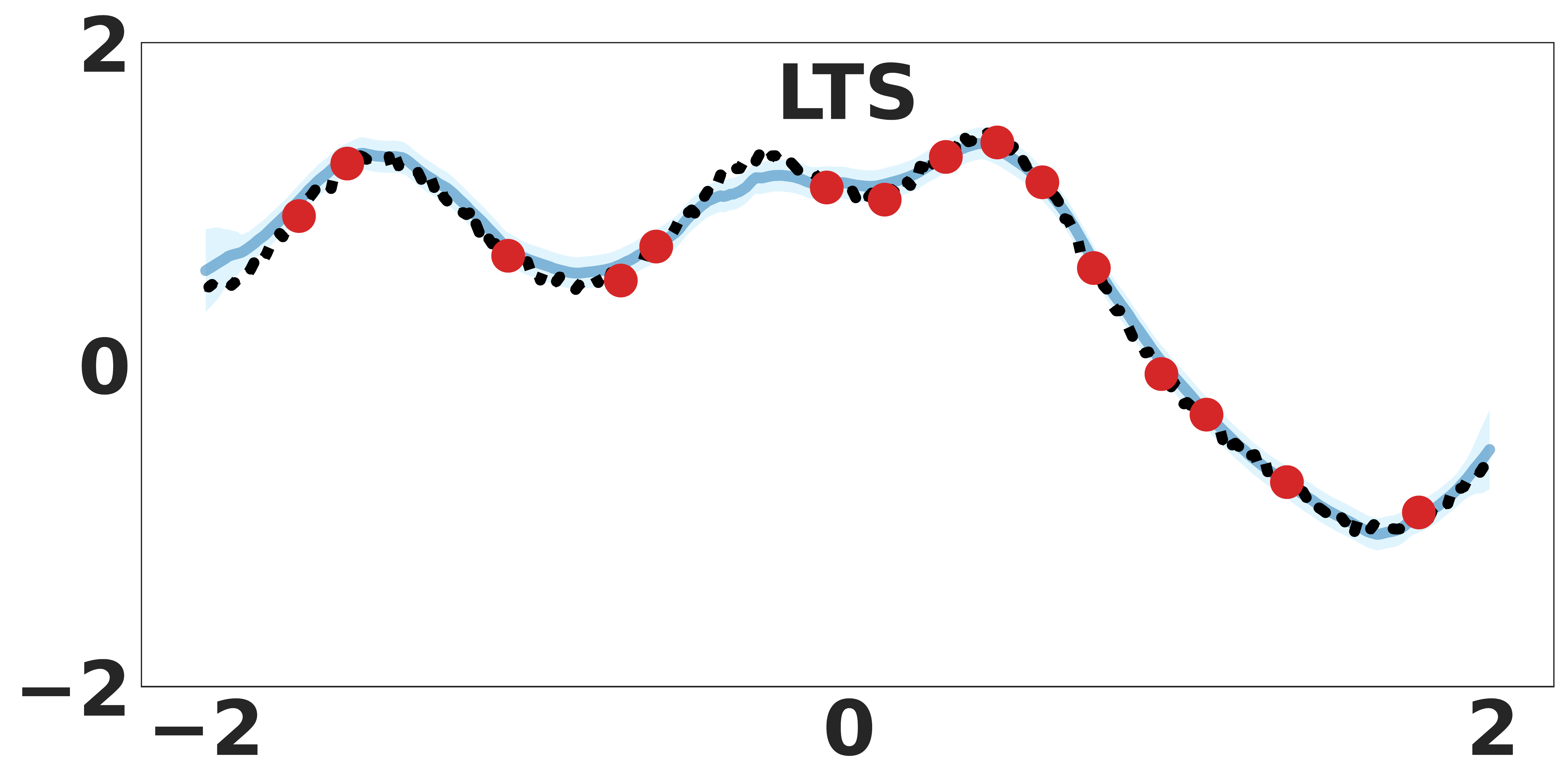}
	\end{subfigure}%
    \begin{subfigure}{0.25\textwidth}
		\centering
		\includegraphics[width=\linewidth]{images/function/appendix/sss1.pdf}
	\end{subfigure}
	    \begin{subfigure}{0.25\textwidth}
		\centering
		\includegraphics[width=\linewidth]{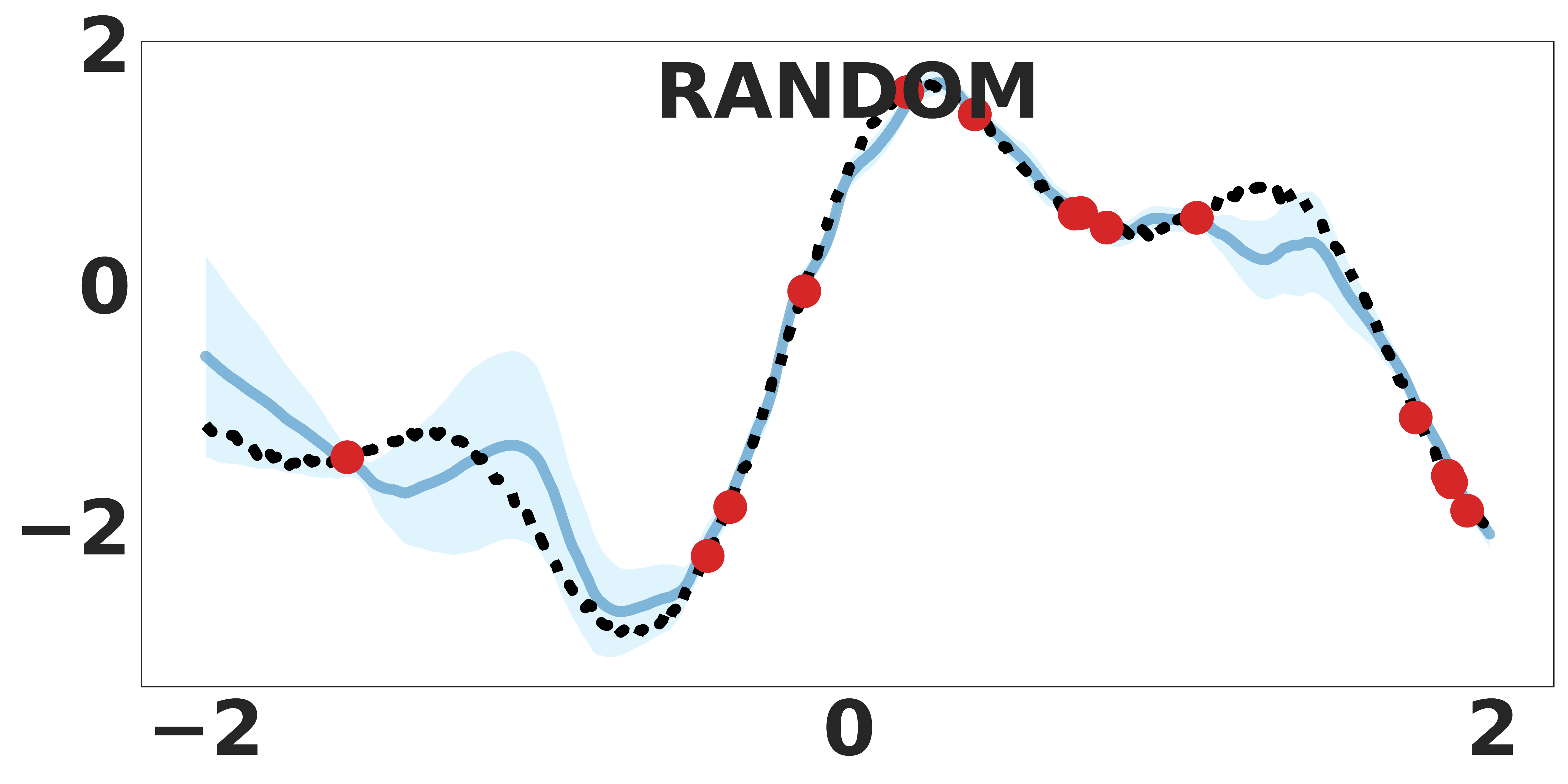}
	\end{subfigure}%
    \begin{subfigure}{0.25\textwidth}
		\centering
		\includegraphics[width=\linewidth]{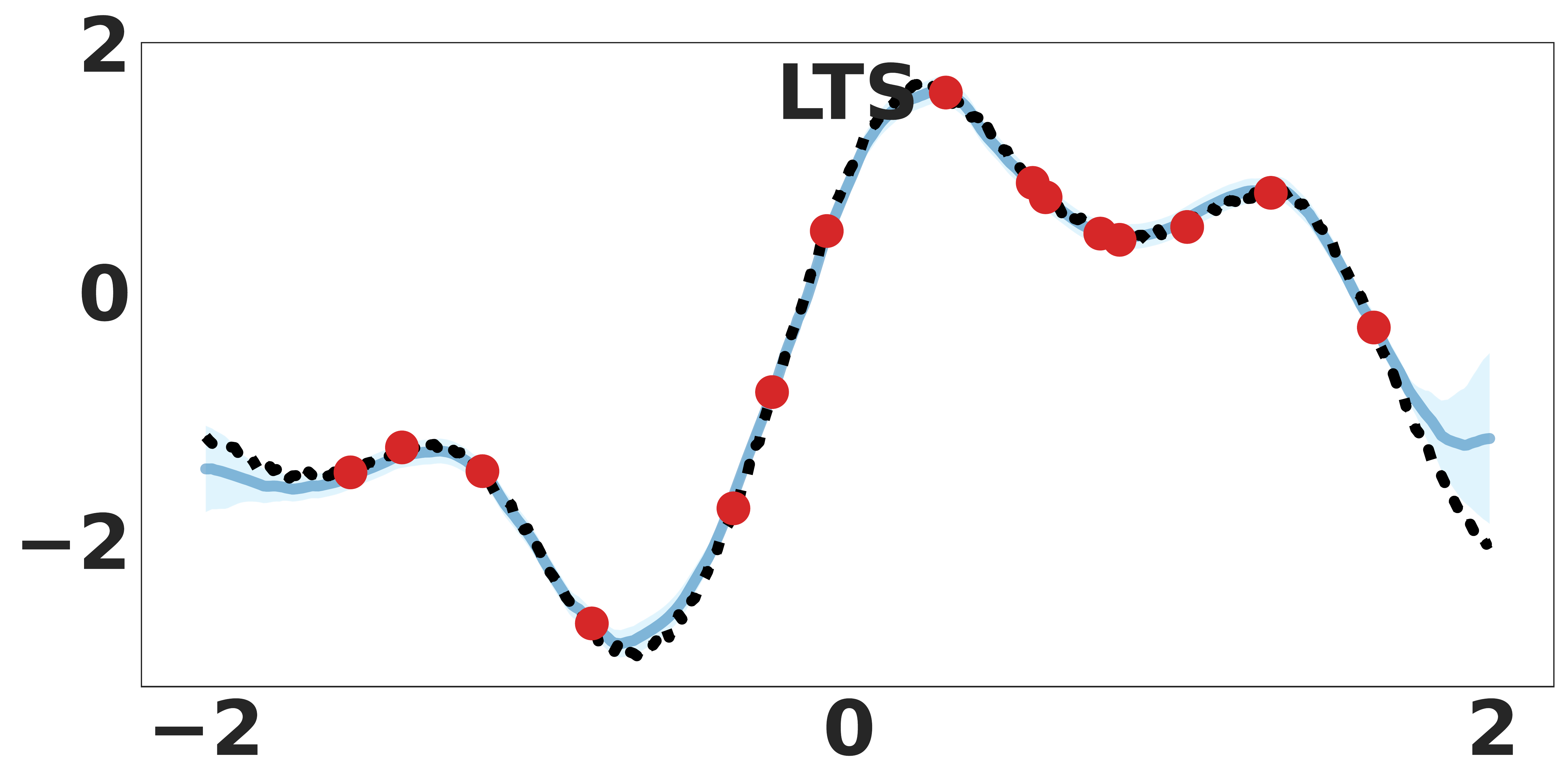}
	\end{subfigure}%
    \begin{subfigure}{0.25\textwidth}
		\centering
		\includegraphics[width=\linewidth]{images/function/appendix/sss2.pdf}
	\end{subfigure}
	    \begin{subfigure}{0.25\textwidth}
		\centering
		\includegraphics[width=\linewidth]{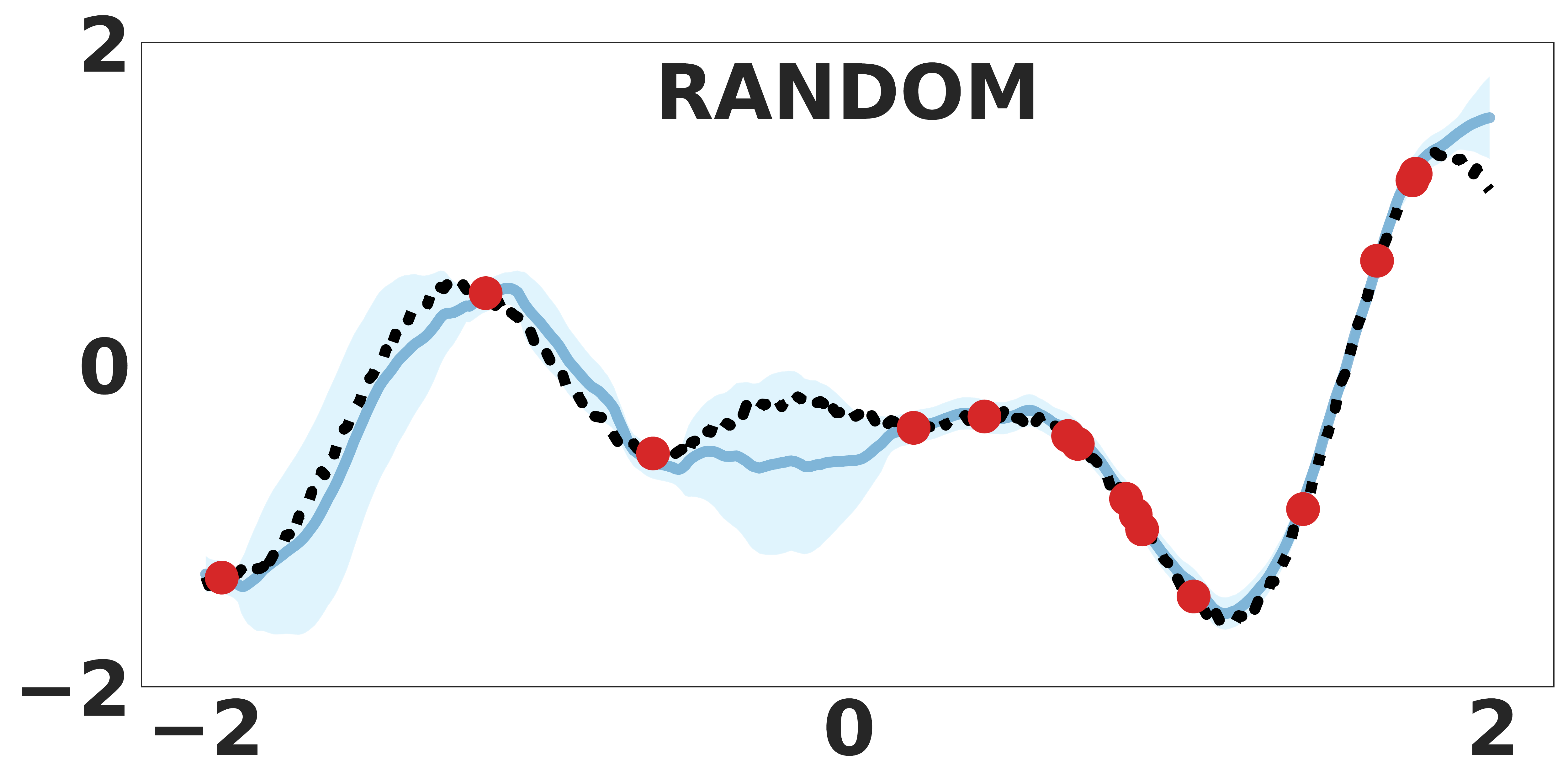}
	\end{subfigure}%
    \begin{subfigure}{0.25\textwidth}
		\centering
		\includegraphics[width=\linewidth]{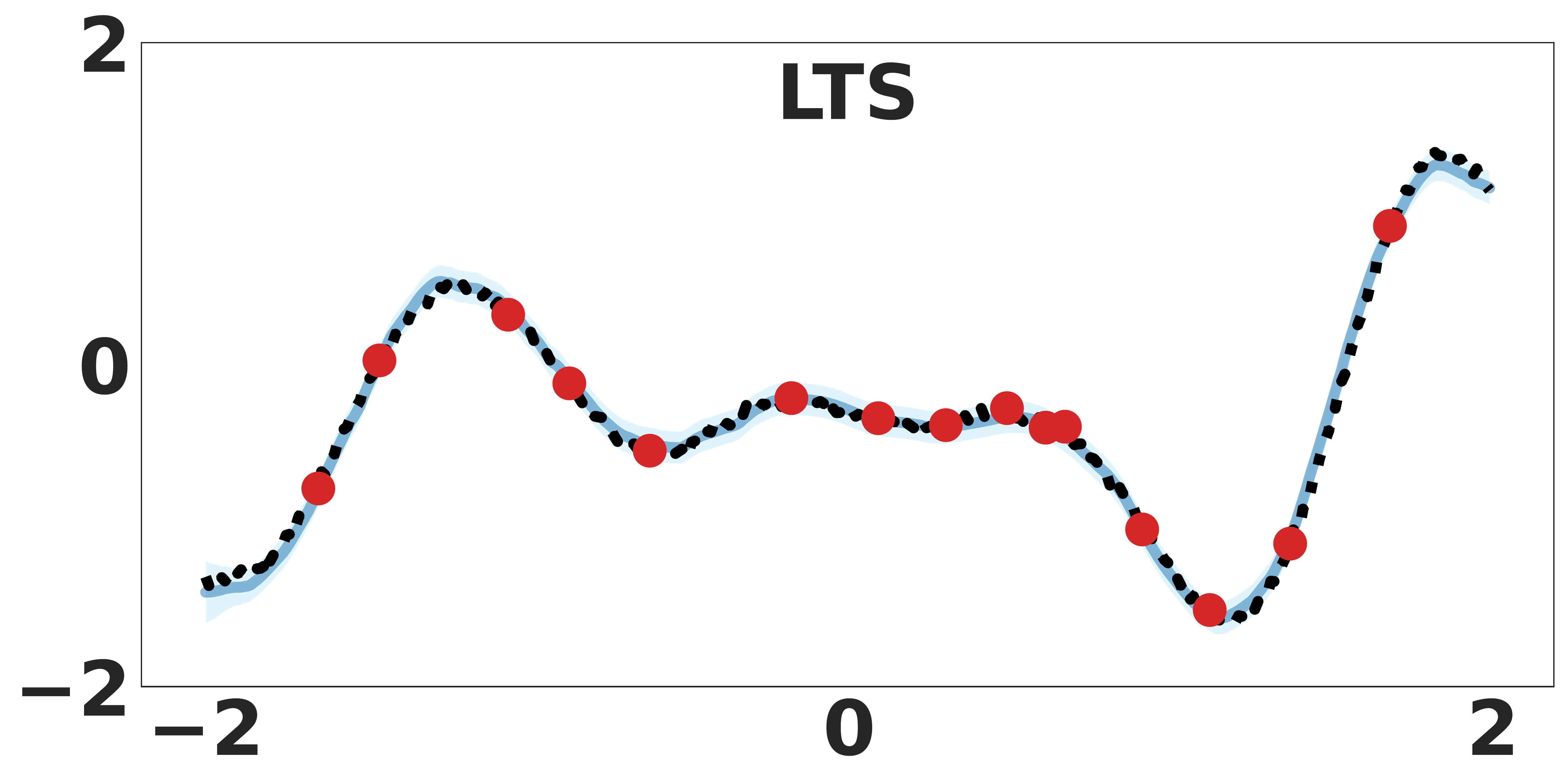}
	\end{subfigure}%
    \begin{subfigure}{0.25\textwidth}
		\centering
		\includegraphics[width=\linewidth]{images/function/appendix/sss3.pdf}
	\end{subfigure}
	    \begin{subfigure}{0.25\textwidth}
		\centering
		\includegraphics[width=\linewidth]{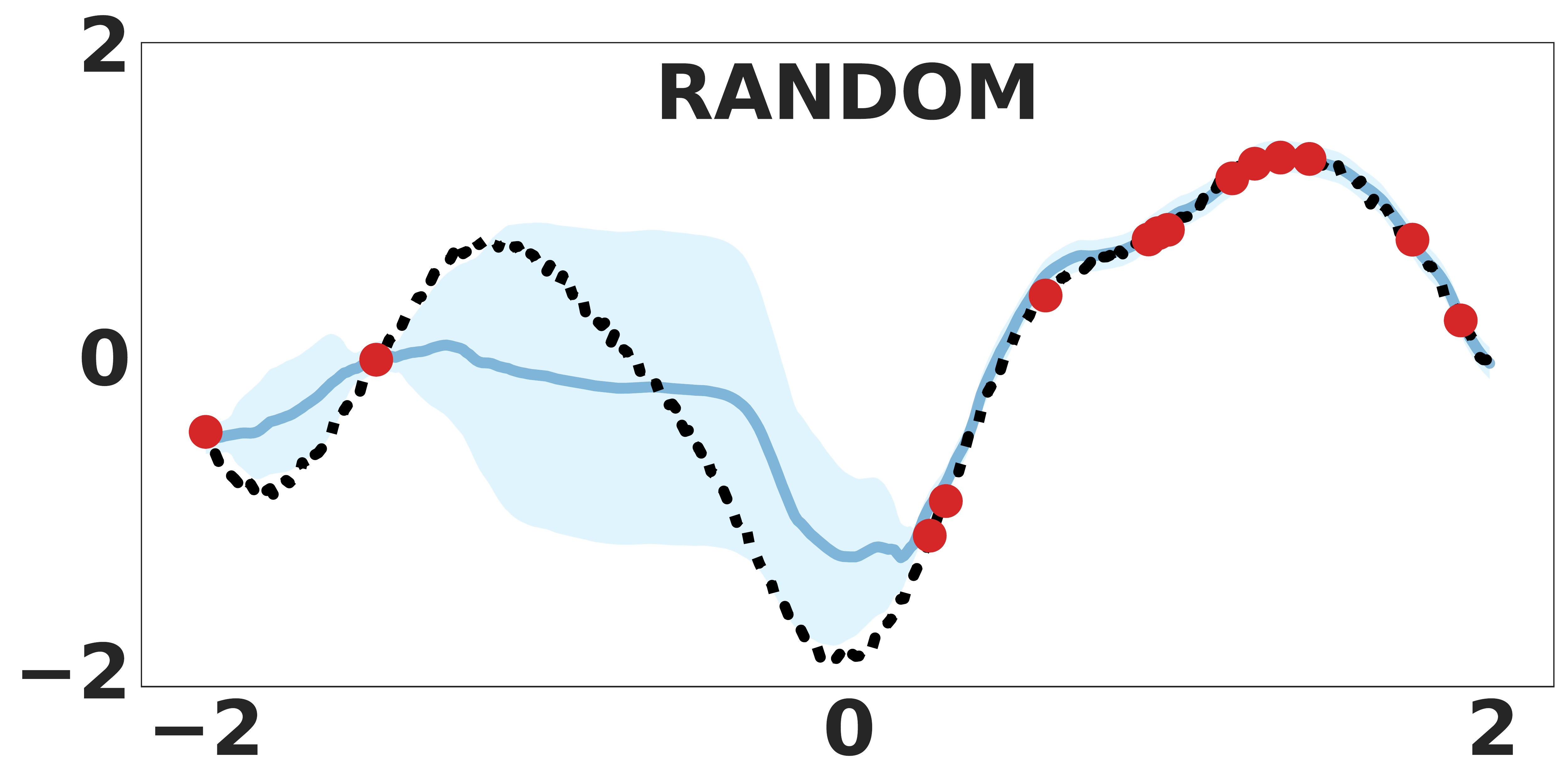}
	\end{subfigure}%
    \begin{subfigure}{0.25\textwidth}
		\centering
		\includegraphics[width=\linewidth]{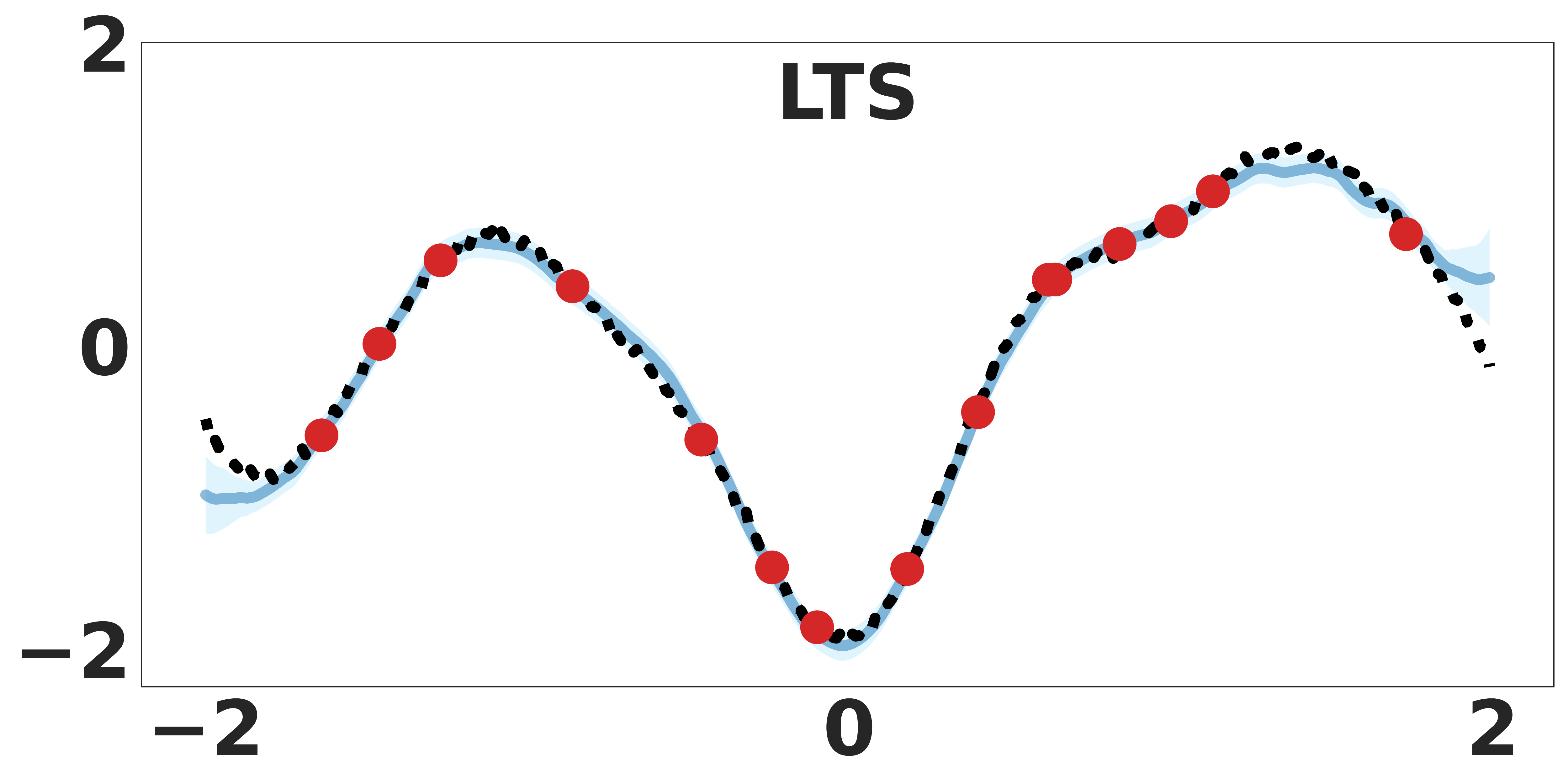}
	\end{subfigure}%
    \begin{subfigure}{0.25\textwidth}
		\centering
		\includegraphics[width=\linewidth]{images/function/appendix/sss4.pdf}
	\end{subfigure}
	\begin{subfigure}{0.25\textwidth}
		\centering
		\includegraphics[width=\linewidth]{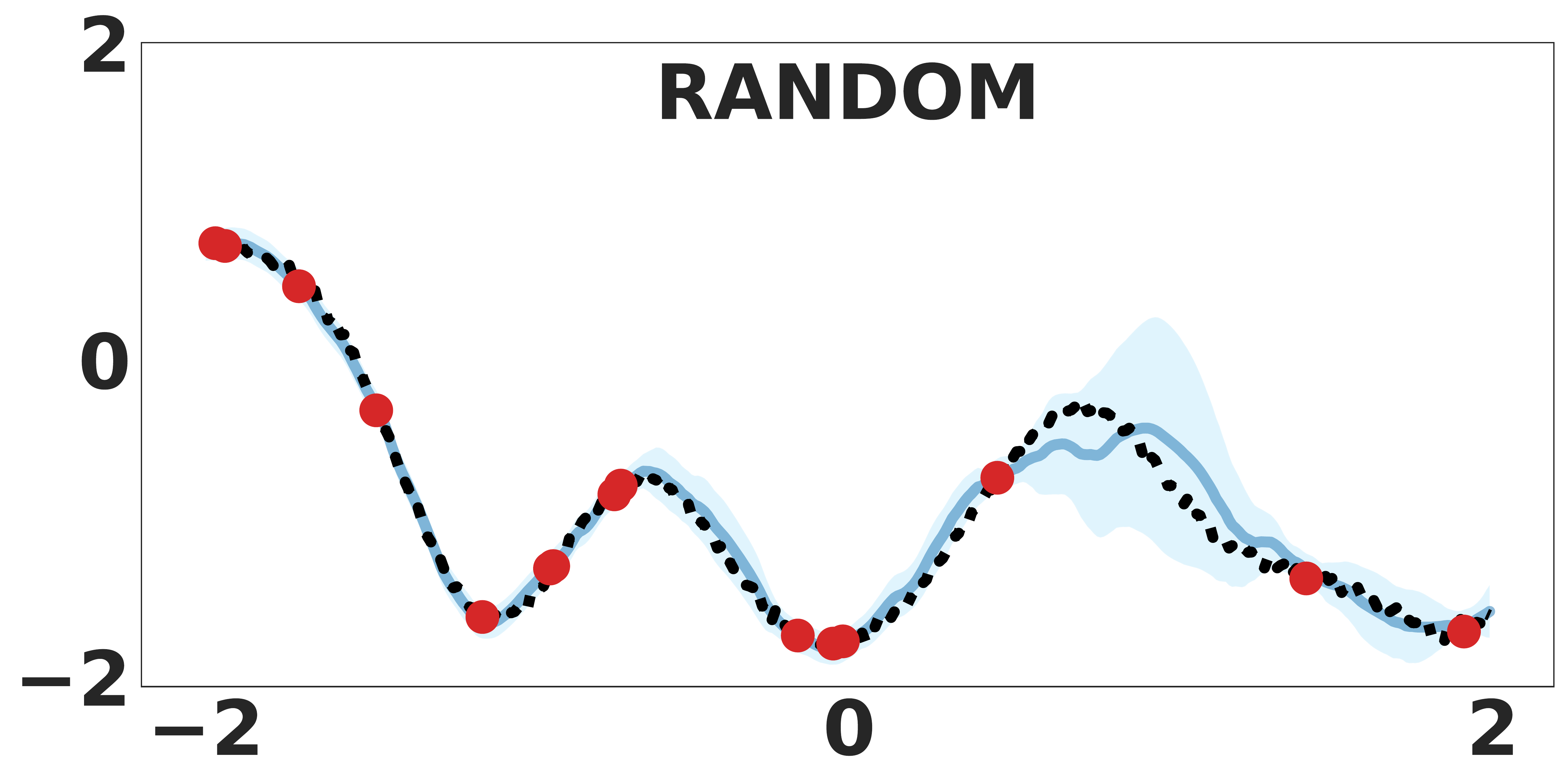}
	\end{subfigure}%
    \begin{subfigure}{0.25\textwidth}
		\centering
		\includegraphics[width=\linewidth]{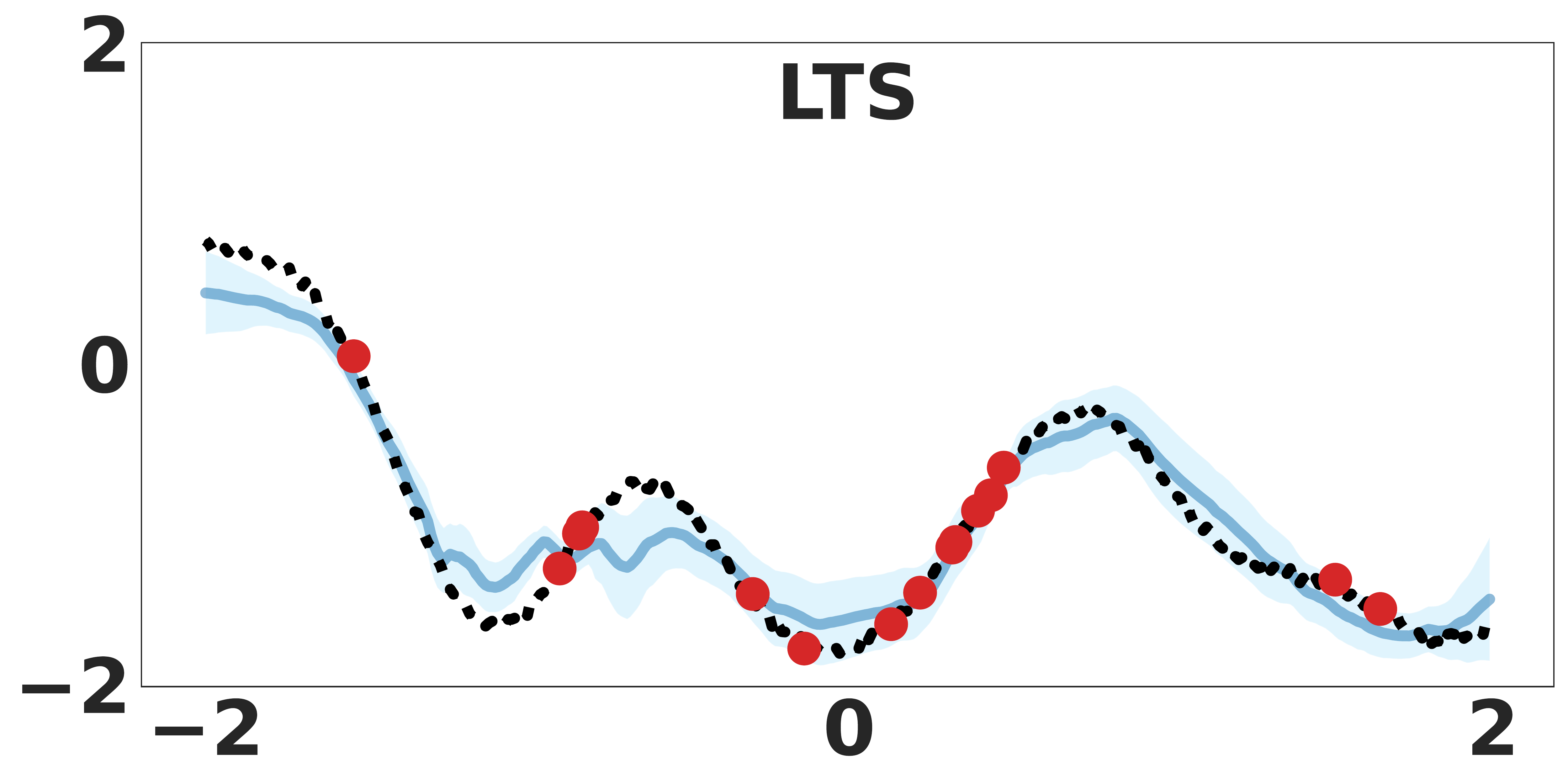}
	\end{subfigure}%
    \begin{subfigure}{0.25\textwidth}
		\centering
		\includegraphics[width=\linewidth]{images/function/appendix/sss5.pdf}
	\end{subfigure}
	\begin{subfigure}{0.25\textwidth}
		\centering
		\includegraphics[width=\linewidth]{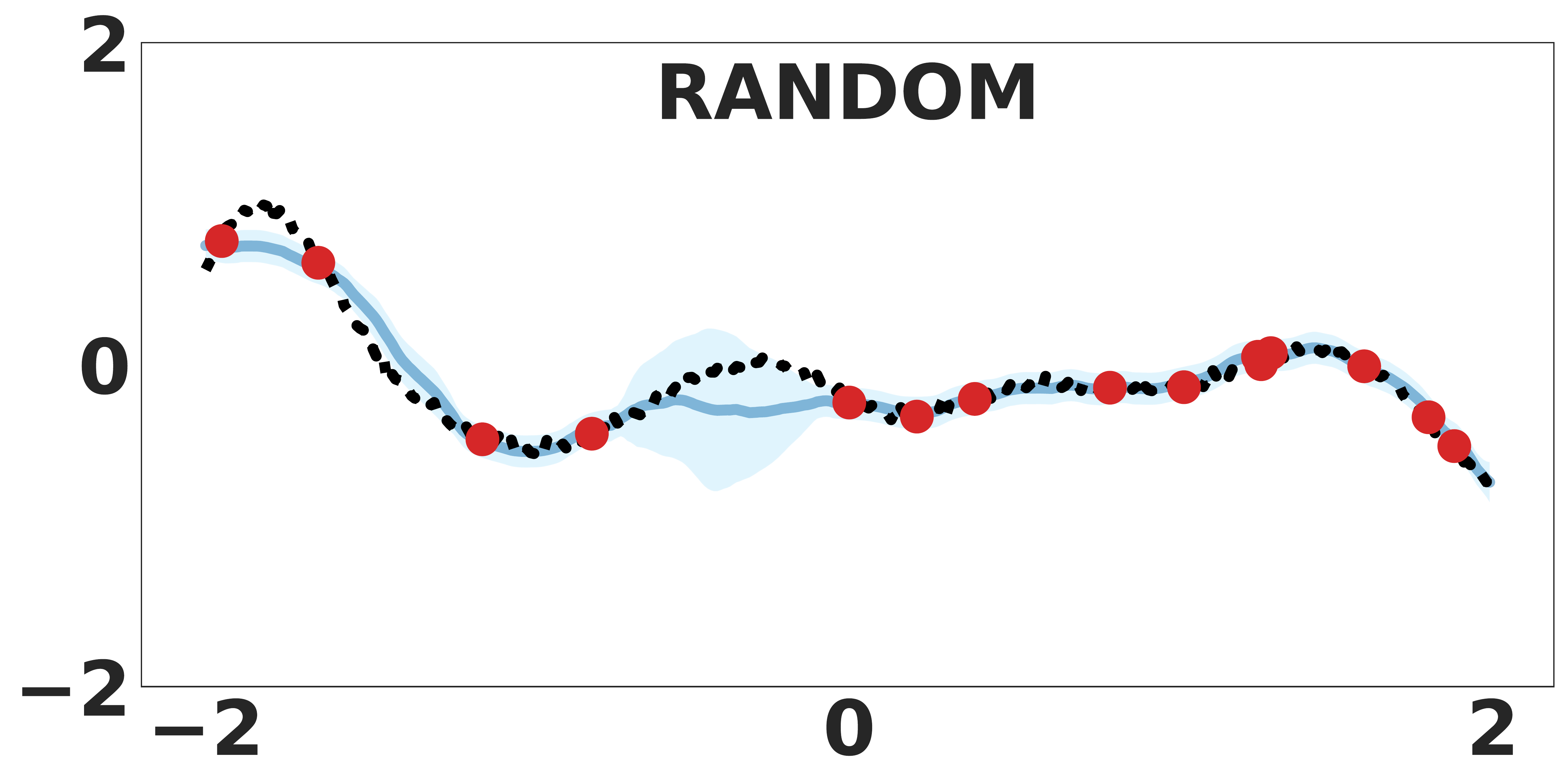}
	\end{subfigure}%
    \begin{subfigure}{0.25\textwidth}
		\centering
		\includegraphics[width=\linewidth]{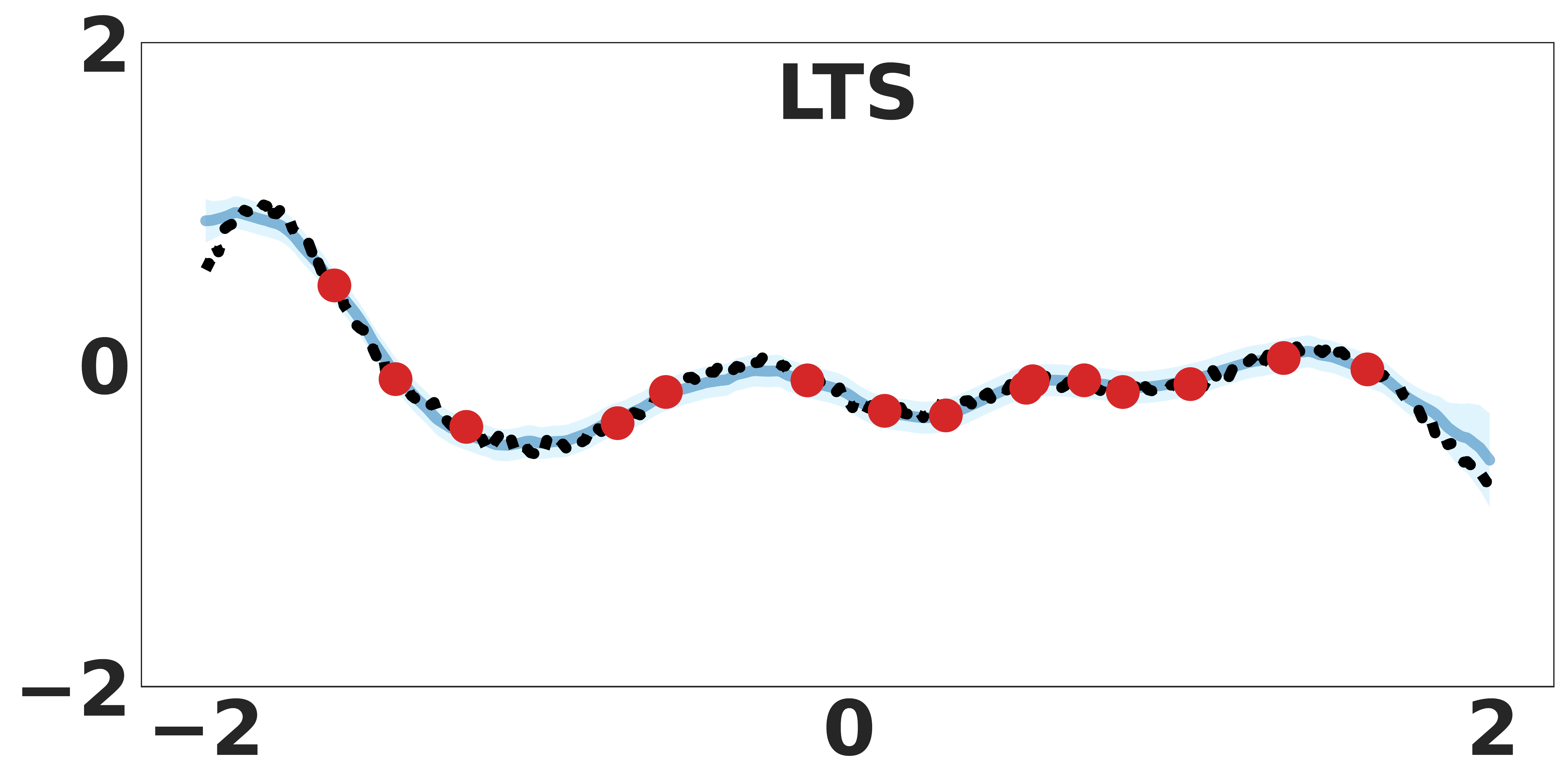}
	\end{subfigure}%
    \begin{subfigure}{0.25\textwidth}
		\centering
		\includegraphics[width=\linewidth]{images/function/appendix/sss6.pdf}
	\end{subfigure}
	\begin{subfigure}{0.25\textwidth}
		\centering
		\includegraphics[width=\linewidth]{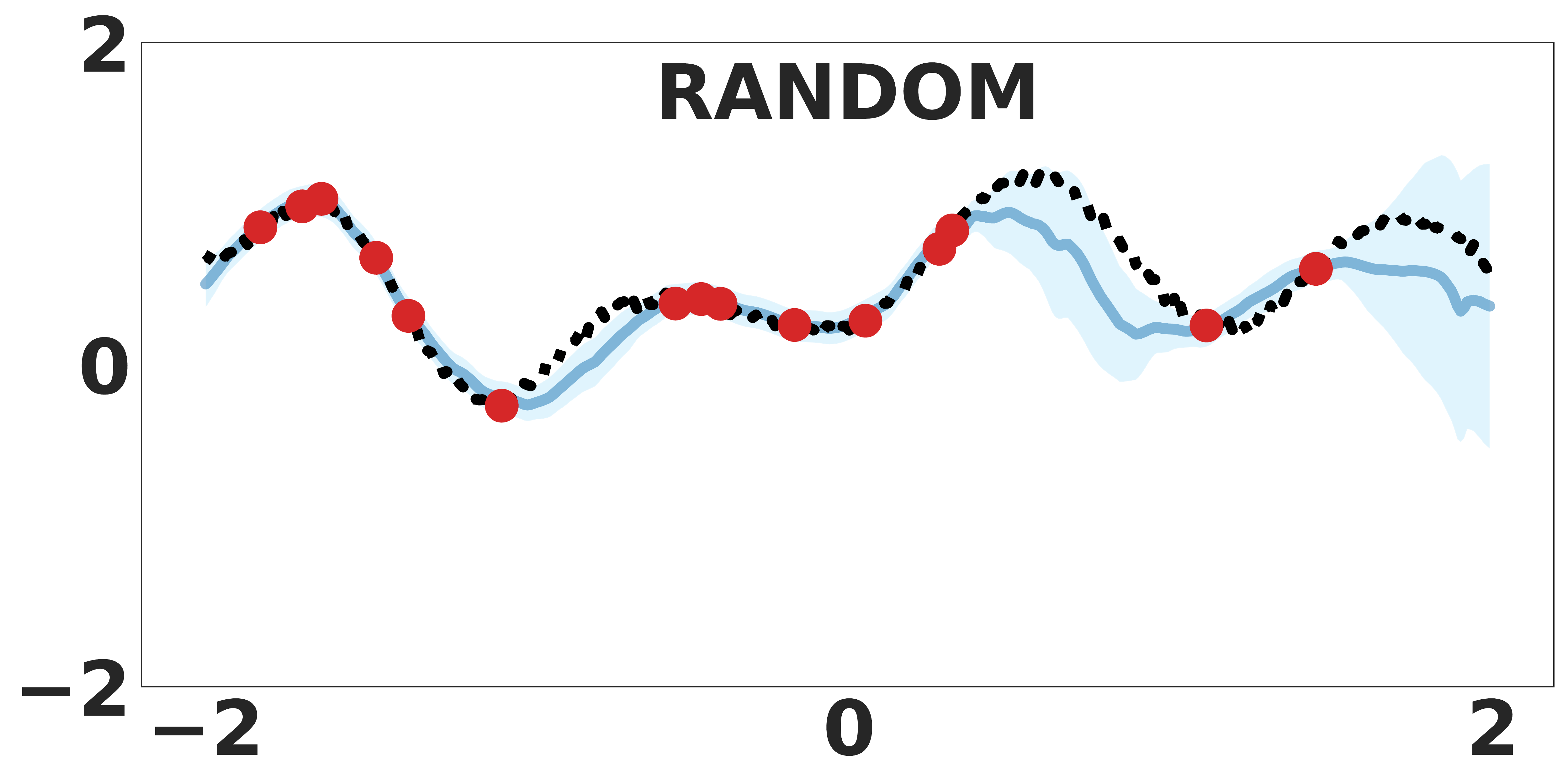}
	\end{subfigure}%
    \begin{subfigure}{0.25\textwidth}
		\centering
		\includegraphics[width=\linewidth]{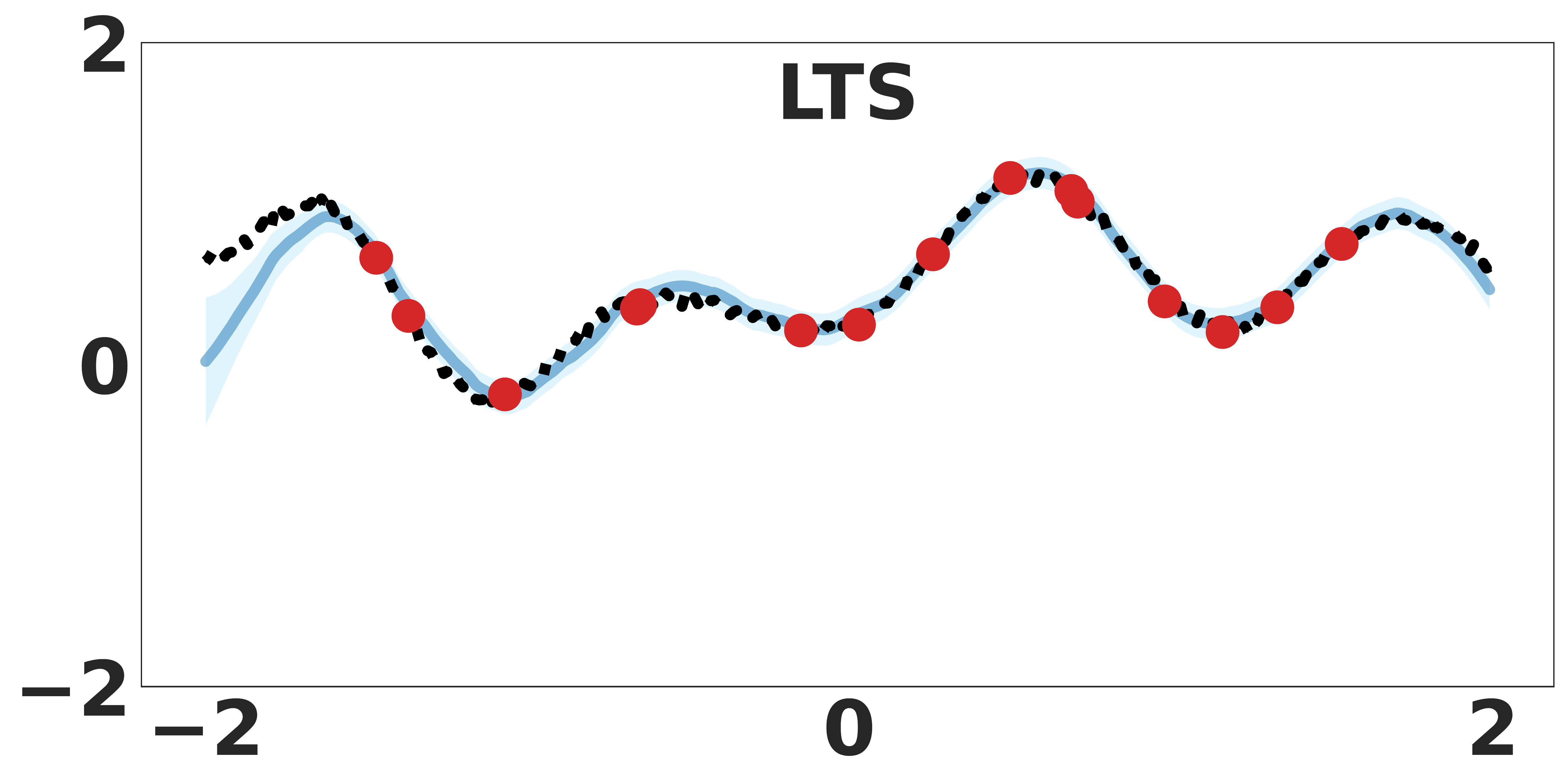}
	\end{subfigure}%
    \begin{subfigure}{0.25\textwidth}
		\centering
		\includegraphics[width=\linewidth]{images/function/appendix/sss7.pdf}
	\end{subfigure}	
	\begin{subfigure}{0.25\textwidth}
		\centering
		\includegraphics[width=\linewidth]{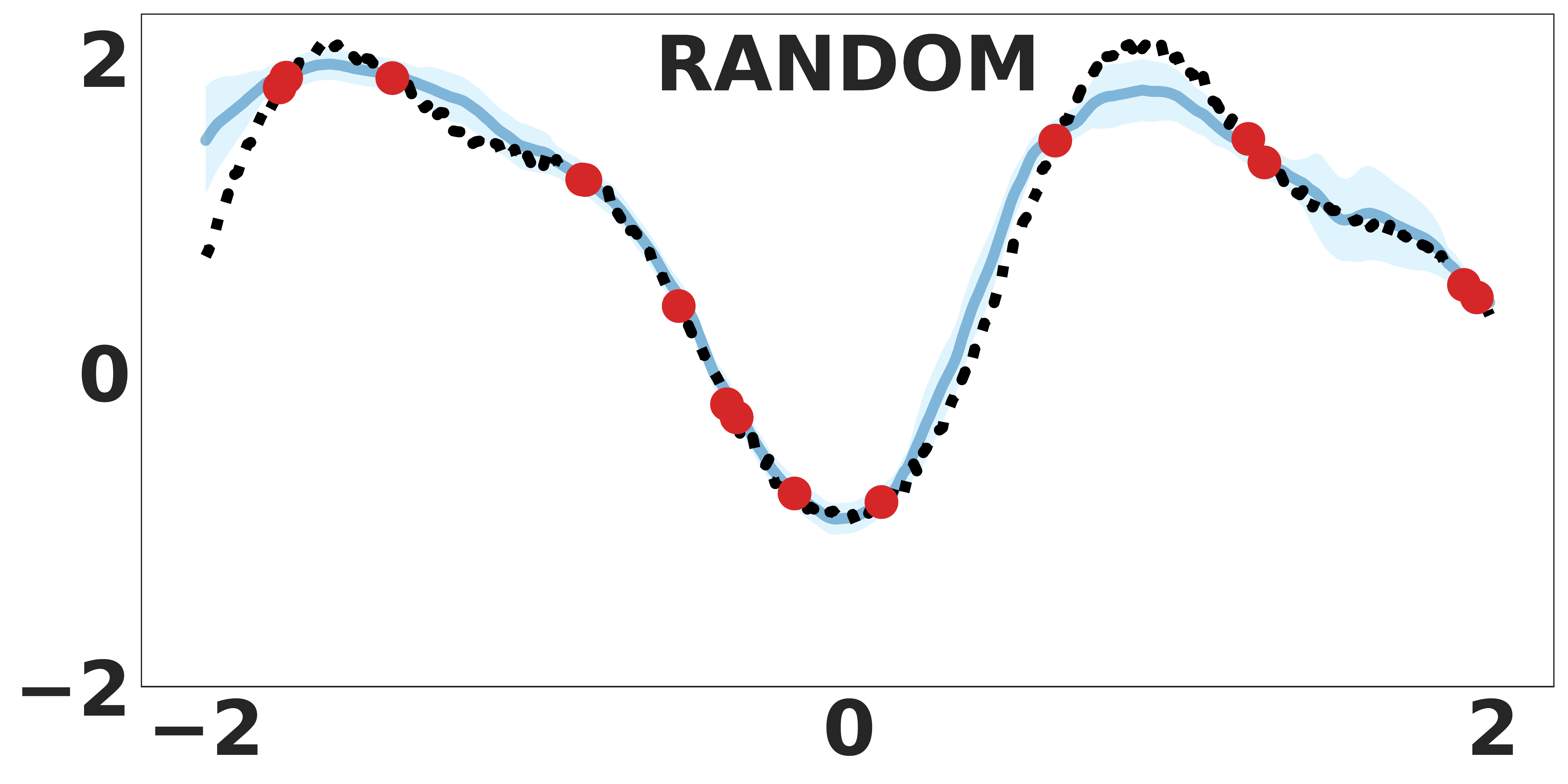}
	\end{subfigure}%
    \begin{subfigure}{0.25\textwidth}
		\includegraphics[width=\linewidth]{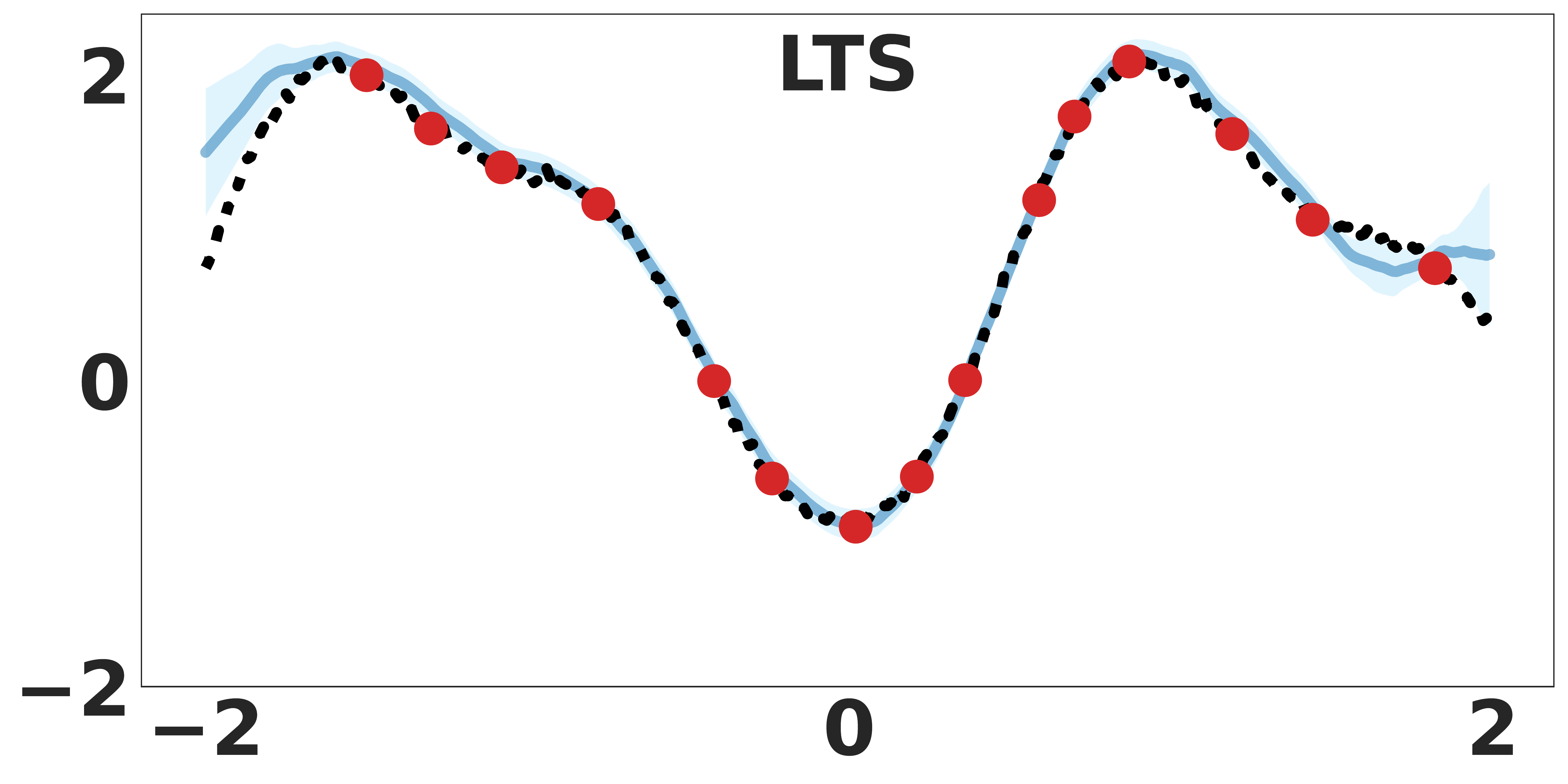}
	\end{subfigure}%
    \begin{subfigure}{0.25\textwidth}
		\centering
		\includegraphics[width=\linewidth]{images/function/appendix/sss8.pdf}
	\end{subfigure}	
	\begin{subfigure}{0.25\textwidth}
		\centering
		\includegraphics[width=\linewidth]{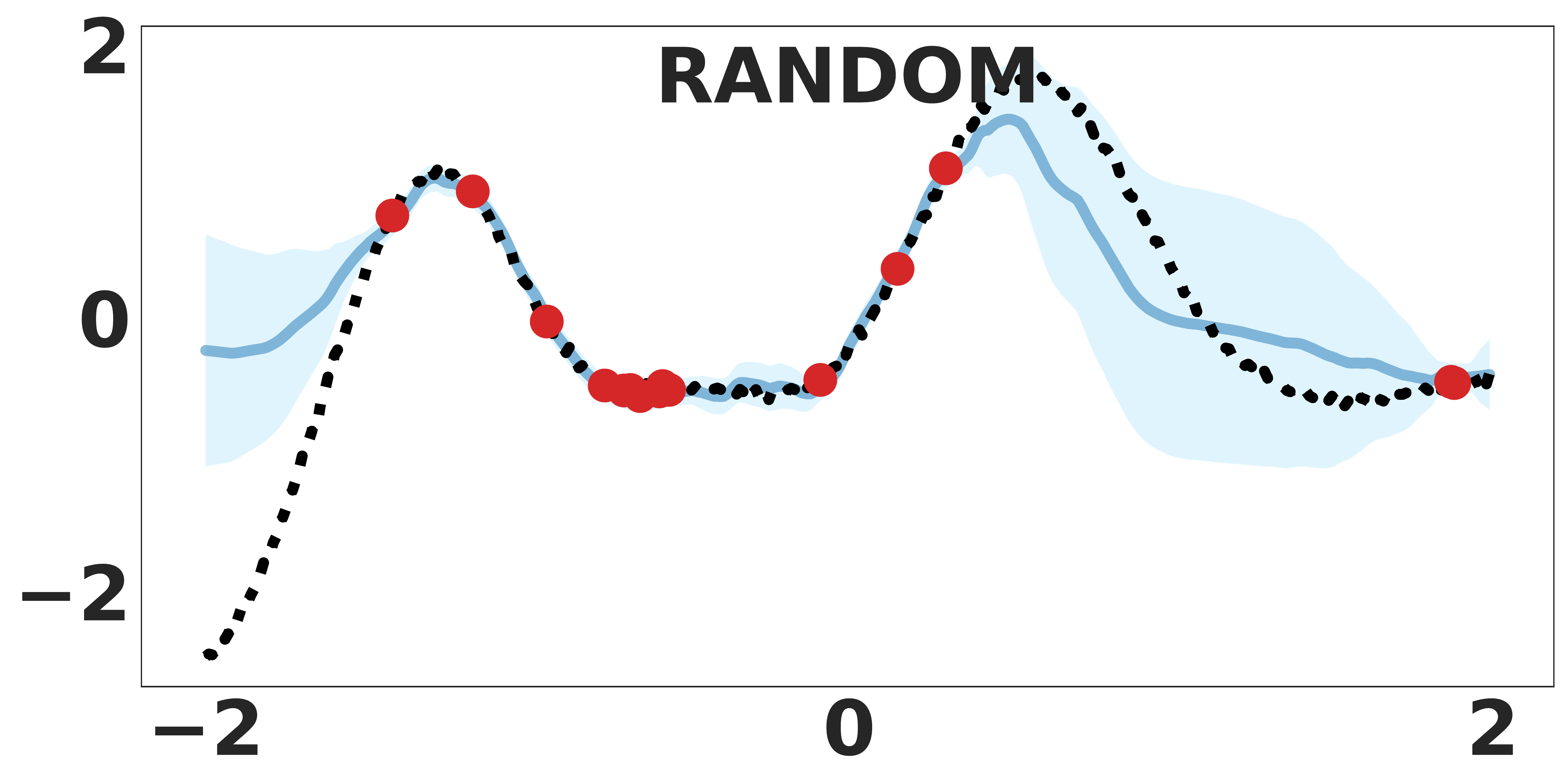}
	\end{subfigure}%
    \begin{subfigure}{0.25\textwidth}
		\centering
		\includegraphics[width=\linewidth]{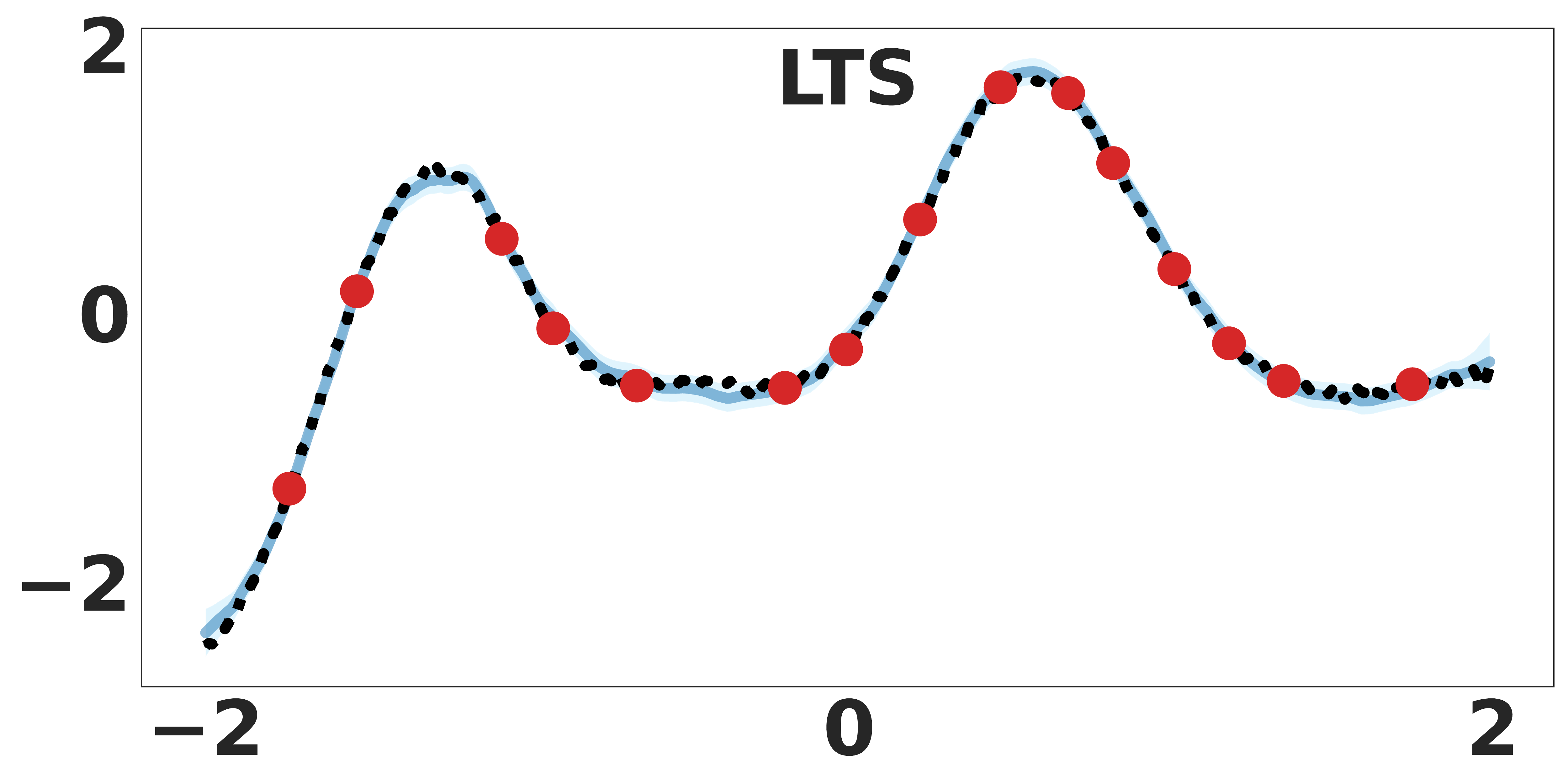}
	\end{subfigure}%
    \begin{subfigure}{0.25\textwidth}
		\centering
		\includegraphics[width=\linewidth]{images/function/appendix/sss9.pdf}
	\end{subfigure}	
	\caption{\small Visualization of 1D function reconstruction with three different subset selection models. Each method selects
	15 out of 400 elements. As can be seen, SSS selects elements that result in better reconstructed functions.}
	\label{reconstruction-functions-appendix}
	\vspace{-0.1in}
\end{figure*}
Figure~\ref{reconstruction-functions-appendix} shows the reconstruction samples of our model on the 1D function dataset, where SSS clearly outperforms Learning to Sample (LTS) and Random Subset (RS). Since RS selects the points randomly, it can leave out important part of the 1D curve leading
to wrong reconstructions. Similarly, LTS also miss some parts of the curves, resulting in suboptimal reconstructions. 

\begin{figure*}
\vspace{-0.1in}
    \centering
    \begin{subfigure}{0.8\textwidth}
		\centering
		\includegraphics[width=\linewidth]{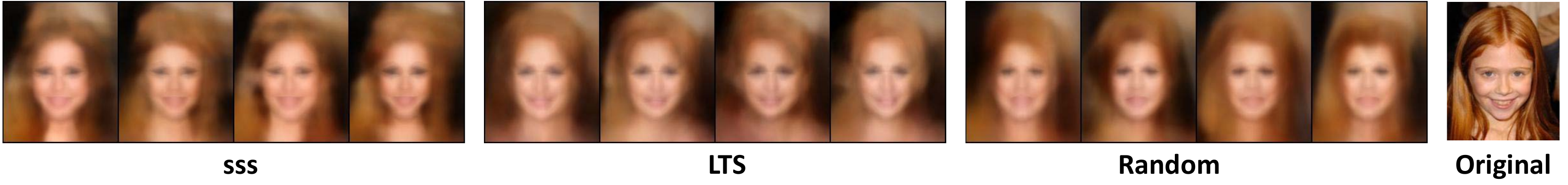}
	\end{subfigure}
    \begin{subfigure}{0.8\textwidth}
		\centering
		\includegraphics[width=\linewidth]{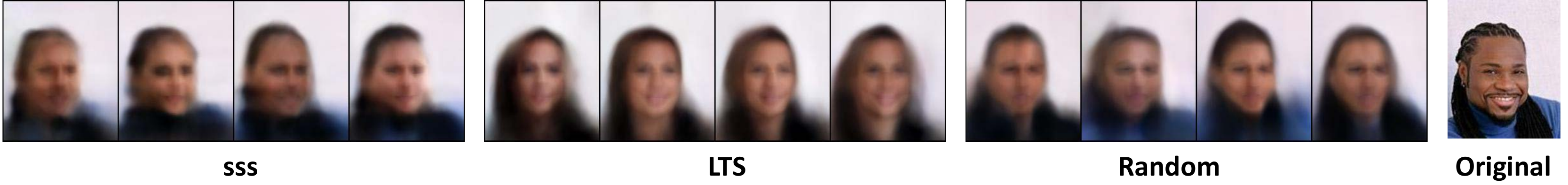}
	\end{subfigure}
    \begin{subfigure}{0.8\textwidth}
		\centering
		\includegraphics[width=\linewidth]{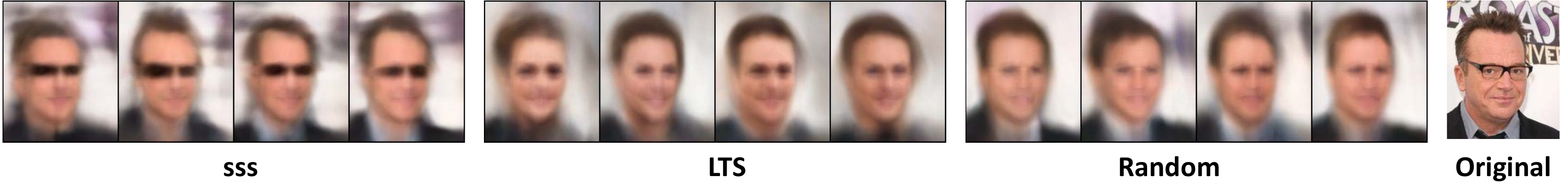}
	\end{subfigure}
    \begin{subfigure}{0.8\textwidth}
		\centering
		\includegraphics[width=\linewidth]{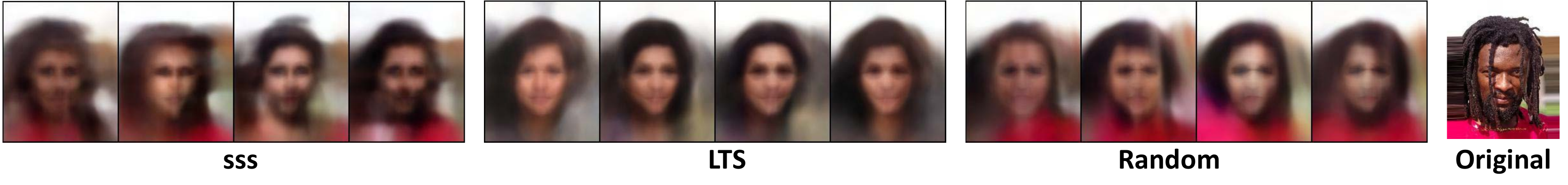}
	\end{subfigure}
    \begin{subfigure}{0.8\textwidth}
		\centering
		\includegraphics[width=\linewidth]{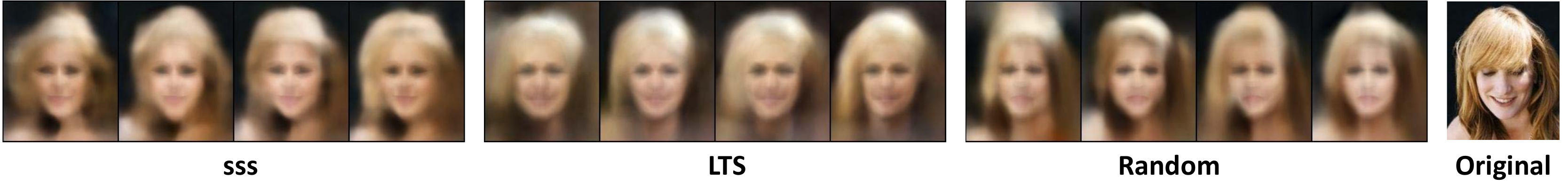}
	\end{subfigure}
    \begin{subfigure}{0.8\textwidth}
		\centering
		\includegraphics[width=\linewidth]{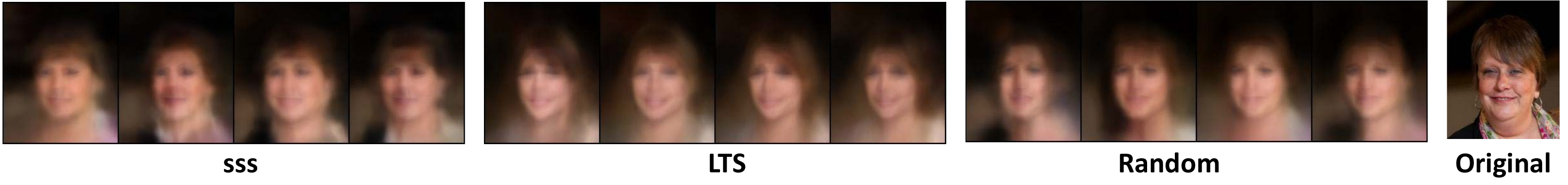}
	\end{subfigure}
    \begin{subfigure}{0.8\textwidth}
		\centering
		\includegraphics[width=\linewidth]{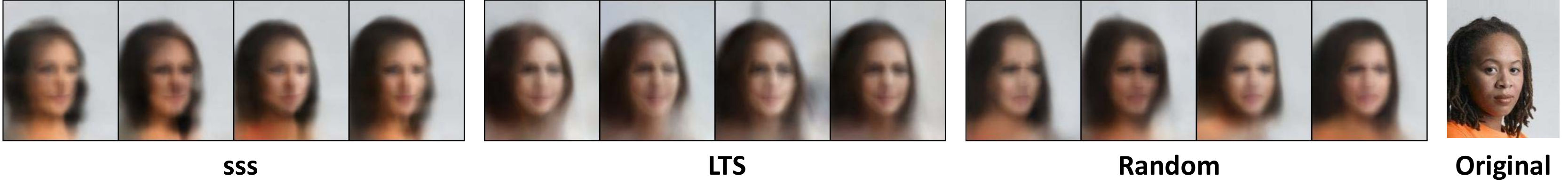}
	\end{subfigure}
    \begin{subfigure}{0.8\textwidth}
		\centering
		\includegraphics[width=\linewidth]{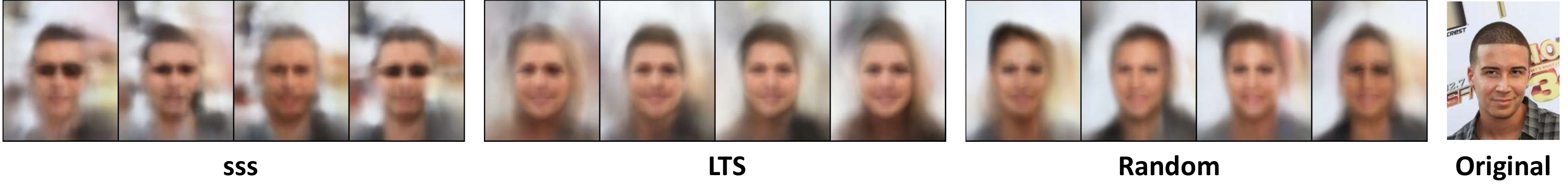}
	\end{subfigure}
	\caption{\small Visualization of reconstructed images for the CelebA dataset. Each model selects 40, 60, 80, and 100 pixels from 
	a $218 \times 178$ image and reconstruct the full image using only the selected pixels.}
	\label{celeba_reconstruction_appendix}
	\vspace{-0.1in}
\end{figure*}
\begin{figure*}
\centering
\includegraphics[width=1.0\linewidth]{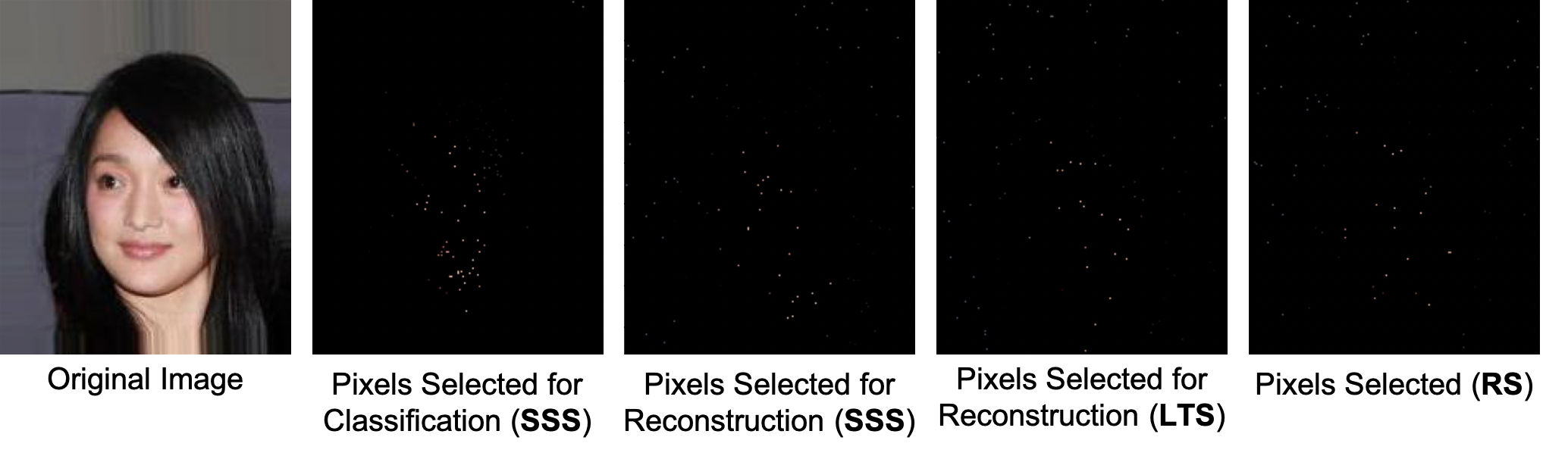}
\vspace{-0.25in}
\caption{\small \textbf{Zoom-In for best view.} Selected pixels for different tasks on CelebA. As can be seen from the selected pixels,
SSS adaptively selects different pixels for both reconstruction and classification. Pixels for reconstruction are more spread out to include the 
background since this contributes to the reconstruction loss. For classification, almost all the pixels are focused on the face since most of the attributes
can be found there.}
\vspace{-0.1in}
\label{celeba_pixels}
\end{figure*}

\subsection{CelebA}\label{celeba-appendix}
Figure~\ref{celeba_reconstruction_appendix} shows samples of reconstructed images while varying the number of selected pixels. Additionally in 
Figure~\ref{celeba_pixels}, we show the selected pixels of our model for both the classification and reconstruction task. For the attribute
classification task, the model tends to select pixels mainly from the face, since the task is to classify characteristics of the person. For
reconstruction, the selected pixels are more evenly distributed, since the background also contributes significantly to the reconstruction loss.

\section{Dataset Distillation: Instance Selection}\label{app:instance_selection} 
For the instance selection experiments, we construct a set by randomly sampling 200 face images from the full dataset. To evaluate the model, we
create multiple such datasets and run the baselines (Random Sampling, K-Center Greedy and FPS) and SSS on the same datasets. We compute the FID metric~\citep{fid} on the instances and averaged on all the randomly constructed datasets. For FPS, we use the open-source implementation in ~\href{FPS}{\url{https://github.com/rusty1s/pytorch\_cluster}}. Further, we provide qualitative results on a single dataset in Figure ~\ref{supp_celeba}
where  our model picks 5 instances from the full set of 200 face images.

\begin{figure*}
  \centering
  \includegraphics[width=0.7\textwidth]{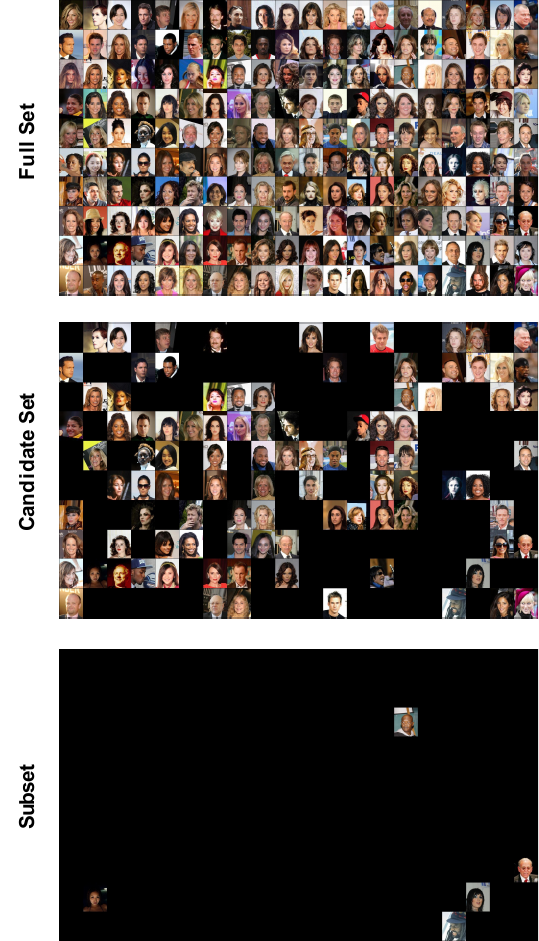}
  \caption{\small Visualization of a set with 200 images for instance selection. The two stage selection method in SSS is visualized as
	  Candidate Set and Subset. A subset of size 5 is visualized.}
  \label{supp_celeba}
\end{figure*}

\subsection{Dataset Distillation: Classification}\label{app:mini_image_net}

In Figure~\ref{miniImagenet} we provide visualizations for the instance selection problem as applied to the few-shot classification task. Here, we go
from a 20-shot to a 1-shot classification problem where the prototype is selected from the support using SSS. The selected subset is then used in place 
of the support set and used to classify new query instances. 
\begin{figure*}
\centering
\includegraphics[width=1.0\linewidth]{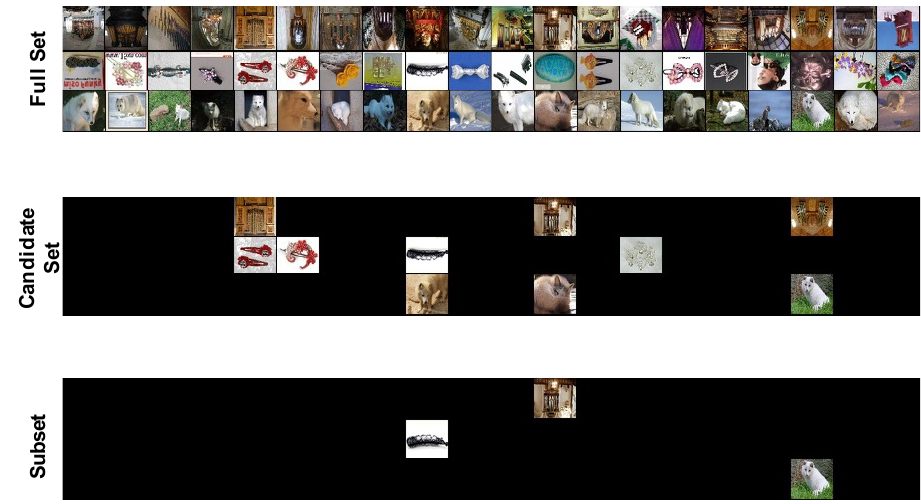}
  \caption{\small Sample visualization of prototype selection for the \textit{mini}Imagenet dataset on the few-shot classification task. Each row
  represents a \textit{set} that corresponds to the \textit{support} from which a prototype is selected for the few-shot classification task.}
\label{miniImagenet}
\end{figure*}

\section{Model Specifications}\label{model-spec}
In this section, we describe the main components of our Set based Stochastic Sumbsampling models --- $g(d), \rho(\overline{d})$ and $\varphi \circ f(D^{(t)}_c, D_s^{(t-1)})$.

For all the experiments, we design $g$ as feedforward neural network with ReLU to project each instance $d$ to lower dimension and average them to obtain the set representation $D_e$, following DeepSets~\citep{zaheer2017deep}. 

We parameterize $\rho(\cdot)$ with a 3 layered feedforward neural network $h$ followed by sigmoid function as described in Equation \ref{can-ber} from Section 3.3. For $f$, we use the function $g$ in Eq.~\ref{set-enc} to extract feature map for each instances in $D^{(t)}_c, D^{(t-1)}_s$ and feed it to set transformer~\citep{lee2018set} for set classification as follows:
\begin{equation}
\begin{gathered}
    f(D^{(t)}_c, D^{(t-1)}_s) = \text{MAB}(D^{(t)}_c, D^{(t-1)}_s) \\
    \text{MAB}(D^{(t)}_c, D^{(t-1)}_s)  = \text{LayerNorm}(H +\text{rFF}(H)) \\
    H=\text{LayerNorm}(D^{(t)}_c + \text{Multihead}(D^{(t)}_c, D^{(t-1)}_s,D^{(t-1)}_s))
\end{gathered}
\end{equation}
where rFF  is a row-wise feedfoward layer which processes each instance independently and Multihead denotes Multihead Attention~\citep{transformer} with each slot of Multihead$(\cdot, \cdot,\cdot)$ representing query, key, and value, respectively.
We use linear layer for $\varphi$ to output logits for each element in $D_c^{(t-1)}$.

\subsection{Attention}
We further elaborate the details of Attention and Multihead Attention for completeness. For a more thorough exposition, we refer the reader to 
\citet{transformer} and \citet{lee2018set}.

An attention module computes the following interactions using the dot product:
\begin{equation}
    \text{Att}(Q, K, V; \omega) = \omega (QK^{\top})V
\end{equation}
where $Q \in \mathbb{R}^{n_q \times d_q}$ are the $n_q$ query vectors each with of dimension $d_q$. $K \in \mathbb{R}^{n_v \times d_q} \text{ and } 
V \in \mathbb{R}^{n_v \times d_q}$ are the keys and values respectively. Interactions are modelled through $QK^{\top}$ and $\omega$ is an 
activation function such as softmax or sigmoid.

Multihead attention projects $Q, K, V$ to $h$ different vectors each with $d_q^M, d_q^M, d_v^M$ dimensions and computes $h$ different 
attention modules according to the following:
\begin{equation}
    \text{Multihead}(Q, K, V; \lambda, \omega) = \text{concat}(O_1, \ldots, O_h)W^{O}
\end{equation}
where
\begin{align}
    O_j &= \text{Att}(QW_j^Q, KW_j^K, VW_j^V; \omega_j)
\end{align}
The leanarble parameters for Multihead Attention are $\lambda = \{W_j^Q, W_j^K, W_j^V\}_{j=1}^{h}$, where  
$W_j^Q, W_j^K \in \mathbb{R}^{d_q \times d_q^M}$, $W_j^V \in \mathbb{R}^{d_v \times d_v^M}$ and $W^O \in \mathbb{R}^{hd_v^M \times d}$. In all 
our experiments, we use sigmoid as the activation function.

\subsection{Architectures for MNIST Classification}\label{app:mnist_architecture}
\textbf{MLP} We use an architecture with 5 layers with outputs 784, 256, 128, 128 and 10 respectively. With the exception of the last layer, all layers are followed by a LeakyReLU activation function. The 3rd linear layer is also followed by a dropout layer with $p=0.2$. This is the same architecture used in ~\citet{dps}. However, we test on the full MNIST test set instead of using half for validation and the remaining for testing as done in ~\citet{dps}.

\textbf{ConvNet} We use two convolutions with output channels 32 and 64 respectively. All convolutions have kernel size of 3 with stride 1 and are followed by ReLU activation. The final convolution layer is followed by a Maxpooling layer with kernel size 2. This is followed by two linear layers with output sizes of 128 and 10 and the first linear layer is followed by the ReLU activation function. 

\subsection{Architecture for CelebA Attribute Classification}\label{app:celeba_cls_architecture}
\textbf{ConvNet} We use 4 convolutions each with outputs 64, 128, 256 and 512 respectively. Each convolution is followed by a batch normalization layer, ReLU activation and a Maxpooling layer with kernel size 2. This is followed by 3 linear layers with outputs 1024, 256 and 40 respectively. The first two linear layers are followed by the ReLU activation function and the last by the Sigmoid function. 

\subsection{Architecture for Regression Problems} For all the regression problems, we use the ANP model of ~\citet{anp}.



\section{Experiments on CIFAR10 Dataset}\label{app:cifar10}
\begin{figure}[H]
\vspace{-0.1in}
\centering
	\begin{subfigure}{.2\textwidth}
		\centering
		\includegraphics[width=\linewidth]{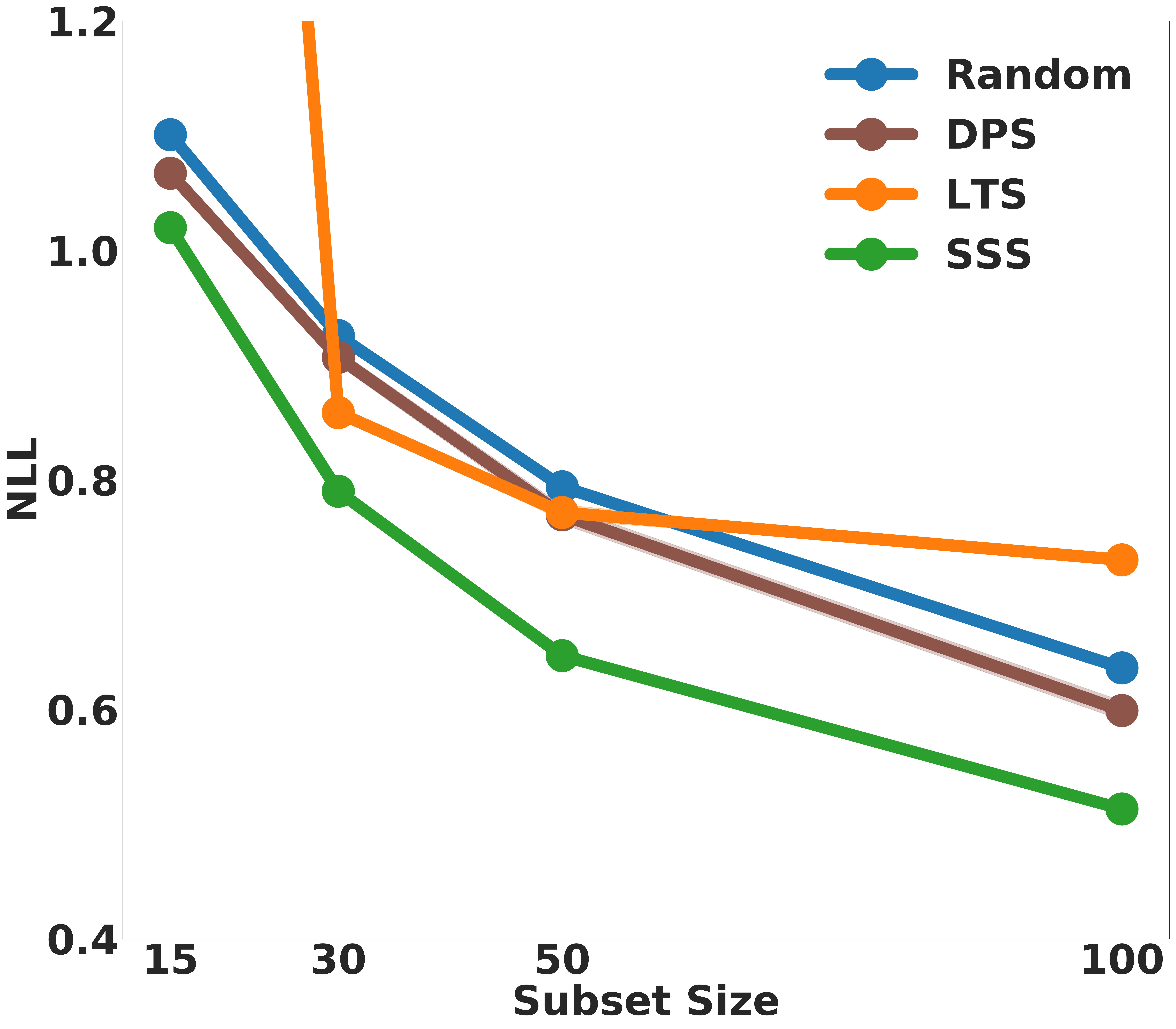}
		 \vspace{-0.25in}
		\caption{\small}
		\label{cifar10_reconstruction}
	\end{subfigure}%
	\begin{subfigure}{.2\textwidth}
		\centering
		\includegraphics[width=\linewidth]{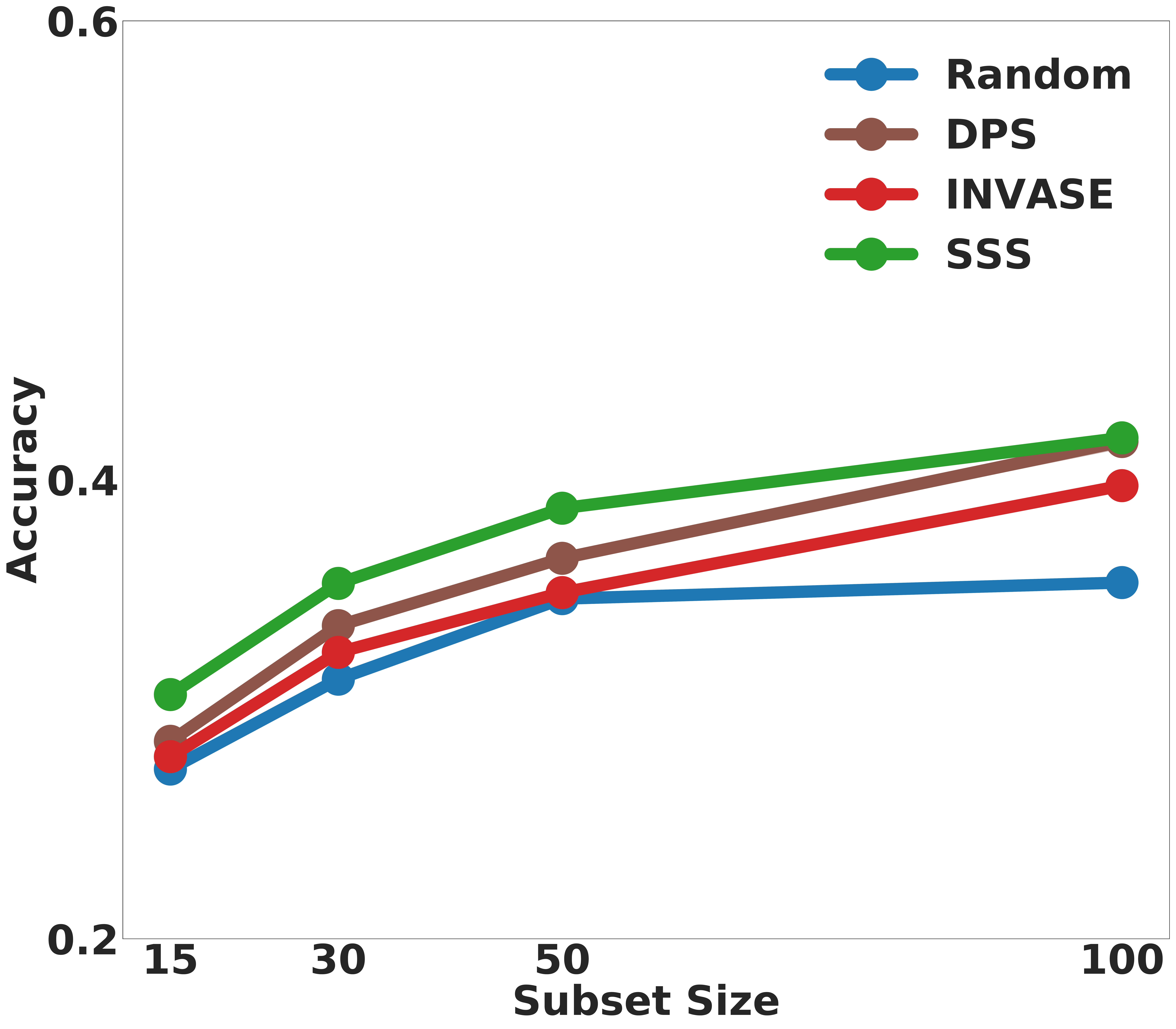}
		\vspace{-0.25in}
		\caption{\small}
		\label{cifar10_classification}
	\end{subfigure}
	\caption[Reconstruction and Classification on CIFAR10]{\small
	\textbf{(a)} CIFAR10 Reconstruction. \textbf{(b)} CIFAR10 Classification.
	}
	\label{fig:cifar10}
\end{figure}
We provide further experimental results on the CIFAR10 dataset in which we subsample pixels for both image reconstruction and image classification. In order to compare with both ~\citet{dps} and ~\citet{invase}, we convert all the images to grayscale, following~\citet{dps} (see Section 4.3). 

The experimental result of the CIFAR10 reconstruction task is presented in Figure~\ref{cifar10_reconstruction}, where for the same subsampling rates, SSS outperforms the competing baselines (DPS, INVASE, LTS and Random Sampling) in terms of the negative log-likelihood. For this task, we use ANP model as we have done in  CelebA and function reconstruction tasks in Section~\ref{sec41}. We do not compare with INVASE since INVASE requires two copies of the reconstruction model and requires more GPU memory.

In Figure~\ref{cifar10_classification}, we present the results for the CIFAR10 classification task. Again we observe that for the same subsampling rate, SSS performs better than DPS, INVASE and Random Sampling and the performance of the same classification model trained on the full input image is $0.70\pm0.02$. Note that we cannot compare with LTS for the same reasons given in Section~\ref{sec:baselines}. 

\section{Experiments on OCT Dataset}
\begin{wrapfigure}{t}{0.2\textwidth}
\small
\centering
\vspace{-0.15in}
\captionof{table}{\small Acc. with \textbf{200} out of 1024 pixels on OCT.}
\resizebox{0.2\textwidth}{!}{
\renewcommand{\arraystretch}{0.75}
\renewcommand{\tabcolsep}{0.5mm}
\begin{tabular}{cccc|c}
\toprule
Random & INVASE & DPS & \textbf{SSS} & Full \\ 
\midrule
38.80 & 41.80 & 76.70 & \textbf{85.90} & 95.50\\
\bottomrule
\end{tabular}
}
\vspace{-0.17in}
\label{tab-exp}
\end{wrapfigure}
We also provide additional results on the \href{https://data.mendeley.com/datasets/rscbjbr9sj/2}{Optical Coherence Tomography} dataset where we need to select the most informative features from tomographic images for prediction and interpretation of diseases in the retina. As shown in Fig.~\ref{fig-vis}, SSS selects \emph{bumpy} regions of the retina crucial for diagnosis of Diabetic Macular Edema. The quantitative result from Table~\ref{tab-exp} further confirms the effectiveness of our method. 

\begin{figure}[H]
\centering
\vspace{-0.14in}
  \includegraphics[width=0.99\linewidth]{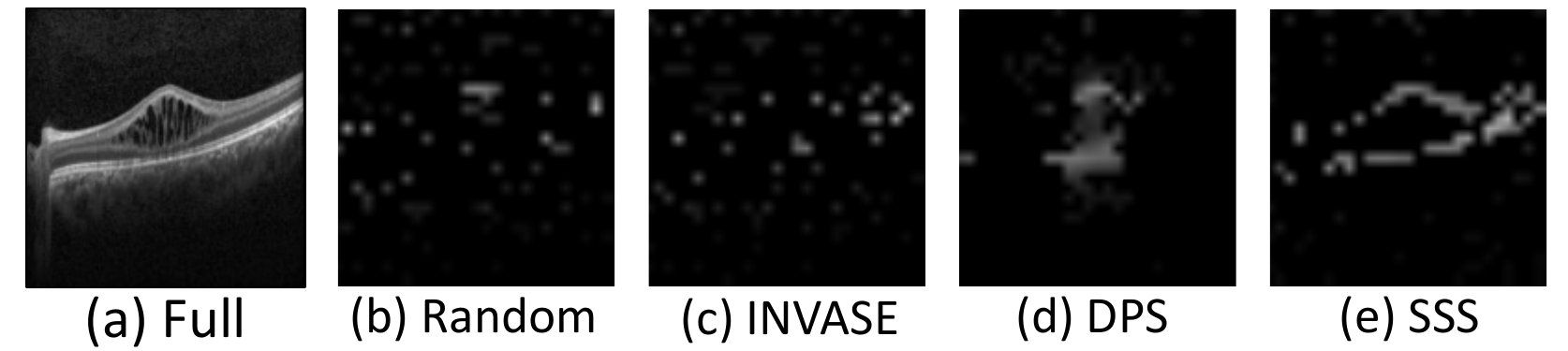}
  \vspace{-0.16in}
  \caption{\small Visualization of selected 100 pixels.}
  \label{fig-vis}
  \vspace{-0.22in}
\end{figure}


\end{document}